\documentclass{article}
\usepackage{graphicx}
\usepackage{authblk}

\usepackage{algorithm}
\usepackage{algorithmic}
\usepackage{basicreq}
\usepackage{math_notations}
\usepackage{times}
\usepackage[round]{natbib}

\title{Bad Values but Good Behavior: Learning\\Highly Misspecified Bandits and MDPs}
\author[1]{Debangshu Banerjee}
\author[2]{Aditya Gopalan}
\affil[1,2]{ Department of Electrical and Communication Engineering, Indian Institute of Science, India}  

\begin{document}

\maketitle

\begin{abstract}
Parametric, feature-based reward models are employed by a variety of algorithms in decision-making settings such as bandits and Markov decision processes (MDPs). The typical assumption under which the algorithms are analysed is realizability, i.e., that the true values of actions are perfectly explained by some parametric model in the class. We are, however, interested in the situation where the true values are (significantly) misspecified with respect to the model class. For parameterized bandits, contextual bandits and MDPs, we identify structural conditions, depending on the problem instance and model class, under which basic algorithms such as $\epsilon$-greedy, LinUCB and fitted Q-learning provably learn optimal policies under even highly misspecified models. This is in contrast to existing worst-case results for, say misspecified bandits, which show regret bounds that scale linearly with time, and shows that there can be a nontrivially large set of bandit instances that are robust to misspecification.
\end{abstract}

\section{Introduction}
\label{s:intro}
Sequential optimization over a set of decisions, e.g., actions in a multi-armed bandit and policies in an MDP, is often carried out by assuming an internal parametric model for the payoff of a decision, which is learnt with time. Well-known instantiations of this approach are algorithms for structured multi-armed bandits, e.g., linear bandits \citep{rusmevichientong2010linearly} and generalized linear bandits \citep{filippi2010parametric}, linear contextual bandits \citep{chu2011contextual}, and, more generally, value function approximation-based methods for Markov Decision Processes \citep{sutton2018reinforcement}.


Learning algorithms that make decisions based on an estimated reward model are known to enjoy attractive performance (regret) guarantees when the true rewards encountered are {\em perfectly realizable} by the model, see, e.g., \cite{oful, filippi2010parametric, jin2020provably}. %
However, it is more likely than not that a parametric class of models is, at best, only an approximation of reality, succinctly expressed by the aphorism `All models are wrong, but some are useful' \citep{box1976science}. But even if the rewards of, say, arms in a multi-armed bandit, are unfortunately estimated with a large error, one may still hope to discern the optimal arm if its (erroneous) estimate ends up dominating those of the other arms. This begs the natural question: While the task of reward estimation can be fraught with (arbitrarily large) error under misspecified models, when (if at all) can the task of optimal action learning remain immune to it?

We initiate a rigorous study of the extent to which sequential decision-making algorithms based on reward model estimation can be robust to misspecification in the model. In particular, we are interested in characterizing the interplay between i) the actual (ground truth) rewards in a decision-making problem and ii) the reward model class used by the algorithm, and how it governs whether the algorithm can still learn to play optimally if the true rewards are not realizable by the reward model. In this respect, our specific contributions are as follows:


\begin{enumerate}
    \item For misspecified linear bandits, we identify a novel family of instances (reward vectors), which we term the robust observation region. Reward vectors in this region are characterized by an invariance property of the greedy action that they induce after being projected, with respect to any weighted Euclidean norm, onto the linear feature subspace. This region, depending upon the feature subspace, can be non-trivially large, and need not be confined to within a small deviation from the subspace. 
    \item We prove that for any instance (i.e., the vector of true mean arm rewards) in the robust observation region, both (i) the  $\epsilon$-greedy algorithm, with least-squares parameter estimation and an exploration rate of $1/\sqrt{t}$ in each round $t$, and (ii) the LinUCB (or OFUL) algorithm, achieve $O(\sqrt{T})$ cumulative regret in time $T$. 
    \item We extend our characterisation of robust instances to linear contextual bandits, for which we provide a generalization of the robust observation region. We show that both the $\epsilon$-greedy and LinUCB algorithms for linear contextual bandits get $O(\sqrt{T})$ regret whenever the true, misspecified, reward vector belongs to this robust observation region.

    \item We finally provide a structural criterion for a finite-horizon Markov Decision Process (MDP), together with a Q-value approximation function class, for which the fitted Q-iteration algorithm provably learns a (near) optimal policy in spite of arbitrarily large approximation error (in the $\infty$-norm sense).
\end{enumerate}
We stress that our results pertain to the original algorithms (i.e., not modified to be misspecification-aware in any manner). It is our novel analytical approach that shows that they achieve nontrivial sublinear regret, even under arbitrarily large misspecification error\footnote{The term `misspecification error' is to be understood as the distance of the arms' reward vector to the model reward subspace.}. This is in contrast to, and incomparable with, existing results that argue that, in the `worst case' across all reward vectors that are a constant distance away from the feature subspace, any algorithm must incur regret that scales linearly with the time horizon, e.g., \cite[Thm. F.1]{pmlr-v119-lattimore20a}.

Our results lend credence, in a rigorous sense, to the observation that reinforcement learning algorithms presumably equipped with only approximate value function models are often able to learn (near-) optimal behavior in practice across challenging benchmarks \citep{mnih2013playing, lillicrap2015continuous}. %
They also shine light on the precise structure of bandit problems that makes robustness possible in the face of significant misspecification.

\paragraph{Illustrative Example} 
\label{par:bandit_example}

This paper's key concepts and results can be understood using a simple toy example of a misspecified 2-armed (non-contextual) linear bandit. Assume that the vector of mean rewards of the arms (the ``instance") is %
{$\bm{\mu} = \begin{bmatrix}
    \mu_1, 
    \mu_2
\end{bmatrix}^\top = \begin{bmatrix}
    20 ,
    3
\end{bmatrix}^\top$} in {$\Real^2$} (marked by $\times$ in Fig. \ref{fig:illustration_eg}). 
Suppose one attempts to learn this bandit via a 1-dimensional linear reward model in which the arms' features are assumed to be {$\phi_1 = 3$} and { $\phi_2 = 1$}. It follows that (i) any (2-armed) bandit instance in this linear model is of the form {$\bm{\Phi} \theta$}, where {$\theta \in \Real$} and {$\bm{\Phi} = \begin{bmatrix}
    \phi_1, 
    \phi_2
\end{bmatrix}^\top = \begin{bmatrix}
    3,
    1
\end{bmatrix}^\top \in \Real^{2\times 1}$}, and corresponds to an element in the range space of {$\bm{\Phi}$}, and (ii) the instance {$\bm{\mu}$} is misspecified as it is off this subspace\footnote{The $l_\infty$ misspecification error (deviation from subspace) of {$\bm{\mu}$}, for this example, is {$2.75$}.}.



For ease of exposition, consider that there is no noise in the rewards observed by pulling arms. In this case, the ordinary least squares estimate of {$\theta$}, computed at time {$t$} from {$n_1$} observations of arm {$1$} and {$n_2$} observations of arm {$2$}, is {$\hat{\theta}_t = \frac{n_1 \phi_1 \mu_1 + n_2 \phi_2 \mu_2}{n_1\phi_1^2 + n_2\phi_2^2}$}. A key observation is that {$\hat{\theta}_t$} can always be written as a convex combination of {$\mu_1/\phi_1 = 6.7$} and {$\mu_2/\phi_2 = 3$}: {$\hat{\theta}_t = \frac{n_1\phi_1^2}{n_1\phi_1^2+n_2\phi_2^2}\Big(\mu_1/\phi_1\Big) + \frac{n_2\phi_2^2}{n_1\phi_1^2+n_2\phi_2^2}\Big(\mu_2/\phi_2\Big)$}, for any sampling distribution of the arms. The corresponding parametric reward estimate {$\bm{\Phi} \hat{\theta}_t$}, must thus lie in the set {$\{[3\theta, \theta]^\top: \theta \in [3,6.7]\}$}, which appears as the hypotenuse of the right triangle with vertex {$(20,3)$} in Fig. \ref{fig:illustration_eg}.


Note that if a greedy rule is applied to play all subsequent actions ({$A_{t+1} = \argmax_{i = 1,2} \phi_i^\top \hat{\theta}_t$}), then the action will be {$1$}, since the point {$\bm{\Phi} \hat{\theta}_t$} will always be  `below' the standard diagonal {$\mu_1 = \mu_2$} (the black line in Fig. \ref{fig:illustration_eg}). Since action {$1$} is optimal for the (true) rewards {$\bm{\mu}$}, the algorithm will never incur regret in the future. 

The instance above has a misspecification error significantly smaller than the reward gap ({$2.75 < 17$}). One can also find instances at the other extreme, e.g., {$\bm{\tilde{\mu}} = [20,18]^\top$} (marked by $\circ$ in Fig. \ref{fig:illustration_eg}) for which the misspecification error is much larger than the gap ({$8.5 > 2$}), that remain robust in the sense of regret. Such instances (all of them colored green) fall outside the scope of existing work on misspecified bandits \citep{zhang2023interplay}, and we address them in our work. 

A rather extreme form of robustness in the face of arbitrary misspecification is depicted in Fig. \ref{fig_b : features}. Here, the 1-dimensional model class is a {\em nonlinear}, tube-shaped manifold\footnote{If a connected manifold is desired, then one can replace the tube with a very thin and long ellipsoid stretched along the diagonal, with similar conclusions.} in $\Real^2$, and similar arguments as above yield that almost all instances $\bm{\mu} \in \Real^2$ yield sublinear regret under, say, $\epsilon$-greedy sampling!




The remainder of the paper formalizes this observation for a variety of decision problems (bandits, both contextual and otherwise, and MDPs), and algorithms that incorporate some form of exploration ({$\epsilon$}-greedy and optimism-based). It explicitly characterizes the set of all (true) reward instances for which no-regret learning is possible.

\begin{figure}[htbp]
    \begin{subfigure}[\footnotesize{A 2-armed, noiseless bandit with 1-dimensional linear approximation. Each point in the plane represents the true rewards of both arms (the ``instance"). The blue line is the set of instances expressed by the linear approximation. The green and red regions denote the robust regions and the non-robust regions for this linear function approximation. The misspecified instances $(20,3)$ and $(20,18)$ yield no regret under greedy arm selection based on an estimated linear model since any linear estimate of the rewards always has arm $1$ dominating arm $2$.}\label{fig:illustration_eg}]{\includegraphics[width = 0.5 \linewidth]{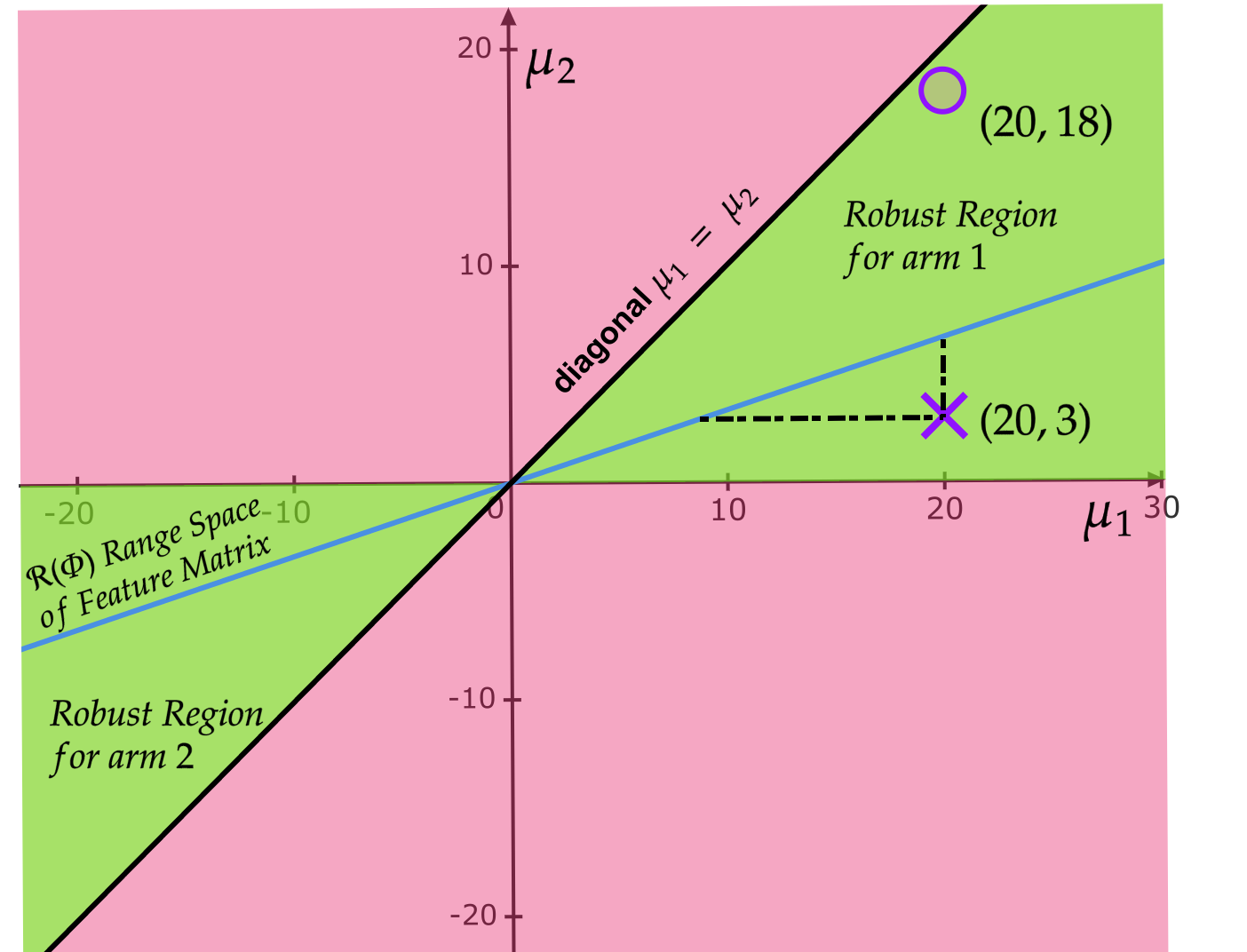}}
    \end{subfigure}
    \begin{subfigure}[\footnotesize{A function approximation class which is described as an $\epsilon$-radius tube about the diagonal. We give a representative diagram for a $\Real^2$ space corresponding to a bandit problem with two arms. We see that except for a measure zero set of bandit instances on the diagonal, which can be interpreted as both arms having the same rewards, all instances are robust.} \label{fig_b : features}]{\includegraphics[width = 0.5\linewidth]{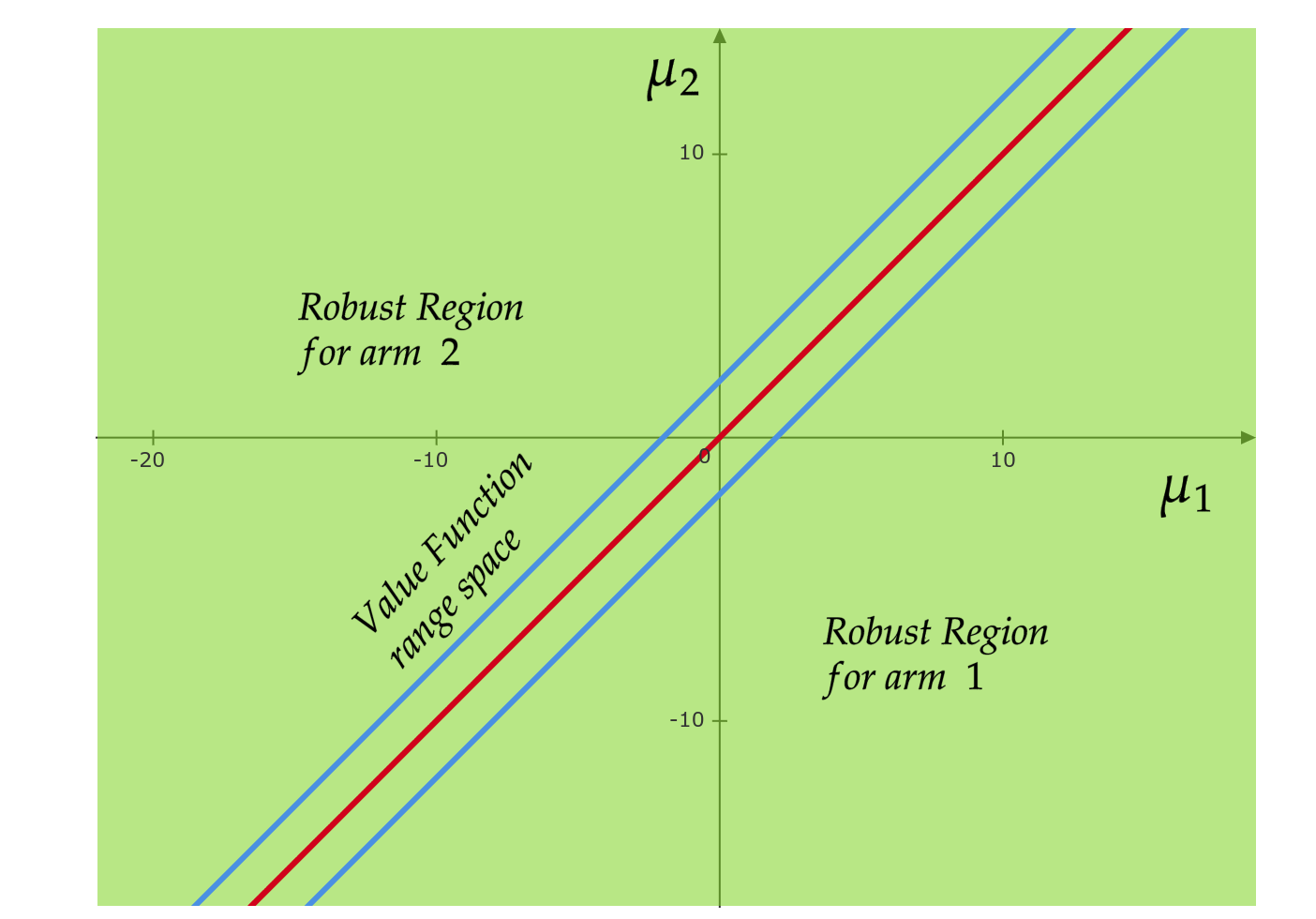}}
    \end{subfigure}
    \caption{\footnotesize{Illustration of robust regions for two function approximations}}
\end{figure}

\subsection{Related Work}
Existing work on misspecified bandits contributes results that can be put into two categories. The first category is negative results of a worst-case form, for instance: there exist reward instances for linear bandits for which (i) the model misspecification error ($l_\infty$ distance of the reward vector from the model subspace $\text{Range}(\bm{\Phi})$) is at most $\rho > 0$ and (ii) the LinUCB algorithm suffers regret $\Omega(\rho T)$ in $T$ rounds. %
Results of this nature are well established in the misspecified bandit and contextual bandit literature, including but not limited to the works of \citet{ghosh2017misspecified, pmlr-v119-lattimore20a, zanette2020learning}. \citet{pmlr-v119-lattimore20a} present a frequentist analysis that shows that even with the knowledge of the misspecification error, a modified form of LinUCB fails to give sub-linear regret in the misspecified contextual setting. \citet{zanette2020learning} extend a similar analysis to learning MDPs.

The second type of result is along a positive direction: to develop conditions, and associated algorithms, under which misspecified bandits can give sub-linear regret. The work of \citet{pmlr-v216-liu23c} analyzes the LinUCB algorithm when the sub-optimality gaps of the arms bound the misspecification. They show that when the misspecification is of low order, the algorithm enjoys sub-linear regret. Under a similar condition, \citet{zhang2023interplay} were able to extend the study to the contextual setting. They propose a phased arm elimination algorithm, which performs similarly to SupLinUCB \citep{chu2011contextual} but requires knowledge of the sub-optimality gap. 

Our results, in this positive spirit, give more fine-grained structural conditions 
to ensure that algorithms such as $\epsilon$-greedy, LinUCB or Fitted-$Q$-learning can achieve robust learning without any additional modifications. A detailed 
survey of previous works is presented in Appendix \ref{sec:previous_app}. 
\section{Problem Statement and Results}
\label{sec:probstatement}
We describe our results for misspecified linear bandits and then proceed to the misspecified contextual case in the second part of this section.



\subsection{Linear Bandits}
\label{subsec:Bandits}
\paragraph{Problem Statement}
We consider a {$K$}-armed bandit with mean rewards  {$\{\mu_i\}_{i=1}^K$}. We assume each arm $i$ is associated with a known feature vector {$\phi_i \in \mathbb{R}^d$}. We also have a given parametric class of functions serving to (approximately) model the mean reward of each arm; each function in this class is of the form {$f_\theta: \mathbb{R}^d \to \mathbb{R}$} for a parameter {$\theta \in \Theta$}, and applying it to arm $i$ yields the expressed reward {$f_\theta(\phi_i) \in \mathbb{R}$}. 



Let us assume that the number of arms, {$K$}, is larger than the dimensionality of the parameter $\theta$, that is {$K > d$} and, for ease of analysis, consider the set of features {$\{\phi_i\}_{i=1}^K$} to span $\mathbb{R}^d$. We shall denote the set of true mean rewards by the vector,
{$\bm{\mu} = \begin{bmatrix}
\mu_i
\end{bmatrix}_{i \in [K]} \in \Real^{\mathrm{K}}$} and the feature matrix {$\bm{\Phi} = \begin{bmatrix}
    \cdots \phi_i^\top \cdots \\
\end{bmatrix}_{i \in [K]} \in \Real^{K \times d}$}. The assumption on the features fixes the rank of {$\bm{\Phi}$} as $d$. 
\begin{remark}[\textbf{Linear Bandits}] Our setting is rather general and covers a broad class of parametric bandits. For example, in linear bandits \citep{dani2008stochastic}, it is assumed that the means are linear functions of the features, that is, there exists a $\theta^*$, such that  $\bm{\mu} = \bm{\Phi} \theta^*$. 
\end{remark}
\begin{remark}[\textbf{Misspecification}] We allow the vector of true rewards $\bm{\mu}$ to be arbitrary, without imposing a realizability condition like {$\bm{\mu}\;\; \in\;\; \{f_\theta(\phi_i) \;\forall\;i \in [K] : \theta \in \Real^d\}$}.
\end{remark}

We shall derive sufficient conditions under which (misspecified) bandit instances $\bm{\mu}$ can yield sub-linear regret under standard algorithms like $\epsilon$-greedy and LinUCB. A key step towards this is to identify a \emph{robust region} which depends on the model class such that any instance in this region is guaranteed to enjoy sub-linear regret. 

We begin by recalling the usual \emph{greedy regions} in $\Real^{\mathrm{K}}$ which characterize reward vectors that share the same optimal arm.
\begin{definition}[\textbf{Greedy Region $\mathcal{R}$}]
Define by {$\cR_k$}, for any $k \in [K]$, as the region in {$\Real^{\mathrm{K}}$} for which the $k^{th}$ arm is optimal, i.e.,  {$\cR_k \triangleq \Big\{ \bm{\mu} \in \Real^{\mathrm{K}} : \mu_k > \mu_i \forall i \neq k \Big\}$}.
\end{definition}

For the purposes of clarity, let us fix our model class to be linear, so that the least squares estimate has a closed form solution amenable to explicit analysis. Denoting the play count of each arm $i$ up to time $t$ by {$\alpha(i,t) = \sum_{s=1}^t \mathbb{1}\{A_s = i\}$}, the least squares estimate {$\hat{\theta}_{t+1}$} can be written as {
$\hat{\theta}_{t+1} = \Big[\sum_{i=1}^K \alpha(i,t)\phi_i\phi_i^\top\Big]^{-1}(\sum_{i=1}^K \alpha(i,t)\phi_i\hat{\mu}_{i,t}) = (\bm{\Phi}^\top \Lambda(t) \bm{\Phi})^{-1}\bm{\Phi}^\top \Lambda\bm{\hat{\mu}_t}\,$},
{ where $\hat{\mu}_{i,t}$ is the sample mean of the rewards from arm $i$ as observed till time $t$, that is $ \hat{\mu}_{i,t} = \frac{\sum_{s=1}^t y_s\mathbb{1}\{A_s = i\}}{\sum_{s=1}^t\mathbb{1}\{A_s = i\}}$, $\Lambda(t)$ is a $K \times K$ diagonal matrix with $\alpha(i,t)$ terms on the diagonal and $\bm{\hat{\mu}_t}$ is a $K$ dimensional vector with $\hat{\mu}_{i,t}$ as its elements \citep{gopalan2016low}}.
\begin{remark}
    Note that as per our assumption on the rank of the feature matrix, {$\bm{\Phi}^\top\Lambda(t)\bm{\Phi}$} is invertible if {$\alpha(i,t) > 0$} for any $d$ linearly independent features. This can be ensured by explicitly forcing the algorithm to sample $d$ such features at least once. For ease of exposition, we shall assume that {$\bm{\Phi}^\top\Lambda(t)\bm{\Phi}$} is invertible. In Appendix \ref{sec:ridge_app} we show how the ideas presented here can be extended to include a regularizer term. 
\end{remark}
We note from the closed form solution of the least squares estimate, that we can equivalently interpret the play counts {$\alpha(i,t)$} for any arm {$i \in [K]$} in terms of (normalized) sampling frequencies: {$\alpha(i) = \frac{\sum_{s = 1}^t \mathbb{1}\{A_s = i\}}{t}$} for any $t \geq 1$, so that {$\{\alpha(i)\}_{i=1}^K$} belongs to the $K$ dimensional simplex {$\Delta_{\mathrm{K}}$}.  that is {$\sum_{i=1}^K \alpha(i) = 1$}. With this notation, we can observe that any least squares estimate is calculated based upon some sampling frequency {$\{\alpha(i)\}_{i=1}^K \in \Delta_{\mathrm{K}}$} and the corresponding sample estimates $\bm{\hat{\mu}_t}$. The least squares estimate can now be rewritten as an optimization problem with respect to a weighted norm: {$\hat{\theta}_t = \argmin_\theta\Big\|\bm{\Phi}\theta - \bm{\hat{\mu}_t}\Big\|_{\Lambda^{1/2}}$}, %
    for some sampling frequency {$\Lambda = \mathrm{diag}(\{\alpha(i)\}_{i=1}^K)$} in the simplex $\Delta_{\mathrm{K}}$. 
\begin{definition}[\textbf{Model Estimate under Sampling Distribution}]
\label{def: projection_bandit}
    For any bandit instance $\bm{\mu}$ in $\Real^{\mathrm{K}}$, we shall denote the model estimate of the $K$ dimensional element $\bm{\mu}$ under any sampling distribution, defined by {$\Lambda = \mathrm{diag}(\{\alpha(i)\}_{i=1}^K)$,} as {$\mathbf{P}^{\Lambda}(\bm{\mu}) \triangleq \argmin_\theta\Big\|\bm{\Phi}\theta - \bm{\mu}\Big\|_{\Lambda^{1/2}}$}.    
\end{definition}
\begin{remark}
    With this definition, we observe that the least squares estimate $\hat{\theta}_t$ is {$\mathbf{P}^{\Lambda}(\bm{\hat{\mu}_t})$} in our notation, for some sampling distribution defined by {$\Lambda = \mathrm{diag}(\{\alpha(i)\}_{i=1}^K)$} and the corresponding sample estimates of $\bm{\hat{\mu}_t}$. Further note, that another way of describing {$\mathbf{P}^{\Lambda}(\bm{\mu})$} is by considering the least squares estimate, with $\bm{\hat{\mu}_t}$ being exactly $\bm{\mu}$, which is equivalent of saying there is no stochasticity in the observations of each arm pull. 
\end{remark}
Given a given bandit instance {$\bm{\mu} \in \Real^{\mathrm{K}}$}, assume, without loss of generality, that $k$ is the optimal arm, that is $\bm{\mu}$ belongs to the $k^{th}$ greedy region $\cR_k$. Now consider the model estimate {$\mathbf{P}^{\Lambda}(\bm{\mu})$} estimated under exact observations of $\bm{\mu}$ (that is, no stochasticity), using a sampling distribution of {$\Lambda = \mathrm{diag}(\{\alpha(i)\}_{i=1}^K)$}. If the projection, {$\bm{\Phi} \mathbf{P}^{\Lambda}(\bm{\mu})$} belongs to the true greedy region $\cR_k$, then under this sampling distribution, the model estimate would always return the optimal arm under a greedy strategy. That is, {$\argmax_{i \in [K]} \phi_i^\top \mathbf{P}^{\Lambda}(\bm{\mu})$} would result in the $k^{th}$ arm being pulled.
With this motivation, we define \emph{robust regions}.
\begin{definition}[\textbf{Robust Parameter Region}]
\label{def:robust_parameter_bandit}
For a given feature matrix $\bm{\Phi}$, we define the $k^{th}$ \emph{robust parameter region} $\Theta_k$ as the the set of all parameters $\theta \in \Real^d$ such that the range space of $\bm{\Phi}$ restricted to {$\Theta_k$} lies in the $k^{th}$ greedy region $\mathcal{R}_k$. That is 
{
$\Theta_k = \Big\{\theta \in \Real^d : \bm{\Phi}\theta \in \mathcal{R}_k \Big\}, \;\;\;\text{for any arm } k.$}
\end{definition}
\begin{definition}[\textbf{Robust Observation Region}]
\label{def:robust_observation_bandit}
For a given feature matrix $\bm{\Phi}$, we define the $k^{th}$ \emph{robust observation region} $\cC_k$, as the set of all {$K$} armed bandit instances $\bm{\mu}$ with optimal arm $k$, such that under any sampling distribution {$\{\alpha(i)\}_{i=1}^K \in \Delta_{\mathrm{K}}$}, the corresponding model estimate, {$\mathbf{P}^{\Lambda}(\bm{\mu})$}, lies in the $k^{th}$ robust parameter region, {$\Theta_k$}. That is, {
\begin{align*}
    \cC_k = \Big\{\bm{\mu} \in \cR_k : \mathbf{P}^{\Lambda}(\bm{\mu}) \in \Theta_k \;\forall\; \Lambda \in \bm{\Lambda} \Big\}\;\;\;\textit{for any arm } k,\end{align*}}
where \hspace{1mm}{ 
$
\bm{\Lambda} = \Big\{\Lambda = \mathrm{diag}(\{\alpha(i)\}_{i=1}^K) :  \{\alpha(i)\}_{i=1}^K \in \Delta_{\mathrm{K}} \Big\}.
$}
\end{definition}
From the definitions of the \emph{robust observation region} and \emph{robust parameter region}, we observe that if $\bm{\mu}$ belongs to the $k^{th}$ robust observation region $\cC_k$, then the greedy strategy {$\argmax_{i \in [K]} \phi_i^\top \mathbf{P}^\Lambda(\mu)$} would ensure the optimal arm $k$ is chosen under any sampling distribution of the arms. 
\paragraph{Characterization of the Robust Observation Region} In the class of linear models, the \emph{robust observation region} $\cC_k$ has a closed-form structure that can be evaluated.
\begin{theorem}
\label{thm:bandit_convex_hull}
For any reward vector $\bm{\mu}$ with optimal arm $k$, $\bm{\mu}$ belongs to the \emph{robust observation region} $\cC_k$ if and only if every $d \times d$ full rank sub-matrix of $\bm{\Phi}$, denoted by $\Phi_{d}$, along with the corresponding $d$ rows of $\bm{\mu}$, denoted by $\bm{\mu}_{d}$, satisfies the condition that {$\Phi_{d}^{-1}\bm{\mu}_d \in \Theta_k$}. In other words,
{
\begin{align*} 
    \forall \;\; \bm{\mu} \in \cR_k\,,\;\;\;\; \bm{\mu} \in \cC_k \iff  \Phi_d^{-1}\bm{\mu}_d \in \Theta_k\; 
\end{align*}}
for all $d\times d$ full rank sub-matrices of $\bm{\Phi}$ (denoted as $\Phi_d$) and the corresponding $d$ rows of $\bm{\mu}$ (denoted as $\bm{\mu}_d$).
\end{theorem}
\begin{proof}
The proof uses a result of \citet{doi:10.1137/S0895479895284014}, presented in Lemma \ref{lemma:forsgren}, that for any any sampling distribution {$\{\alpha(i)\}_{i=1}^K \in \Delta_{\mathrm{K}}$}, the model estimate {$\mathbf{P}^\Lambda(\bm{\mu})$}, with {$\Lambda = \mathrm{diag}(\{\alpha(i)\}_{i=1}^K)$}, lies in the convex hull of the {\em basic solutions} {$\Phi_d^{-1}\bm{\mu}_d$}. \footnote{Where we have implicitly assumed that $\bm{\Phi}^\top\Lambda\bm{\Phi}$ is invertible} 
Thus for any {$\bm{\mu}\in \cR_k$} \begin{align*}
    \bm{\mu} \in \cC_k &\iff \mathbf{P}^\Lambda (\bm{\mu}) \in \Theta_k\; \forall\; \Lambda \in \bm{\Lambda} \\
    &\iff  \mathrm{conv}\{\Phi_d^{-1}\bm{\mu}_d \;\forall\;\Phi_d\subset\bm{\Phi} \} \subset \Theta_k, \\
    &\iff \Phi_d^{-1}\bm{\mu}_d \in \Theta_k \; \forall\; \Phi_d \subset \Phi.
\end{align*}
The first condition above follows from Definition \ref{def:robust_observation_bandit}, and the second condition follows from Lemma \ref{lemma:forsgren}. (We abuse notation to denote $d\times d$ full rank sub-matrices of $\bm{\Phi}$ by $\Phi_d \subset \bm{\Phi}$ and use $\mathrm{conv}$ to denote the convex hull.) The last assertion follows because {$\Theta_k$} is a convex set.
\end{proof}
\paragraph{Example} We return to our example presented at the beginning (in Figure \ref{fig:illustration_eg}) to highlight the definitions we have made so far. The greedy regions, {$\cR_1$} and {$\cR_2$}, are the two half-spaces separated by the diagonal $\mu_1 = \mu_2$. From our choice of the feature matrix $\bm{\Phi}$ as {$\begin{bmatrix}
    3,\;\
    1
\end{bmatrix}^\top$}, we note that for any parameter $\theta$ more than $0$, the range space of $\bm{\Phi}$ belong to {$\cR_1$}. Thus, {$\Theta_1$}, the robust parameter region corresponding to arm $1$, is the set of all positive scalars. Similarly, {$\Theta_2$}, is the set of all negative scalars. {$\cC_1$}, the robust observation region, corresponding to arm $1$, is given by the set, {$\{\bm{\mu}\in \Real^2 \,:\, \mu_1>\mu_2>0 \}$}. The robust observation region for arm $2$, {$\cC_2$} is given by the set {$\{\bm{\mu}\in \Real^2 \,:\, \mu_1<\mu_2<0 \}$}. This illustrates the existence of a large class of bandit problems which are misspecified but robust for our model class. 
\paragraph{Sufficient Conditions for Zero Regret with Stochastic Rewards} The astute reader must have noticed by now that we have made our claims of robustness based on least squares estimates calculated from {\em noiseless} observations of $\bm{\mu}$. In the general setting of stochastic rewards, the empirical estimate of $\bm{\hat{\mu_t}}$ need not always belong to the \emph{robust observation region}. We shall show, however, that (under standard sub-Gaussian noise assumptions) this estimate will {\em eventually} fall inside the \emph{robust} region when playing $\epsilon$-greedy and LinUCB algorithms, after which the algorithms cease to suffer linear regret. We make the following standard assumption regarding the noise in rewards being sub-Gaussian.
\begin{assumption}[\textbf{sub-Gaussian Reward Observations}]
\label{assm:bandit_subg_noise}
    We shall assume that the {$K$} armed bandit instance {$\bm{\mu}$} is {$1/2\text{ sub-Gaussian}$}\footnote{The reason for choosing $1/2$ is purely for ease of calculation and can be replaced by any other constant.}.
\end{assumption}
We shall show that $\bm{\mu}$ belonging to a robust region is a sufficient condition for $\epsilon$-greedy to achieve sub-linear regret. For a technical reason we shall require a bit stronger condition, namely that $\bm{\mu}$ is an interior point of a robust region, that is, there exists a ball {$\mathcal{B}_\delta(\bm{\mu})$} for some {$\delta > 0$} contained in {$\cC_k$} \footnote{Here without loss of generality, we assume that the $k^{th}$ arm is optimal for the bandit instance.}. For the purpose of analysis we take the topology of {$\Real^{\mathrm{K}}$}, as open rectangles in {$\Real^{\mathrm{K}}$}. With this topology we can define an interior point as follows.
\begin{definition}[\textbf{Interior Point}]
\label{def:bandit_interior}
$\bm{\mu}$ is an interior point of $\cC_k$ implies that there exists a {$\delta > 0$}, such the {$K$}-dimensional cell, with size {$\delta$} and centre {$\bm{\mu}$}, is a subset of {$\cC_k$}. That is, 
{\begin{align*}
    \bm{\mu} \in \mathrm{Int}(\cC_k) \iff \exists \; \delta>0\;\;\; s.t.\;\;\; \mathbf{1}_\delta(\bm{\mu}) \subset \cC_k,
\end{align*}} 
where {$\mathbf{1}_\delta(\bm{\mu})$} is the {$K$}-dimensional cell defined as, 
{$\mathbf{1}_\delta(\bm{\mu}) \triangleq \{ \bm{\xi} \in \Real^{\mathrm{K}}\;\; :\;\; \xi_i \in (\mu_i-\delta,\; \mu_i+\delta)\;\; \forall \;\;i \in [K]\}$}.   
\end{definition}
\begin{remark}
    As can be observed from Figure \ref{fig:illustration_eg}, larger the separation between the arms, {$\Delta_i = \mu^*-\mu_i$}, the larger the {$\delta$} one can choose so that {$\mathbf{1}_\delta(\bm{\mu})$} is contained in {$\cC_k$}. Thus, instances with larger sub-optimality gaps are more interior and hence more robust. This observation leads us to think of {$\delta$} as a measure of robustness.
\end{remark}
\paragraph{$\epsilon$-greedy algorithm} The $\epsilon$-greedy algorithm is one of the more popular algorithms in bandit and reinforcement learning literature \citep{sutton2018reinforcement}. We show that if any bandit instance lies in a robust region then $\epsilon$-greedy is guaranteed to enjoy sub-linear regret. We prove the following result.
\begin{theorem}[$\epsilon$-greedy, Proof in Appendix \ref{sec:proofs}]
\label{thm:bandit_eps_greedy}
For a given feature matrix {$\bm{\Phi}$} the {$\epsilon$}-greedy algorithm with {$\epsilon_t$} set as {$1/\sqrt{t}$} achieves {$O(\sqrt{T})$} regret for any bandit instance belonging to the robust observation region.
\end{theorem}
\begin{remark}
    Note that our regret guarantee in Theorem \ref{thm:bandit_eps_greedy} as exposed in detail in the Appendix does not depend upon the $l_\infty$ misspecification error (or any measure of misspecification) and depends only on the suboptimality gap. Our experiments illustrated in Appendix \ref{subsec:exp_bandits} corroborate this.  
\end{remark}
\paragraph{LinUCB} The LinUCB algorithm is a classic algorithm \citep{oful} which is known to be regret optimal in the perfect realizability setting. We prove the following result.
\begin{theorem}[LinUCB, Proof in Appendix \ref{sec:linucb}]
\label{thm:bandit_LinUCB_main}
For a feature matrix $\Phi$ and its associated robust observation region $\cC_k$, the LinUCB algorithm achieves a regret of {$O(d\sqrt{T})$} on any bandit parameter $\bm{\mu}$ which in an interior point of the robust observation region. 
\end{theorem}
\begin{remark}
The difficulty in the proof is that while it is true that sample estimates {$\bm{\hat{\mu}_t}$} must fall inside the \emph{robust observation regions} with high probability, we cannot trivially conclude the same for the parameter estimates {$\hat{\theta}_t$}. We resolve this by bounding the number of times the algorithm plays suboptimally. The details are presented in Appendix \ref{sec:linucb}.
\end{remark}

\subsection{Contextual Linear Bandits}
\label{subsec:Context}

\paragraph{Problem Statement}
We consider the contextual bandit setting where contexts are drawn from a finite set {$\cX$}, and each context has finite arms in {$\cA$} giving rewards with means {$\{\mu_{x,a}\}_{(x,a) \in \cX \times \cA}$}. We assume that each context-arm pair $(x,a)$ is associated with a known feature {$\phi(x,a)$} in $\mathbb{R}^d$. We also have an available parametric class of functions that serves to (approximately) model the mean reward of each context-arm pair; each function in this class is of the form {$f_\theta: \mathbb{R}^d \to \mathbb{R}$} for a parameter {$\theta \in \Theta$,} and applying it to arm $a$ at context $x$ yields the expressed reward {$f_\theta(\phi(x,a)) \in \mathbb{R}$}. 
\paragraph{Notation} We denote the context space size and the action space size by {$\abs{\cX}$} and {$\abs{\cA}$}, respectively. We shall assume that the number of arms is larger than the dimension of the parameter, that is {$\abs{\cA} > d$} and, for ease of analysis, consider the set of features {$\{\phi(x,a)\}_{x \in \cX,a \in \cA }$} to span $\mathbb{R}^d$. We shall denote the true reward mean as a vector in {${\abs{\cX\cA}}$}-dimension,
{$\bm{\mu} = \begin{bmatrix}
\mu_{x,a}
\end{bmatrix}_{x \in \cX,a \in \cA } \in \Real^{\abs{\cX\cA}}$} and the feature matrix by {$\bm{\Phi} = \begin{bmatrix}
    \cdots \phi(x,a)^\top \cdots 
\end{bmatrix}_{x \in \cX,a \in \cA }$} an element in {$\Real^{\abs{\cX\cA} \times d}$}. \footnote{By abuse of notation we will write $\Real^{\abs{\cX\cA}}$ to denote the dimension of the product space $\Real^{\abs{\cX}\abs{\cA}}$.} We shall denote by {$\bm{\Phi_x}$}, the context specific feature matrix, as the {$\abs{\cA} \times d$} sub-matrix of {$\bm{\Phi}$} corresponding to the features {$[\phi(x,a)]_{a \in \cA }$} for a fixed context $x$. Similarly we shall denote by {$\bm{\mu_x}$}, the context specific reward vector, as the {$\abs{\cA}$} dimensional sub-vector of {$\bm{\mu}$} corresponding to the rewards {$[\mu_{x,a}]_{a \in \cA}$} for a fixed context $x$.
\begin{remark}[\textbf{Linear Contextual bandits}] In linear contextual bandits \citep{chu2011contextual}, it is assumed that the mean rewards for a context-action pair $\mu_{x,a}$ is a linear function of the features $\phi(x,a)$, that is there exists a $\theta^*$ such that $\bm{\mu} = \bm{\Phi}\theta^*$. Note that we make no such assumption.
\end{remark}


\paragraph{Technical Definitions} We define analogous concepts as those introduced in the bandits setting for the contextual setting. The critical observation is that we recover the multi-armed bandit setup for any fixed context. 
\begin{definition}[\textbf{Greedy Region $\mathcal{R}$ for context $x$}]
\label{def:context_greedy}
Define by {$\cR^x_a$}, for any context {$x$} and arm {$a$}, as the region in {$\mathbb{R}^{\abs{\cA}}$} for which the $a^{th}$ arm is the optimal arm at context $x$, that is 
{
$\cR^x_a \triangleq \Big\{ \bm{\mu_x} \in \Real^{\cA} : \mu_{x,a} > \mu_{x,b}\; \forall\; b \neq a \Big\}$.}
\end{definition}
For clarity, let us fix our model class to be linear, which results in the least square estimate to have a closed-form solution given by, 
{
$\hat{\theta}_{t+1} = (\bm{\Phi}^\top \Lambda \bm{\Phi})^{-1}\bm{\Phi}^\top \Lambda\bm{\hat{\mu}_t}\;\;,$}
where {$\Lambda$} is a diagonal matrix of size {$\abs{\cX\cA} \times \abs{\cX\cA}$} consisting of the normalized sample frequencies {$\{\alpha(x,a)\}_{(x,a) \in \cX \times \cA}$} in the {$\abs{\cX\cA}$} dimensional simplex {$\Delta_{\mathrm{XA}}$}, and {$\bm{\hat{\mu}_t}$} is a vector in {$\Real^{\abs{\cX\cA}}$} consisting of the sample estimates 
{$\hat{\mu}_{x,a,t}$}, the sample mean of the observations from context-arm pair {$(x,a)$} till time $t$. \footnote{Like in the bandit section, we shall assume $\bm{\Phi}\Lambda\bm{\Phi}$ is invertible, which can be ensured by a forced sampling of $d$ linearly independent features.}

We define the \emph{model estimate under a sampling distribution} as
\begin{definition}[\textbf{Model Estimate under Sampling Distribution}]
\label{def: projection_context}
    For any contextual bandit instance $\bm{\mu}$ in the {$\abs{\cX\cA}$} dimensional space, we shall denote the model estimate of the  element $\bm{\mu}$ under any sampling distribution {$\Lambda = \mathrm{diag}(\{\alpha(x,a)\}_{(x,a) \in \cX \times \cA})$}, as 
    {
    \begin{align*}
        \mathbf{P}^{\Lambda}(\bm{\mu}) \triangleq \argmin_\theta\Big\|\bm{\Phi}\theta - \bm{\mu}\Big\|_{\Lambda^{1/2}}\;,
    \end{align*}}
    where {$\{\alpha(x,a)\}_{(x,a) \in \cX \times \cA}$} belongs to the {$\abs{\cX\cA}$} dimensional simplex {$\Delta_{\mathrm{XA}}$}.  
\end{definition}


Analogous to the bandits setting, we define the robust regions, namely \emph{robust parameter region} and \emph{robust observation region}, for any arbitrary but fixed context $x$.
\begin{definition}[\textbf{Robust Parameter Region for context $x$}]
\label{def:robust_parameter_context}
For a given feature matrix $\bm{\Phi}$ and a context $x$, define the $a^{th}$ \emph{robust parameter region} {$\Theta^x_{a}$} as the set of all parameters {$\theta \in \Real^d$} such that the range space of the context-specific feature matrix, $\bm{\Phi_x}$, restricted to {$\Theta^x_{a}$} lies in the corresponding $(x,a)^{th}$ greedy region $\mathcal{R}^x_a$. That is,
{
$\Theta^x_a= \Big\{\theta \in \Real^d : \bm{\Phi_x}\theta \in \mathcal{R}^x_a \Big\}\;\;\;\textit{for any arm } a \in \cA$}.
\end{definition}
\begin{definition}[\textbf{Robust Observation Region for context $x$}]
\label{def:robust_observation_context}
For a given feature matrix {$\bm{\Phi}$} and a context $x$, we define the $a^{th}$ \emph{robust observation region} {$\cC^x_a$} as the set of all {$\abs{\cX\cA}$} dimensional contextual bandit instances $\bm{\mu}$ such that (i) the {$\abs{\cA}$} armed bandit problem at context $x$ has arm $a$ as the optimal and (ii) the model estimate {$\mathbf{P}^{\Lambda}(\bm{\mu})$} calculated under any sampling distribution {$\{\alpha(x,a)\}_{(x,a) \in \cX \times \cA}$} belongs in the $(x,a)^{th}$ robust parameter region {$\Theta^x_{a}$}. That is, 
{
\begin{equation*}
\cC^x_{a} = \Big\{\bm{\mu} \in \Real^{\abs{\cX\cA}} : \bm{\mu_x} \in \cR^x_{a} \text{ and } 
\mathbf{P}^\Lambda (\bm{\mu}) \in \Theta^x_{a}\;\; \forall\;\;  \Lambda \in \bm{\Lambda} \Big\}\,,    
\end{equation*}
}
where 
{
$
\bm{\Lambda} = \Big\{\Lambda = \mathrm{diag}(\{\alpha(x,a)\}_{(x,a) \in \cX \times \cA}) \text{ such that }
\{\alpha(x,a)\}_{(x,a) \in \cX \times \cA} \in \Delta_{\mathrm{XA}} \Big\}.    
$
}
\end{definition}
With the help of these definitions we arrive at a sufficient condition for robustness: any contextual bandit $\bm{\mu}$ that belongs in a robust observation region, the \emph{noiseless} model estimate, {$\mathbf{P}^{\Lambda}(\bm{\mu})$} -  computed under any sampling distribution - would cause the greedy algorithm to be consistently optimal at every context.
\begin{theorem}
\label{thm:context_sufficient}
    Define {$\cC^{\cX} \triangleq \bigcap_{x \in \cX}\cC^x_{\mathrm{OPT}(x)}$}\footnote{Here we use the notation $\mathrm{OPT}(x)$ to denote the optimal arm of context $x$}. If a contextual bandit instance $\bm{\mu}$ belongs in the robust region $\cC^{\cX}$ then {$\bm{\Phi_x} \mathbf{P}^\Lambda(\bm{\mu}) \in \cR^x_{\mathrm{OPT}(x)}\; \forall x \in \cX$} under any sampling distribution {$\{\alpha(x,a)\}_{(x,a)\in \cX\times\cA} \in \Delta_{\mathrm{XA}}$}.
\end{theorem}
\begin{proof}
    If {$\bm{\mu} \in \cC^{\cX}$}, then {$\bm{\mu} \in \cC^x_{\mathrm{OPT}(x)}$} for every context $x$.
    Therefore, from the definition of {$\cC^x_{\mathrm{OPT}(x)}$}, that {$\bm{\mu_x} \in \cR^x_{\mathrm{OPT}(x)}$} and {$\mathbf{P}^\Lambda(\bm{\mu}) \in \Theta^x_{\mathrm{OPT}(x)}$} for any sampling distribution {$\{\alpha(x,a)\}$} and for every context $x$. Thus, from the definition of {$\Theta^x_{\mathrm{OPT}(x)}$}, we have that {$\bm{\Phi_x}\mathbf{P}^\Lambda(\bm{\mu}) \in \cR^x_{\mathrm{OPT}(x)}$} for any sampling distribution {$\{\alpha(x,a)\}$} and for every context {$x \in \cX$}. Therefore, the greedy algorithm is guaranteed to play optimally in every context.      
\end{proof}
As a corollary we can show that the model estimate {$\mathbf{P}^\Lambda(\bm{\mu})$} computed under any sampling distribution {$\{\alpha(x,a)\}_{(x,a) \in \cX \times \cA}$} must belong to the robust parameter region {$\Theta^x_{\mathrm{OPT}(x)}$} for every context $x$. 
\begin{corollary}[Proof in Appendix \ref{sec:proofs}]
\label{corr:robust_parameter}
    For any contextual bandit instance $\bm{\mu}$, we have {$\bm{\mu} \in \cC^{\cX}$} if and only if the model estimate {$\mathbf{P}^\Lambda(\bm{\mu})$}, computed under any sampling distribution {$\{\alpha(x,a)\}_{(x,a) \in \cX \times \cA} \in \Delta_{\mathrm{XA}}$}, belongs to {$\Theta^{\cX}$}, where {$\Theta^{\cX} \triangleq \bigcap_{x \in \cX}\Theta^x_{\mathrm{OPT}(x)}$}.
\end{corollary}
We shall refer to {$\cC^{\cX}$} and {$\Theta^{\cX}$} as the \emph{robust observation region} and \emph{robust parameter region} respectively. 
\paragraph{Characterization of the Robust Observation Region} For the linear model class we can use Lemma \ref{lemma:forsgren} to characterize the \emph{Robust Observation Region}, {$\cC^{\cX}$}, in a manner analogous to the bandit setting in Theorem \ref{thm:bandit_convex_hull}. The proof is presented in Appendix \ref{sec:proofs}.
\begin{theorem}
\label{thm:context_convex_hull}
For any contextual bandit instance {$\bm{\mu}$} we have, {$\bm{\mu}$} belongs to the \emph{robust observation region} {$\cC^{\cX}$} if and only if for each {$d \times d$} full rank sub-matrix of {$\bm{\Phi}$}, denoted by {$\Phi_{d}$}, along with the corresponding {$d$} rows of {$\bm{\mu}$}, denoted as {$\bm{\mu}_{d}$}, satisfy {$\Phi_{d}^{-1}\bm{\mu}_d \in \Theta^{\cX}$}. That is 
for any {$\bm{\mu}$}
{
\begin{align*} 
    \bm{\mu} \in \cC^{\cX} \iff  \Phi_d^{-1}\bm{\mu}_d \in \Theta^{\cX}\; 
\end{align*}}
for all {$d\times d$} full rank sub-matrices of {$\bm{\Phi}$} (denoted as {$\Phi_d$}) and the corresponding $d$ rows of {$\bm{\mu}$} (denoted as {$\bm{\mu}_d$}).
\end{theorem}
\paragraph{Sufficient Condition} We show that under the assumption of sub-Gaussian noise and a context distribution where every context has a positive probability of observation $\epsilon$-greedy and LinUCB algorithm can achieve sub-linear regret. 
\begin{assumption}[\textbf{sub-Gaussian Observations}]
\label{assm:context_subg_noise}
    We shall assume that for any context $x$, the {$\abs{\cA}$} armed bandit instance {$\bm{\mu_x}$} is {$1/2$} sub-Gaussian.
\end{assumption}
\begin{assumption}[\textbf{Context Distribution}]
\label{assm: context_distr}
Each context {$x \in \cX$} has positive probability {$\mathbf{p}_x$} of observation.
\end{assumption}
Similar to the bandit setting, we require the instance to be an interior point of the robust observation region, and we define the open sets in the {$\abs{\cX\cA}$} dimensional space as open rectangles.
\begin{definition}[\textbf{Interior Point}]
\label{def:context_interior}
{$\bm{\mu}$} is an interior point of {$\cC^{\cX}$} if and only if there exists a {$\delta > 0$}, such the {$\abs{\cX\cA}$}-dimensional cell, with size {$\delta$} and centre {$\bm{\mu}$}, is a subset of {$\cC^{\cX}$}. That is,
{
\begin{align*}
    \bm{\mu} \in \mathrm{Int}(\cC^{\cX}) \iff \exists \; \delta>0\; s.t.\; \mathbf{1}_\delta(\bm{\mu}) \subset \cC^{\cX}
\end{align*}}
where {$\mathbf{1}_\delta(\bm{\mu})$} is the {$\abs{\cX\cA}$}-dimensional cell in {$\Real^{\abs{\cX\cA}}$} topology defined as, {
$\mathbf{1}_\delta(\bm{\mu}) \triangleq \{ \bm{\xi} \in \Real^{\abs{\cX\cA}}\,\, :\,\, \xi_i \in (\mu_i-\delta,\, \mu_i+\delta)\;\;\; \forall \;\; i \in [\abs{\cX\cA}]\}$.}   
\end{definition}
\paragraph{$\epsilon$-greedy Algorithm}The $\epsilon$-greedy algorithm is arguably the simplest and one of the most popular algorithms in bandits and reinforcement learning. We prove the following result.
\begin{theorem}[$\epsilon$-greedy, Proof in Appendix \ref{sec:proofs}]
\label{thm:context_eps_greedy}
For a given feature matrix {$\bm{\Phi}$} the $\epsilon$-greedy algorithm with {$\epsilon_t$} set as {$1/\sqrt{t}$} achieves a regret of {$O(\sqrt{T})$} for any contextual bandit instance belonging to the robust observation region.
\end{theorem}
In Appendix \ref{subsec:exp_context} we show how to use Theorem \ref{thm:context_convex_hull} to construct robust regions for a linear model. We sample instances from the robust regions and run $\epsilon$-greedy to corroborate our theory.
\paragraph{LinUCB} Our robustness definition is used to analyze LinUCB in the contextual setting as well. We show the following result.
\begin{theorem}[LinUCB, Proof in Appendix \ref{sec:linucb_context}]
\label{thm:linucb_context}
For a given feature matrix {$\bm{\Phi}$} the LinUCB algorithm achieves a regret of {$O(\sqrt{T})$} for any contextual bandit instance belonging to the robust observation region.
\end{theorem}

\subsection{Misspecification in Markov Decision Processes (MDPs)} \label{subsec:MDPs}
\paragraph{Problem Statement}
We consider a finite horizon episodic MDP setting where at horizon/stage, {$h \in [H]$}, states are drawn from a finite set {$\cS$}, and each state has finite actions in {$\cA$}. The reward function at any stage {$h \in [H]$} is a deterministic function of the state and action, that is, {$R_h:\cS \times \cA \to \Real$}. The state transition kernel at any stage {$h \in [H]$} is denoted by {$\cP_h(s'|s, a)$} and the actions are chosen according to a behavioral policy {$\pi_b: \cS \to \Delta_{\cA}$}. The optimal {$Q_h^*$} value at any stage {$h \in [H]$} is a function of the state-action pair, that is, {$Q_h^* : \cS \times \cA \to \Real$} and satisfies the optimal Bellman equation {$Q^*_h(s,a) = T_hQ^*_{h+1}(s,a)$ where {$T_hQ^*_{h+1}(s,a)$}} is the Bellman Operator, {$R_h(s,a) + \mathbf{E}[\max_{a' \in \cA} Q^*_{h+1}(s',a') | s_h = s, a_h = a]$}. 
\paragraph{Function Approximation} We use a class of parameterized functions to approximate the optimal {$Q^*_h$} values. Specifically, we have $H$ parameters {$\{\theta_h\}_{h=1}^H$} and functions $f$ drawn from a function class {$\cF_h$} parameterized by {$\{\theta_h\}$} for {$h \in [H]$}, that is, {$
    \cF_h = \{f_{\theta_h} : \cS \times \cA \to \Real \;\;\;\;\forall\;\;\; \theta_h \in \Real^{d_h} \}$}.
We will expose the main technical details here and defer a more detailed study for general nonlinear functions to Appendix \ref{sec:non_lin_app}. We abbreviate by {$f_{\theta_h}$} to be an element an {$\cF_h$} parameterized by {$\theta_h$}. We shall use {$f_{\theta_h}$} to approximate the optimal Q-value, {$Q^*_h$}. 
Consider fitted-Q iteration (Algorithm \ref{alg:mdp_algorithm}) \citep{szepesvari2022algorithms} for learning the MDP. At each stage $h$, it uses data collected by a behavior policy to estimate, using least-squares, the approximate Q-value for the stage. 
{\begin{algorithm}[htbp]
\caption{Fitted-Q Learning}
\label{alg:mdp_algorithm}
\textbf{Input}: Behavioral Policy $\pi_b$\\
\textbf{Output}: Updated parameters $\{\hat{\theta}_h\}_{h=1}^H$ after $T$ rounds
\begin{algorithmic}[1] 
\FOR{episode $t = 1$ to $T$}
\STATE Set $\hat{\theta}_{H+1} = 0$
\FOR{Horizon $h = H$ to $1$}
\STATE Fit $Q$-function with least squares regression 
\begin{align*}
    \hat{\theta}_h = \arg\min_\theta \hspace{-2mm}\sum_{(s_h,a_h,r_h,s_{h+1}) \in \cD_h} \hspace{-4mm}\left( f_\theta(s_{h},a_{h}) - r_h - \max_{a}f_{\hat{\theta}_{h+1}}(s_{h+1},a)\right)^2
\end{align*} 
\ENDFOR
\STATE Sample one episode $(s_1,a_1,r_1,\cdots, s_H,a_H,r_H)$ using $\pi_b$
\STATE Update the observation dataset $\cD_h \leftarrow \cD_h \cup \{(s_h, a_h, r_h, s_{h+1})\}$ for all $h \in [H]$.
\ENDFOR
\end{algorithmic}
\end{algorithm}}
We shall show that under robust conditions defined analogously as in the previous sections, the greedy policy after {$T$} rounds will be the optimal policy with high probability. Note that we make no realizability assumption as introduced in previous works. Particularly, we do not assume {$Q^*_h \in \cF_h$} for any {$h$}, nor do we assume that the Bellman Operator satisfies the completeness property, in the sense that {$T_hf_{h+1} \in \cF_h$} for all {$f_{h+1} \in \cF_{h+1}$}. 
\paragraph{Main Results} We rely on the ideas we developed for Contextual Bandits, which are $1$-stage MDPs. In general, note that any {$Q_h^*(x, \;\; \cdot)$} is an element in a {$\Real^{\cA}$} dimensional space and hence belongs to the greedy region {$\cR^s_{\mathrm{OPT}(s)}$} (See the definition of the greedy regions, in particular Definition \ref{def:context_greedy} for reference). The \textbf{model-$h$ estimate under the sampling distribution}\, is defined analogously as in the previous sections as 
{
\begin{align*}
    \mathbf{\Lambda}^{\pi_b}_h(\theta') = \arg\min_\theta \sum_{s,a,s'}\alpha^{\pi_b}_h(s,a,s')\Big(f_\theta(s,a) - r - \max_{a'}f_{\theta'}(s',a')\Big)^2\;, 
\end{align*}}\footnote{We use the symbol $\mathbf{\Lambda}$ to denote the model estimate instead of the previous symbol $\mathbf{P}$ so as not cause confusion between the probability symbols. }
for any {$\theta'$}, where {$\alpha^{\pi_b}_h(s,a,s')$} is the true distribution of observing the pair {$(s,a,s')$} under the behavioral policy {$\pi_b$} at stage $h$, denoted as, {$\alpha^{\pi_b}_h(s,a,s') = \mathbf{P}^{\pi_b}\{s_h = s, a_h = a, s_{h+1} = s'\}$}. Note that, unlike the previous sections, where the sampling strategy was rather arbitrary, we have fixed a behavioral policy for ease of exposition in this section. Thus, the sampling strategy is fixed by the behavioral policy {$\pi_b$}. We define the \textbf{robust parameter-$h$ region for a state $s$}\, as {$\Theta^s_h = \Big\{\theta \,:\, f_\theta(s,\; \cdot) \in \cR^s_{\mathrm{OPT}(s)}\Big\}\;$,} and the \textbf{robust parameter-$h$ region } as {$\Theta_h = \cap_{s \in \cS} \Theta^s_h$}. The \textbf{robust condition} can thus be described as 
{
\begin{align}
\label{cond:robust}    \mathbf{\Lambda}^{\pi_b}_h(\theta_{h+1}) \in \Theta_h\;\; \forall\; \theta_{h+1}\in \Theta_{h+1}\,, \text{ for all } h \in [H].
\end{align}}
 To account for the inherent stochasticity of the MDP, we include an interior point assumption analogous to the previous sections (e.g., Definition \ref{def:context_interior}):
\begin{assumption}
\label{assm:interior_mdp}
    If the MDP satisfies the robust condition \ref{cond:robust} under a behavioral policy {$\pi_b$}, then there exists a {$\delta > 0$} such that for any policy {$\pi$} satisfying 
    { $\mid\alpha_h^{\pi}(s, a, s') - \alpha_h^{\pi_b}(s, a,s')\mid \leq \delta \;\;\; \forall \;\; (s,a,s') \textit{ and } h \in [H]$,}
we have, 
{
$
    \mathbf{\Lambda}^{\pi}_h(\theta_{h+1}) \in \Theta_h\;\; \forall\; \theta_{h+1}\in \Theta_{h+1}\;\;\; \text{ for all } h \in [H]\;.
$}
\end{assumption}
This assumption supposes that for any arbitrary MDP, which is robust under the true behavioral policy (that is satisfies Condition \ref{cond:robust}), any empirical distribution of the observations {$(s,a, s')$} which is {$\delta$}-close to the true distribution {$\pi_b$} continues to satisfy the robust condition.
\begin{theorem}
    Given that a MDP satisfies the robust condition \ref{cond:robust} and the interior point Assumption \ref{assm:interior_mdp} with parameter $\delta$, then for any {$\epsilon>0$} and for {$T \geq \ln\Big(\frac{\cS^2\cA H}{\epsilon}\Big)\frac{1}{2\delta^2}$} with probability more than {$1-\epsilon$}, the greedy policy, defined as {$\pi^{\text{greedy}}_h(s) = \arg\max_{a}f_{\hat{\theta}_T}(s,a)$} is the optimal policy {$\arg\max_{a\in \cA}Q^*(s,a)$}.
\end{theorem}
\begin{proof}
    Let {$n_h(s,a,s',t)$} denote the number of times the transition {$(s_h = s, a_h = a, s_{h+1} = s')$} is observed till time {$t$} under the behavioral policy {$\pi_b$}. Since every trajectory is sampled independently we have {$\expect{\frac{n_h(s,a,s',T)}{T}} = \expect{\frac{\sum_{i=1}^T \mathbb{1}\big\{s_{hi} = s, a_{hi} = a, s_{h+1,i} = s'\big\}}{T}} = \alpha^{\pi_b}_h(s,a,s')$}. Thus from Hoeffding's Inequality, we get %
    { 
    $        \mathbf{P}\Big\{ \Big\lvert \frac{n_h(s,a,s',T)}{T} - \alpha^{\pi_b}_h(s,a,s') \Big\rvert > \delta \Big \} \;\;\leq\;\; 2\exp(-2\delta^2T)\,.$}
Taking a uniform bound over all {$(s,a,s')$} observations and all {$H$} stages, we find that the Assumption \ref{assm:interior_mdp}  is not satisfied with probability less than {$\cS^2\cA H\exp\big(-2\delta^2T\big)$}. 
\end{proof}
\subsubsection{Example of a Misspecified but Highly Robust MDP}
\label{sec:example_mdp}
\begin{figure}[htbp]
\hspace{-2mm}
    \begin{subfigure}[\footnotesize{A function approximation feature class which is described as an $\epsilon$-radius tube about the diagonal. We give a representative diagram for a $\Real^2$ space corresponding to a bandit problem with two arms. We see that except a measure zero set of bandit instances on the diagonal, which can be interpreted as both arms having the same rewards, all instances are robust.} \label{fig : features}]{\includegraphics[width = 0.4\linewidth]{diagram-2.png}}
    \hfill
    \end{subfigure}
    \begin{subfigure}[\footnotesize{A $2$ stage deterministic MDP with three states and each state having two actions. The rewards are ordered as $r_{11}> r_{12},\;\; r_{21}>r_{22},\;\; r_{31}>r_{32}$. However, the rewards are designed such that the optimal action in state $s_1$ is $a_2$, because $r_{31}$ is significantly higher than $r_{21}$. } \label{fig : examples}]
    {\includegraphics[width = 0.5\linewidth]{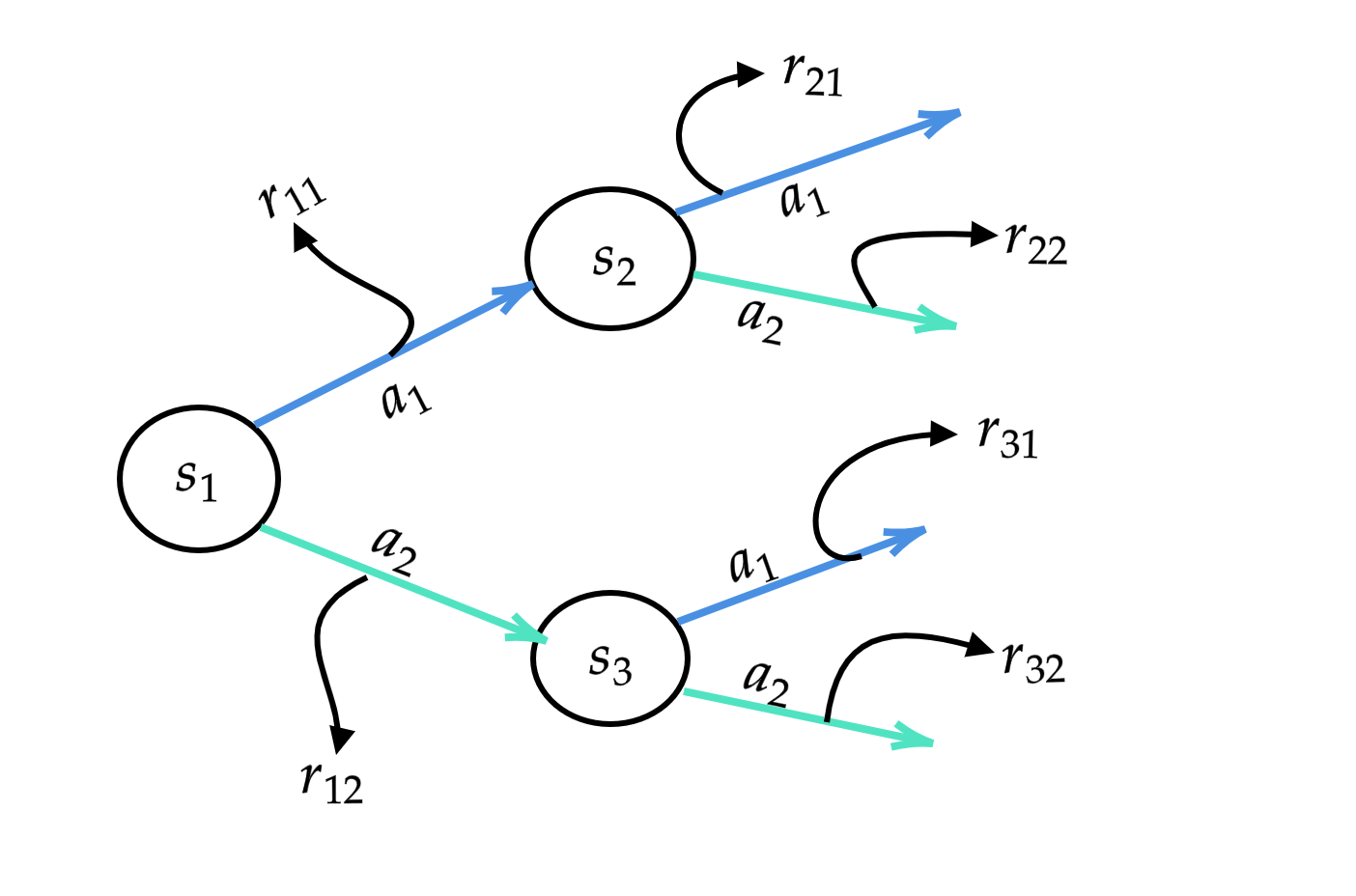}}
    \end{subfigure}
    \caption{\footnotesize{An example of a MDP and a function class we designed to approximate the $Q$ value. The optimal $Q^*$ values are misspecified in the function class, yet we can learn the optimal policy using this function class. }}
\end{figure}

In Appendix \ref{sec:robust_bandit} and \ref{sec:robust_context} we propose a feature class with a large, robust region for both the bandit and the contextual bandit setting. Here we present a self-contained example of a feature class and a simple two-stage deterministic MDP as an example to highlight our ideas. 
The function class {$\cF$} we choose for this example is as a tube with radius {$\epsilon$} about the diagonal of the {$\Real^{\cS\cA}$} space. In a two-armed bandit scenario, the function class is shown in Figure \ref{fig : features}. The figure shows that the robust examples for this function class are almost all of {$\Real^2$} except instances on the diagonal. We present a simple two-stage MDP, with three states, each having two actions in Figure \ref{fig : examples} as an example. The state transitions are deterministic based on the actions. At stage {$h=1$}, the process starts at state {$s_1$} and, based on the action, gets the associated reward and moves on to either state {$s_2$} or state {$s_3$}. At each of the subsequent states, one chooses one of two available actions again, observes the reward, and the process ends. We design the rewards such that { $r_{11}> r_{12}$} at state {$s_1$}, { $r_{21} > r_{22}$} at state {$s_2$}, and { $r_{31} > r_{32}$} at state {$s_3$}. The {$Q^*$} values for state {$s_1$}, {$s_2$} and {$s_3$} are {$r_{12}+r_{31}$}, {$r_{21}$} and {$r_{31}$} respectively. The optimal policy is described as {$\pi(s_1) = a_2$}, {$\pi(s_2) = a_1$} and {$\pi(s_3) = a_1$}. Note that employing a myopic strategy and playing greedily at the first stage leads to a suboptimal policy. At stage {$h=2$}, there are four possible state-action pairs, {$(s_2,a_1)$}, {$(s_2, a_2)$}, {$(s_3, a_1)$} and {$(s_3, a_2)$}. Thus at stage {$h=2$}, the reward vector {$\overrightarrow{r_2} = [r_{21},\;\; r_{22},\;\; r_{31},\;\; r_{32}]^\top$} is an element of {$\Real^4$}. The function approximation class is defined as 
{ 
\begin{align*} \cF_2 = 
        \Big\{[x, \;\; x,\;\; x,\;\; x\;\;]^\top + \epsilon \frac{\overrightarrow{v}}{\|\overrightarrow{v}\|_2} \;\;\;\;\forall\;\;\; \overrightarrow{v} \in \big([1,\;\; 1,\;\; 1,\;\; 1]^\top\big)^\perp \;\;\;\;\forall\;\;\; x \in \Real\;\Big\}, \text{ for a fixed } \epsilon > 0.
\end{align*}}
 This describes an {$\epsilon$}-radius tube about the diagonal. Assuming the behavior policy that samples uniformly samples actions at all states, the function approximated value {$f_2 \in \cF_2$} of the reward vector {$\overrightarrow{r_2}$} at stage {$h=2$} is {$\overrightarrow{x_0} + \epsilon\frac{(\overrightarrow{r_2}-\overrightarrow{x_0})}{\|\overrightarrow{r_2}-\overrightarrow{x_0}\|}$} with {$\overrightarrow{x_0} = [x_0,\;x_0,\;x_0,\;x_0,]^\top$} and {$x_0 = (r_{21}+r_{22}+r_{31}+r_{32})/4$} calculated as the orthogonal projection of the point {$\overrightarrow{r_2}$} onto the diagonal. Note that at stage {$h=2$} one can have a potential huge misspecification error of $l_2$ norm approximately {$\|\overrightarrow{r_2} - \overrightarrow{x_0}\|$}. The function approximated $Q$ values functions for the states can be read off, as {$\begin{bmatrix}
     f_{2}(s_2 ,\;a_1) \\
     f_{2}(s_2 ,\;a_2)
 \end{bmatrix} = \begin{bmatrix}x_0 \\ x_0\end{bmatrix} + \frac{\epsilon}{\|\overrightarrow{r_2}-\overrightarrow{x_0}\|}\begin{bmatrix}
     r_{21}-x_0  \\
     r_{22}-x_0 
 \end{bmatrix}$} and {$\begin{bmatrix}
     f_{_2}(s_3, \; a_1) \\
     f_{_2}(s_3, \; a_2)
 \end{bmatrix} = \begin{bmatrix}
     x_0\\
     x_0
 \end{bmatrix} + \frac{\epsilon}{\|\overrightarrow{r_2}-\overrightarrow{x_0}\|}\begin{bmatrix}
     r_{31}-x_0 \\
     r_{32}-x_0
 \end{bmatrix}$}. Note that because {$r_{21}> r_{22}$} and {$r_{31}> r_{32}$} we have, {$f_{2}(s_2 , a_1) > f_{2}(s_2 , a_2)$} and {$f_{2}(s_3 , a_1) > f_{2}(s_3 , a_2)$}, and thus at stage {$h=2$} the state-action pair satisfy the robust condition, that is the greedy policy {$\arg\max_a f_2(s_2, a)$} is {$a_1$}, the optimal policy at state {$s_2$} and similarly for {$s_3$}. At stage {$h=1$} the target of the value function for the regression problem for the transition {$(s_1, a_1)$} is {$r_{11} + x_0 + \frac{\epsilon}{\|\overrightarrow{r_2}-\overrightarrow{x_0}\|}(r_{21}-x_0)$} and for the transition {$(s_1, a_2)$} is { $r_{12}+ x_0 + \frac{\epsilon}{\|\overrightarrow{r_2}-\overrightarrow{x_0}\|}(r_{31}-x_0)$}. Thus, at stage {$h=1$}, we have a single state and two actions with two bootstrapped targets, {$\overrightarrow{r_1} = \begin{bmatrix}
     r_{11} + \max_a f_2(s_2,\; a) \\
     r_{12} + \max_a f_2(s_3,\; a)
 \end{bmatrix}$} which is, based on the value of {$f_2$} given as {$\begin{bmatrix}
     r_{11} + x_0 + \frac{\epsilon}{\|\overrightarrow{r_2}-\overrightarrow{x_0}\|}(r_{21}-x_0) \\
     r_{12}+ x_0 + \frac{\epsilon}{\|\overrightarrow{r_2}-\overrightarrow{x_0}\|}(r_{31}-x_0)
 \end{bmatrix}$}. At stage {$h=1$}, we choose a similar function class but now in {$\Real^2$} as shown in Figure \ref{fig : features}. Specifically, we have a {$\delta$} (possibly the same as the one chosen at stage {$h=2$}) radius tube about the diagonal as the function approximating class. Note that any {$\overrightarrow{r_1}$} not on the diagonal is robust. Thus, in order for the optimal policy at state {$s_1$} to be {$a_1$} we must have {$r_{11} + x_0 + \frac{\epsilon}{\|\overrightarrow{r_2}-\overrightarrow{x_0}\|}(r_{21}-x_0) < r_{12}+ x_0 + \frac{\epsilon}{\|\overrightarrow{r_2}-\overrightarrow{x_0}\|}(r_{31}-x_0)$} which can be simplified as {$r_{11} < r_{12} + \frac{\epsilon}{\|\overrightarrow{r_2}-\overrightarrow{x_0}\|}(r_{31}-r_{21})$}. Thus for {$r_{31}> r_{21}$}, we have a robust MDP. The analysis for the stochastic variant of this example is presented in Appendix \ref{sec:robust_mdp}. This requires the function approximation class described in Appendix \ref{sec:robust_context}.

\section{Conclusion}
In this work, we present a systematic study of model misspecified instances in the Bandit, Contextual Bandit, and Markov-Decision-Process settings which are acquiescent to \emph{standard} algorithms. Previous theoretical works have all indicated that these standard algorithms are not enough to learn the optimal policy without the realizability assumption, despite their panacean use in practical experiments. In our work, we drop the realizability assumption while systematically theorizing the existence, and identifying instances that, though unrealized, are not antagonistic to the learning algorithms. On this note, it is to be observed that we give instances that can be highly misspecified, and yet the model class can learn the optimal policy. Our work is not an attempt to provide insights into how one can better utilize the concept of robust regions to construct better features or better algorithms. Rather, it is a theoretical understanding through geometry, the interplay between model misspecification, and the essence of learning algorithms. We construct a theory that can explain our observations and subsequently build upon the theory to encompass a wide array of sequential decision-making problems. 


\bibliographystyle{plainnat}
\bibliography{aaai24}

\appendix
\onecolumn
\section{Description of Algorithms and Proofs of Theorems}
\label{sec:proofs}
In this section we provide the proofs of all theorems  presented in the main paper.

\subsection{Bandits}

\paragraph{$\epsilon$-greedy} $\epsilon$-greedy algorithm is a popular forced-exploration based algorithm widely used in practice. We describe it in Algorithm \ref{alg:bandit_algorithm} for the sake of completeness.
\begin{algorithm}[htbp]
\caption{$\epsilon$-greedy algorithm}
\label{alg:bandit_algorithm}
\begin{algorithmic}[1] 
\FOR{t = 1 to T}
\STATE With an estimate $\hat{\theta}_t$, play arm $i$ such that 
\begin{align*}
    A_t &= \argmax_{i \in [K]} \phi_i^\top\hat{\theta}_t \;\mathrm{ w.p. }\;1-\epsilon_t \\
    &= \textit{play uniformly over } K \textit{ arms}  \;\mathrm{ w.p. }\;\epsilon_t
\end{align*} 
\STATE Observe the reward $Y_t$.
\STATE Update the estimate as 
\begin{align*}
    \hat{\theta}_{t+1} = \argmin_{\theta} \sum_{s=1}^t[\phi_{A_s}^\top \theta-Y_s]^2
\end{align*}
\ENDFOR
\end{algorithmic}
\end{algorithm}
We give a detailed description and proof of the result presented in Theorem \ref{thm:bandit_eps_greedy}. 
\begin{theorem}[\textbf{Sufficient Condition for Zero Regret in Misspecified Bandits}]
\label{thm:bandit_eps_greedy_app}
For a feature matrix $\bm{\Phi}$, and it's associated robust observation region $\cC_k$, any $\frac{1}{2} \text{ sub-Gaussian}$ bandit parameter $\bm{\mu}$ which is an interior point of the robust observation region, that is $\bm{\mu} \in \mathrm{Int}(\cC_k)$,  
the $\epsilon$-greedy algorithm with $\epsilon_t$ set as $\frac{1}{\sqrt{t}}$, achieves a regret of $O(\Delta_{\max}\sqrt{T})$, where $\Delta_{\max}$ is the largest sub-optimal gap, that is $\Delta_{\max} = \max_{i \in [K]} \mu_k - \mu_i$.
\end{theorem}
\begin{remark}
    The proof essentially uses the same technique as presented in \citet{auer2002finite}. The key observation is that the least squares estimate $\hat{\theta}_t$, in our notation $\mathbf{P}^{\Lambda}(\bm{\hat{\mu}_t})$, is guaranteed to generate optimal play under the greedy strategy, if the sample mean estimate $\bm{\hat{\mu}_t}$ falls inside the robust region $\cC_k$. The concentration of sub-Gaussian random variables ensure that given enough samples $\bm{\hat{\mu}_t}$ will fall within a $\delta$-neighbourhood of the true rewards $\bm{\mu}$. The fact that $\bm{\mu}$ is an interior point of $\cC_k$, ensures a that there exists $\delta$- neighbourhood about $\bm{\mu}$ that is contained in the robust region $\cC_k$. 
\end{remark}
\begin{proof}
    Since $\bm{\mu}$ is an interior point of the robust region $\cC_k$, by Definition \ref{def:bandit_interior}, there exists a $\delta$-cell centred at $\bm{\mu}$, denoted as $\mathbf{1}_\delta(\bm{\mu})$ which is contained in the robust observation region $\cC_k$. We observe that the probability of choosing a sub-optimal arm $i$ at $t^{th}$ round is either due to a random play, which happens with probability $\frac{\epsilon_t}{K}$ or during the greedy play, when the reward estimate $\bm{\hat{\mu}_t}$ does not belong to $\mathbf{1}_\delta(\mu)$, the $\delta$ cell of $\bm{\mu}$. That is,
    \begin{align*}
        \prob{A_t = i} \leq \frac{\epsilon_t}{K} + (1-\frac{\epsilon_t}{K})\prob{\bm{\hat{\mu}_t} \notin \mathbf{1}_\delta(\bm{\mu})}.
    \end{align*}
    Now $\prob{\bm{\hat{\mu}_t} \notin \mathbf{1}_\delta(\bm{\mu})}$, from the definition of the $\delta$-cell, can be upper bounded by taking a union bound over all arms to give
    \begin{align*}
        \prob{\bm{\hat{\mu}_t} \notin \mathbf{1}_\delta(\bm{\mu})} \leq \sum_{i=1}^K\prob{|\hat{\mu}_i^{n_{i,t}}-\mu_i| \geq \delta}\;,
    \end{align*}
where $\hat{\mu}_i^{n_{i,t}}$ is the sample estimate of arm $i$, having played $n_{i,t}$ times till time $t$. The remainder of the proof has the same flavour as in the work of \citet{auer2002finite}. We provide here for the sake of completeness.
Note that each term $\prob{|\hat{\mu}_i^{n_{i,t}}-\mu_i| \geq \delta}$ can be bounded in the following manner,
\begin{align*}
    &\mathbb{P}\Big\{|\hat{\mu}_i^{n_{i,t}}-\mu_i| \geq \delta\Big\} = \sum_{s=1}^t\mathbb{P}\Big\{|\hat{\mu}_i^s-\mu_i| \geq \delta, n_{i,t}=s\Big\} \\
    &=\sum_{s=1}^t\mathbb{P}\Big\{n_{i,t}=s \mid \abs{\hat{\mu}_i^s-\mu_i} \geq \delta\Big\}\mathbb{P}\Big\{|\hat{\mu}_i^s-\mu_i| \geq \delta \Big\} \\
    &\leq \sum_{s=1}^t\mathbb{P}\Big\{n_{i,t}=s \mid \abs{\hat{\mu}_i^s-\mu_i} \geq \delta\Big\}2\exp{(-2s\delta^2)},
\end{align*}
where we use the sub-Gaussian concentration by Assumption \ref{assm:bandit_subg_noise} and Lemma \ref{lemma:subG_conc},
\begin{align*}
    &\leq\sum_{s=1}^{t_0}\mathbb{P}\Big\{n_{i,t}=s \mid \abs{\hat{\mu}_i^s-\mu_i} \geq \delta\Big\} + \sum_{s = t_0 + 1}^t2\exp{(-2s\delta^2)} \\
    &\leq \sum_{s=1}^{t_0}\mathbb{P}\Big\{n_{i,t}=s \mid \abs{\hat{\mu}_i^s-\mu_i} \geq \delta\Big\} + \frac{1}{\delta^2}\exp{(-2t_0 \delta^2)}\;, 
\end{align*}
where we use the identity $\sum_{t = x+1}^{\infty}\exp(-\kappa t) \leq \frac{1}{\kappa}\exp(-\kappa x)$. Let $n_{i,t}^R$ be the number of times arm $i$ has been played randomly till time $t$, then, 
\begin{align*}
    &\sum_{s=1}^{t_0}\mathbb{P}\Big\{n_{i,t}=s \mid \abs{\hat{\mu}_i^s-\mu_i} \geq \delta\Big\} \\
    &\leq \sum_{s=1}^{t_0}\mathbb{P}\Big\{n_{i,t}^R \leq s \mid \abs{\hat{\mu}_i^s-\mu_i} \geq \delta\Big\} \\
    &\leq \sum_{s=1}^{t_0}\mathbb{P}\{n_{i,t}^R \leq s \} \\
    &\leq t_0 \mathbb{P}\{n_{i,t}^R \leq t_0 \}.
\end{align*}

Now
\begin{align*}
\expect{n_{i,t}^R} = \frac{1}{K}\sum_{s=1}^t \epsilon_s     
\end{align*}
and variance of $n_{i,t}^R$ is 
\begin{align*}
    \mathbb{V}[n_{i,t}^R] = \sum_{s=1}^t \frac{\epsilon_s}{K}\Big( 1- \frac{\epsilon_s}{K} \Big) \leq \sum_{s=1}^t \frac{\epsilon_s}{K}.
\end{align*}
Thus, choosing $t_0 = \frac{1}{2K}\sum_{s=1}^t \epsilon_s$, we have 
\begin{align*}
    \mathbb{P}\{n_{i,t}^R \leq t_0 \} \leq \exp(-t_0/5)\;,
\end{align*}
from Bernstein's Inequality, Lemma \ref{lemma:Bernstein}.
Thus putting it all together we have,
\begin{align*}
    \prob{A_t = i} \leq \frac{\epsilon_t}{K} + Kt_0\exp(-t_0/5) + \frac{K}{\delta^2}\exp{(-2t_0 \delta^2)}.
\end{align*}
Thus, in order to complete the proof we only need to find a lower bound on $t_0$.
As per our definition, 
\begin{align*}
    t_0 &= \frac{1}{2K}\sum_{s=1}^t \epsilon_s.
\end{align*}
Plugging the value of $\eps_s = 1/\sqrt{s}$, we get
\begin{align*}
    t_0 &= \frac{1}{2K}\sum_{s=1}^t \frac{1}{\sqrt{s}}\\
    &\geq \frac{1}{K}(\sqrt{t+1} - 1).
\end{align*}
Thus, 
\begin{align*}
    \prob{A_t = i} \leq \frac{1}{K\sqrt{t}} + (\sqrt{t+1} - 1)\exp(-\frac{\sqrt{t+1} - 1}{5K}) + \frac{K}{\delta^2}\exp{\Big(-\frac{2\delta^2(\sqrt{t+1}-1)}{K}\Big)}.
\end{align*}
Therefore regret at any time $t$ is
\begin{align*}
    \mathrm{Regret_t} = \sum_{i=1}^K \Delta_i \prob{A_t = i} \leq \frac{\Delta_{\max}}{\sqrt{t}} + o(1/t),
\end{align*}
where $\Delta_i$ is the sub-optimality gap $\mu_k - \mu_i$.
Thus cumulative regret is $\sum_{t=1}^T \mathrm{Regret}_t = O(\Delta_{\max}\sqrt{T})$.
\end{proof}

\begin{remark}
    Note that our regret guarantee in Theorem \ref{thm:bandit_eps_greedy_app} does not depend upon the $l_\infty$ misspecification error (or any measure of misspecification) and depends only on the suboptimality gap. Our experiments illustrated in Appendix \ref{subsec:exp_bandits} corroborate this.  
\end{remark}

The reason that we demand explicitly for $\bm{\mu}$ to be an interior point of $\cC_k$, is because $\cC_k$is not necessarily an open set, as shown in Proposition \ref{thm:G_delta_app}.
\begin{proposition}
\label{thm:G_delta_app}
    The \emph{robust observation region} $\cC_k$ is a $G_\delta$ set, that is a countable intersection of open sets.
\end{proposition}

\begin{proof}
    Note that the for any arbitrary but fixed sampling distribution $\{\alpha(i)\}_{i = 1}^K$ belonging in the $K$ dimensional simplex $\Delta_{\mathrm{K}}$, we have
    \begin{align*}
        \Big\{ \bm{\mu} \in \cR_k : \mathbf{P}^{\Lambda}(\bm{\mu}) \in \Theta_k \Big \}\;,
    \end{align*}
    is an open set as the set $\Theta_k$ is an open set and the projection operator is continuous.
    Now, since any sampling distribution $\{\alpha(i)\}_{i = 1}^K$ are rationals, that is $\{\alpha(i)\}_{i = 1}^K \in \mathbb{Q}^\mathrm{K}$, we have
    \begin{align*}
        \cC_k \triangleq \bigcap_{\{\alpha(i)\}_{i = 1}^K \in \Delta_{\mathrm{K}}} \Big\{ \mu \in \cR_k : \mathbf{P}^{\Lambda}(\bm{\mu}) \in \Theta_k \Big \}\;,
    \end{align*}
    is a $G_\delta$ set.
\end{proof}

\subsection{Contextual Bandits}

The $\epsilon$-greedy algorithm in the contextual setup is described in Algorithm \ref{alg:bandit_algorithm_context} for the sake of completeness.
\begin{algorithm}[htbp]
\caption{Generic $\epsilon$-greedy algorithm}
\label{alg:bandit_algorithm_context}
\begin{algorithmic}[1] 
\FOR{t = 1 to T}
\STATE Observe context $X_t$ at time $t$
\STATE With an estimate $\hat{\theta}_t$, play arm $A_t$ such that 
\begin{align*}
    A_t &= \argmax_{a \in \cA} \phi(X_t,a)^\top \hat{\theta}_t \;\mathrm{ w.p. }\;1-\epsilon_t \\
    &= \textit{play uniformly over } \cA \textit{ arms}  \;\mathrm{ w.p. }\;\epsilon_t
\end{align*} 
\STATE Observe the reward $Y_t$.
\STATE Update the estimate as
\begin{align*}
 \hat{\theta}_{t+1} = \argmin_{\theta} \sum_{s=1}^t[\phi(X_s,A_s)^\top \theta-Y_s]^2.   
\end{align*}

\ENDFOR
\end{algorithmic}
\end{algorithm}
We now give a detailed description and proof of the result presented in Theorem \ref{thm:context_eps_greedy}.
\begin{theorem}[\textbf{Sufficient Condition for Zero Regret in Misspecified Contextual Bandits}]
\label{thm:context_eps_greedy_app}
For a given feature matrix $\bm{\Phi}$, any $\frac{1}{2} \text{ sub-Gaussian}$ contextual bandit instance $\bm{\mu}$ which is an interior point of the robust observation region, that is $\bm{\mu} \in \mathrm{Int}(\cC^{\cX})$,  
the $\epsilon$-greedy algorithm with $\epsilon_t$ set as $\frac{1}{\sqrt{t}}$, achieves a regret of $O(\Delta_{\max}\sqrt{T})$, where $\Delta_{\max} = \max_{x \in \cX, a \in \cA}\Delta_{x,a}$ and $\Delta_{x,a} = \mu_{x, \mathrm{OPT}(x)} - \mu_{x,a}$
\end{theorem}
\begin{remark}
    The proof essentially uses the same technique as presented in the proof for the bandits section. The key observation is that the least squares estimate $\hat{\theta}_t$, in our notation $\mathbf{P}^{\Lambda}(\bm{\hat{\mu}_t})$, is guaranteed to generate optimal play under the greedy strategy, if the sample mean estimate $\bm{\hat{\mu}_t}$ falls inside the robust region $\cC^{\cX}$. The concentration of sub-Gaussian random variables ensure that given enough samples $\bm{\hat{\mu}_t}$ will fall within a $\delta$-neighbourhood of the true rewards $\bm{\mu}$. The fact that $\bm{\mu}$ is an interior point of $\cC^{\cX}$, ensures a that there exists $\delta$- neighbourhood about $\bm{\mu}$ that is contained in the robust region $\cC^{\cX}$.
\end{remark}
\begin{proof}
Since $\bm{\mu}$ is an interior point of the robust region $\cC^{\cX}$, by Definition \ref{def:context_interior}, there exists a $\delta$-cell centred at $\bm{\mu}$, defined as $\mathbf{1}_\delta(\bm{\mu})$ which is contained in the robust observation region $\cC^{\cX}$.
    The probability of choosing a sub-optimal arm $a$ in state $x$ at $t^{th}$ round is either due to a random play, which happens with probability $\frac{\epsilon_t}{\abs{\cA}}$ or, during the greedy play, the reward estimate $\bm{{\hat{\mu}_t}} \in \Real^{\abs{\cX\cA}}$ does not belong to the $\delta$ cell of $\bm{\mu}$, $\mathbf{1}_\delta(\bm{\mu})$. That is,
    \begin{align*}
        \prob{A_t = a \mid X_t = x} \leq \frac{\epsilon_t}{\abs{\cA}} + (1-\frac{\epsilon_t}{\abs{\cA}})\prob{\bm{\hat{\mu}_t} \notin \mathbf{1}_\delta(\bm{\mu})}.
    \end{align*}
    Now $\prob{\bm{\hat{\mu}_t} \notin \mathbf{1}_\delta(\bm{\mu})}$, from the definition of the $\delta$-cell, can be upper bounded by taking a union bound over all arms and contexts to give
    \begin{align*}
    \prob{\bm{\hat{\mu}_t} \notin \mathbf{1}_\delta(\bm{\mu})} \leq 
        \sum_{a \in \cA}\sum_{x \in \cX}\prob{|\hat{\mu}_{x,a}^{n_{x,a,t}}-\mu_{x,a}| \geq \delta}\;,
    \end{align*} where $\hat{\mu}_{x,a}^{n_{x,a,t}}$ is the sample estimate of the state-action pair $(x,a)$, having been sampled $n_{x,a,t}$ times till time $t$.  The remainder of the proof has the same flavor as in the work of \citet{auer2002finite}. We provide here for the sake of completeness. Note that each term $\prob{|\hat{\mu}_{x,a}^{n_{x,a,t}}-\mu_{x,a}|| \geq \delta}$ can be bounded in the following manner,
\begin{align*}
    &\mathbb{P}\Big\{|\hat{\mu}_{x,a}^{n_{x,a,t}}-\mu_{x,a}| \geq \delta\Big\} = \sum_{s=1}^t\mathbb{P}\Big\{|\hat{\mu}_{x,a}^s-\mu_{x,a}| \geq \delta, {n_{x,a,t}}=s\Big\} \\
    &=\sum_{s=1}^t\mathbb{P}\Big\{{n_{x,a,t}}=s \mid \abs{\hat{\mu}_{x,a}^s-\mu_{x,a}} \geq \delta\Big\}\mathbb{P}\Big\{|\hat{\mu}_{x,a}^s-\mu_{x,a}| \geq \delta \Big\} \\
    &\leq \sum_{s=1}^t\mathbb{P}\Big\{n_{x,a, t}=s \mid \abs{\hat{\mu}_{x,a}^s-\mu_{x,a}} \geq \delta\Big\}2\exp{(-2s\delta^2)},
\end{align*}
where we use the sub-Gaussian concentration by Assumption \ref{assm:context_subg_noise} and Lemma \ref{lemma:subG_conc},
\begin{align*}
    &\leq\sum_{s=1}^{t_0}\mathbb{P}\Big\{n_{x,a,t}=s \mid \abs{\hat{\mu}_{x,a}^s-\mu_{x,a}} \geq \delta\Big\} + \sum_{s = t_0 + 1}^t2\exp{(-2s\delta^2)} \\
    &\leq \sum_{s=1}^{t_0}\mathbb{P}\Big\{n_{x,a,t}=s \mid \abs{\hat{\mu}_{x,a}^s-\mu_{x,a}} \geq \delta\Big\} + \frac{1}{\delta^2}\exp{(-2t_0 \delta^2)}\;, 
\end{align*}
where we use the identity $\sum_{t = x+1}^{\infty}\exp(-\kappa t) \leq \frac{1}{\kappa}\exp(-\kappa x)$. Let $n_{x,a,t}^R$ be the number of times arm $a$ has been played randomly at state $x$ till time $t$, then, 
\begin{align*}
    &\sum_{s=1}^{t_0}\mathbb{P}\Big\{n_{x,a, t}=s \mid \abs{\hat{\mu}_{x,a}^s-\mu_{x,a}} \geq \delta\Big\} \\
    &\leq \sum_{s=1}^{t_0}\mathbb{P}\Big\{n_{x,a,t}^R \leq s \mid \abs{\hat{\mu}_i^s-\mu_i} \geq \delta\Big\} \\
    &\leq \sum_{s=1}^{t_0}\mathbb{P}\{n_{x,a,t}^R \leq s \} \\
    &\leq t_0 \mathbb{P}\{n_{x,a,t}^R \leq t_0 \}.
\end{align*}
Now, since states $x \in \cX$ is chosen statistically independently with probability $\mathbf{p}_x$ over the state-space (Assumption \ref{assm: context_distr}) and during random play actions are chosen independently and uniformly over action space $\cA$, we have,
\begin{align*}
\expect{n_{x,a,t}^R} = \frac{\mathbf{p}_x}{\abs{\cA}}\sum_{s=1}^t \epsilon_s     
\end{align*}
and variance of $n_{x,a,t}^R$ is 
\begin{align*}
    \mathbb{V}[n_{x,a,t}^R] = \sum_{s=1}^t \frac{\mathbf{p}_x\epsilon_s}{\abs{\cA}}\Big( 1- \frac{\mathbf{p}_x\epsilon_s}{\abs{\cA}} \Big) \leq \frac{\mathbf{p}_x}{\abs{\cA}}\sum_{s=1}^t \epsilon_s.
\end{align*}
Thus, choosing $t_0 = \frac{\mathbf{p}_x}{2\abs{\cA}}\sum_{s=1}^t \epsilon_s$, we have 
\begin{align*}
    \mathbb{P}\{n_{x,a,t}^R \leq t_0 \} \leq \exp(-t_0/5)\;,
\end{align*}
from Bernstein's Inequality, Lemma \ref{lemma:Bernstein}.
Thus putting it all together, we have,
\begin{align*}
    \prob{A_t = a \mid X_t = x} \leq \frac{\epsilon_t}{\abs{\cA}} + \abs{\cX\cA}t_0\exp(-t_0/5) + \frac{\abs{\cX\cA}}{\delta^2}\exp{(-2t_0 \delta^2)}.
\end{align*}
Thus, to complete the proof, we only need to find a lower bound on $t_0$.
As per our definition, 
\begin{align*}
    t_0 &= \frac{\mathbf{p}_x}{2\abs{\cA}}\sum_{s=1}^t \epsilon_s
\end{align*}
Plugging the value of $\eps_s = 1/\sqrt{s}$, we get
\begin{align*}
    t_0 &= \frac{\mathbf{p}_x}{2\abs{\cA}}\sum_{s=1}^t \frac{1}{\sqrt{s}}\\
    &\geq \frac{\mathbf{p}_x}{\abs{\cA}}(\sqrt{t+1} - 1).
\end{align*}
Thus, 
\begin{align*}
    \prob{A_t = a \mid X_t = x} \leq \frac{1}{\abs{\cA}\sqrt{t}} + \mathbf{p}_x\abs{\cX}(\sqrt{t+1} - 1)\exp(-\frac{\mathbf{p}_x(\sqrt{t+1} - 1)}{5\abs{\cA}}) + \frac{\abs{\cX\cA}}{\delta^2}\exp{\Big(-\frac{2\delta^2\mathbf{p}_x(\sqrt{t+1}-1)}{\abs{\cA}}\Big)}.
\end{align*}
Thus, the expected regret at time $t$ at state $x$ is
\begin{align*}
    \mathrm{Regret}_{x,t} &= \sum_{a \in \cA} \Delta_{x,a} \prob{A_t = a \mid X_t = x} \\
    &\leq \Delta_{x , \max}\frac{1}{\sqrt{t}} + o(1/t)\;,
\end{align*}
where $\Delta_{x,\max} = \max_{a \in \cA} \Delta_{x,a}$ and $\Delta_{x,a} = \mu_{x, \mathrm{OPT}(x)} - \mu_{x,a}$.
Thus the expected regret at time $t$ is
\begin{align*}
    \mathrm{Regret}_t &= \sum_{x \in \cX}\mathrm{Regret}_{x,t}\prob{X_t = x} \leq \sum_{x \in \cX} \mathbf{p}_x\Big(\Delta_{x , \max}\frac{1}{\sqrt{t}} + o(1/t)\Big) \\
    &\leq \frac{\Delta_{\max}}{\sqrt{t}} + o(1/t).
\end{align*}
The cumulative regret is, therefore,
\begin{align*}
    \sum_{t=1}^T \mathrm{Regret}_t \leq O(\Delta_{\max} \sqrt{T}).
\end{align*}
\end{proof}
\begin{remark}
    Note that our regret guarantee in Theorem \ref{thm:context_eps_greedy_app} does not depend upon the $l_\infty$ misspecification error (or any measure of misspecification) and depends only on the suboptimality gap. Our experiments illustrated in Appendix \ref{subsec:exp_context} corroborate this.  
\end{remark}

The remainder of this subsection is devoted to the proof of the two results presented in the main paper.
We provide proof of Corollary \ref{corr:robust_parameter} that we presented in the main paper.

\begin{corollary}
\label{corr:robust_parameter_app}
    For any contextual bandit instance $\bm{\mu}$, we have $\bm{\mu} \in \cC^{\cX}$ if and only if the model estimate $\mathbf{P}^\Lambda(\bm{\mu})$, computed under any sampling distribution $\{\alpha(x,a)\}_{(x,a) \in \cX \times \cA} \in \Delta_{\mathrm{XA}}$, belongs to $\Theta^{\cX}$, where $\Theta^{\cX} \triangleq \bigcap_{x \in \cX}\Theta^x_{\mathrm{OPT}(x)}$.
\end{corollary}
\begin{proof}
For any $\bm{\mu}$ we have $\bm{\mu}_x \in \cR^x_{\mathrm{OPT}(x)}$ for every context $x$, by definition.
Thus, from definition of $\cC^{\cX} = \bigcap_{x\in \cX}\cC^x_\mathrm{OPT(x)}$ we have,
    \begin{align*}
        \bm{\mu} \in \cC^{\cX} &\iff \bm{\mu}_x \in \cC^x_{\mathrm{OPT}(x)} \; \forall x \in \cX.\\
        &\iff\mathbf{P}^\Lambda (\bm{\mu}) \in \Theta^x_{\mathrm{OPT}(x)} \; \forall x \in \cX.
    \end{align*}
    for any sampling distribution $\{\alpha(x,a)\}_{(x,a)\in \cX\times\cA} \in \Delta_{\cX\cA}$.
    This follows from definition of $\cC^x_{\mathrm{OPT}(x)}$.
    But this is the definition of set intersection, that is,
    \begin{align*}
        &\iff \mathbf{P}^\Lambda (\bm{\mu}) \in \bigcap_{x \in \cX}\Theta^x_{\mathrm{OPT}(x)} \;
    \end{align*}
    for any sampling distribution $\{\alpha(x,a)\}_{(x,a)\in \cX\times\cA} \in \Delta_{\cX\cA}$.
\end{proof}
\paragraph{Characterization of the Robust Observation Region in Contextual Settings} When the model class is chosen to be linear, the \emph{robust observation region}, $\cC^{\cX}$, has an explicit analytic description which was presented without proof in Theorem \ref{thm:context_convex_hull}. We provide the proof here. 
\begin{theorem}
\label{thm:context_convex_hull_app}
For any contextual bandit instance $\bm{\mu}$ we have, $\bm{\mu}$ belongs to the \emph{robust observation region} $\cC^{\cX}$ if and only if for each $d \times d$ full rank sub-matrix of $\bm{\Phi}$, denoted by $\Phi_{d}$, along with the corresponding $d$ rows of $\bm{\mu}$, denoted as $\bm{\mu}_{d}$, satisfy $\Phi_{d}^{-1}\bm{\mu}_d \in \Theta^{\cX}$. That is 
for any $\bm{\mu}$
\begin{align*} 
    \bm{\mu} \in \cC^{\cX} \iff  \Phi_d^{-1}\bm{\mu}_d \in \Theta^{\cX}\; 
\end{align*}
for all $d\times d$ full rank sub-matrices of $\bm{\Phi}$ (denoted as $\Phi_d$) and the corresponding $d$ rows of $\bm{\mu}$ (denoted as $\bm{\mu}_d$).
\end{theorem}

\begin{proof}
The proof uses a result of \citet{doi:10.1137/S0895479895284014}, presented in Lemma \ref{lemma:forsgren}, that for any any sampling distribution $\{\alpha(x,a)\}_{(x,a) \in \cX \times \cA} \in \Delta_{\mathrm{XA}}$, the model estimate $\mathbf{P}^\Lambda(\bm{\mu})$, with $\Lambda = \mathrm{diag}(\{\alpha(x,a)\}_{x,a) \in \cX \times \cA})$, lies in the convex hull of the {\em basic solutions} $\Phi_d^{-1}\bm{\mu}_d$.
Note that for any $\bm{\mu}$, we have $\bm{\mu}_x \in \cR^x_{\mathrm{OPT}(x)}$ for any context $x$ by definition. Therefore, we have
\begin{align*}
    \bm{\mu} \in \cC^{\cX} &\iff \mathbf{P}^\Lambda (\bm{\mu}) \in \Theta^{\cX}\; \forall\; \Lambda \in \bm{\Lambda} \\
    &\iff  \mathrm{conv}\{\Phi_d^{-1}\bm{\mu}_d \;\forall\;\Phi_d\subset\bm{\Phi} \} \subset \Theta^{\cX}, \\
    &\iff \Phi_d^{-1}\bm{\mu}_d \in \Theta^{\cX} \; \forall\; \Phi_d \subset \bm{\Phi}.
\end{align*}
The first line above follows from Corollary \ref{corr:robust_parameter_app}, and the second line follows from Lemma \ref{lemma:forsgren}. (We abuse notation to denote $d\times d$ full rank sub-matrices of $\bm{\Phi}$ by $\Phi_d \subset \bm{\Phi}$ and use $\mathrm{conv}$ to denote the convex hull.) The last assertion follows because $\Theta^{\cX}$ is a convex set.
\end{proof}

\section{LinUCB}
\label{sec:linucb}

LinUCB remains a canonical regret optimal bandit algorithm. In this section we show that under standard assumptions, we retain the optimal regret of LinUCB even under misspecified settings, given that the bandit instance belongs to the robust region. We begin by first introducing the algorithm in Algorithm \ref{alg:LinUCB}.

\begin{algorithm}[htbp]
\caption{OFUL Algorithm}
\label{alg:LinUCB}
\begin{algorithmic}[1] 
\STATE \textit{Forced Exploration Phase of $d$ linearly independent features}
\STATE Set $V = \mathbf{0}^{d \times d}$ and $S = \mathbf{0}^d$
\FOR{i = 1 to d}
\STATE Play feature $\phi_i$ and observe noisy reward $y_i$
\STATE Compute $V = V + \phi_i \phi_i^\top$
\STATE Compute $S = S + \phi_iy_i$
\ENDFOR
\STATE \textit{Standard OFUL Phase}
\STATE Set $V_t = V$ and $S_t = S$
\FOR{t = 1 to T}
\STATE Estimate $\hat{\theta}_t = \big[V_t\big]^{-1}S_t$ 
\STATE Play arm $A_t$, such that $\phi_{A_t}, \tilde{\theta}_t = \argmax_{i \in [K], \theta \in \cC_t}\phi_i^\top\theta$, where $\cC_t = \Big\{\theta \;:\; \big\|\theta - \hat{\theta}_t\big\|_{V_t} \leq \sqrt{\beta_t(\delta)} \Big\}$
\STATE Observe the reward $y_t$.
\STATE Update $V_{t+1} =  V_t + \phi_{A_t}\phi_{A_t}^\top$ 
\STATE Update $S_{t+1} =  S_t + \phi_{A_t}y_t$
\ENDFOR
\end{algorithmic}
\end{algorithm}
\begin{remark}
    Note that the only non standard modification that we have included in the algorithm as given in \citet{oful} is that instead of using a regularizer to ensure invertibility of the design matrix ($\sum_{s=1}^t \phi_{A_s}\phi_{A_s}^\top$), we sample $d$ linearly independent arms explicitly, to ensure that the design matrix is invertible. This is a mild technicality and can be removed by the introduction of regularizers as discussed in Appendix \ref{sec:ridge_app}.
\end{remark}

We have the following standard assumptions \citep{oful}.

\begin{assumption}[Conditionally sub-Gaussian Noise]
\label{assm:subGLin}
    At any time $t$, the observation $y_t$ corresponding to the arm played $A_t$, is given by
    \begin{align*}
        y_t = \mu_{A_t} + \eta_t,
    \end{align*}
    where $\eta_t$ is conditionally $R$-sub Gaussian, that is,
    \begin{align*}
        \mathbf{E}[e^{\lambda \eta_t}\mid A_{1:t}, \eta_{1:t-1}] \leq \exp{\Big(\frac{\lambda^2 R^2}{2}\Big)}\;\;\;\; \forall\;\lambda\in\Real.
    \end{align*}
\end{assumption}

\begin{assumption}[Bounded Features]
\label{assm:Bounded_Features}
We assume that for any arm $i$ in the arm set $[K]$, the corresponding feature $\phi_i$ is bounded in the $l^2$ norm by $1$, that is,
\begin{align*}
    \|\phi_i\|_2 \leq 1\;\;\;\forall\; i \in [K].
\end{align*}
\end{assumption}

\begin{assumption}[$d$ rank feature matrix]
\label{assm:full_rank}
We assume that the feature matrix $\bm{\Phi} \in \Real^{\mathrm{K}\times d}$ is of full column rank and that the first $d$ rows are an orthogonal basis of $\Real^d$, such that at the end of the forced exploration phase, the design matrix $V$ has minimum eigenvalue $\lambda_{\min}(V)$ more than $1$.
\end{assumption}

\begin{remark}
    The last two assumptions are connected via the technical result, Lemma \ref{lemma:matrix_self_normalize}. In general, if the norm of the features were bounded by $L$, that is, if $\|\phi_i\|_2 \leq L$ for any $i \in [K]$, then we would require a forced exploration phase till the design matrix $V$ has minimum eigenvalue satisfying $\lambda_{\min}(V) \geq \max\{1, L^2\}$.  
\end{remark}
\begin{remark}
    The last assumption can be dropped by introducing regularizers and ensuring that the regularization parameter, $\lambda$ satisfies the condition $\lambda \geq \max\{1, L^2\}$ if the features are bounded by $L$. As mentioned, we discuss the regularized estimates in Appendix \ref{sec:ridge_app}. 
\end{remark}

We refer to the result in the main paper in Theorem \ref{thm:bandit_LinUCB_main} here. 

\begin{theorem}[\textbf{Sufficient Condition for Zero Regret in Misspecified Bandits}]
\label{thm:bandit_LinUCB_app}
Given a feature matrix $\bm{\Phi}$ satisfying Assumptions \ref{assm:Bounded_Features} and \ref{assm:full_rank}, and any bandit parameter $\bm{\mu}$ which is an interior point of the \emph{robust observation region}, $\cC_{\mathrm{OPT}(\bm{\mu})}$ and satisfies Assumption \ref{assm:subGLin},  
the LinUCB algorithm, Algorithm \ref{alg:LinUCB}, enjoys a regret of $\widetilde{O}(d\sqrt{T})$ on $\bm{\mu}$.
\end{theorem}

\begin{remark}
    The main difficulty is while we can prove that the sample estimates $\bm{\hat{\mu}_t}$ must fall inside the \emph{robust observation regions} with high probability, we cannot trivially conclude the same for the high confidence ellipsoids of the least squares estimate formed under LinUCB.
\end{remark}

Consider the model estimate, $\mathbf{P}^{\Lambda(t)}(\bm{\mu})$ based on noiseless observations of $\bm{\mu}$ under the sampling distribution $\{\alpha(i,t)\}_{i=1}^K$ at time $t$, that is,
\begin{align*}
    \mathbf{P}^{\Lambda(t)}(\bm{\mu}) = \Big(\sum_{i \in [K]} \alpha(i, t)\phi_i \phi_i^\top\Big)^{-1} \sum_{i \in [K]}\alpha(i, t)\phi_i\mu_i\;.
\end{align*}
Here we use $\Lambda(t) = \diag\{\alpha(i, t)\}_{i \in [K]}$ to denote the weight matrix of the least squares problem. We now make the following observation.

\begin{lemma}[Lower Bound of per instant Regret in the Function Space]
\label{lemma:per_instant_lower_bound}
If $\bm{\mu}$ is an interior point of the robust observation region, that is, if $\bm{\mu} \in \mathrm{Int}(\cC_{\mathrm{OPT}(\bm{\mu})})$, then  there exists a $\Delta_{\min} > 0$, such that for any sampling distribution $\{\alpha(i, t)\}_{i \in [K]}$ at any time $t \geq 1$, we have $\phi_{\mathrm{OPT}(\bm{\mu})}^\top \mathbf{P}^{\Lambda(t)}(\bm{\mu}) -\phi_i^\top \mathbf{P}^{\Lambda(t)}(\bm{\mu}) \geq \Delta_{\min}$ for any sub-optimal arm $i$.    
\end{lemma}

\begin{proof}
By definition of \emph{robust observation region}, Definition \ref{def:robust_observation_bandit}, $\bm{\mu} \in \mathrm{Int}(\cC_{\mathrm{OPT}(\bm{\mu})})$, if and only if, $\mathbf{P}^{\Lambda}(\bm{\mu})$ belongs to the robust parameter region $\Theta_{\mathrm{OPT}(\bm{\mu})}$ for any sampling distributions $\{\alpha(i)\}_{i \in [K]}$ in the $K$-dimensional simplex. From the definition of $\Lambda(t)$ which is defined by the sampling distribution $\{\alpha(i, t)\}_{i \in [K]}$ till time $t$, we must have $\mathbf{P}^{\Lambda(t)}(\bm{\mu})$ belonging to the robust parameter region $\Theta_{\mathrm{OPT}(\bm{\mu})}$ for all $t \geq 1$ when $\bm{\mu} \in \mathrm{Int}(\cC_{\mathrm{OPT}(\bm{\mu})})$. Therefore, $\bm{\Phi}\mathbf{P}^{\Lambda(t)}(\bm{\mu})$ belongs to the greedy region $\cR_{\mathrm{OPT}(\bm{\mu})}$ and hence $\phi_{\mathrm{OPT}(\bm{\mu})}^\top\mathbf{P}^{\Lambda(t)}(\bm{\mu}) > \phi_i^\top \mathbf{P}^{\Lambda(t)}(\bm{\mu})$ for any sub-optimal arm $i$ for all time $t \geq 1$.

Now note that from Theorem \ref{thm:bandit_convex_hull} we have $\mathbf{P}^{\Lambda(t)}(\bm{\mu})$ belongs to the closed convex hull, $\overline{\cM}$, of all $K \choose d$ basic solutions, where $\overline{\cM} \subset \Theta_{\mathrm{OPT}(\bm{\mu})}$. Note that $0 \notin \overline{\cM}$ for otherwise $0 \in \Theta_{\mathrm{OPT}(\bm{\mu})}$ which results in $0^K$, the $0$ vector in $\Real^K$ to belong to a greedy region, a contradiction. Thus for any time $t \geq 1$, and for any sub-optimal arm $i$
\begin{align*}
    \phi_{\mathrm{OPT}(\bm{\mu})}^\top\mathbf{P}^{\Lambda(t)}(\bm{\mu}) - \phi_i^\top \mathbf{P}^{\Lambda(t)}(\bm{\mu}) \geq \min_{\theta \in \overline{\cM}} \phi_{\mathrm{OPT}(\bm{\mu})}^\top\theta - \phi_i^\top \theta = \Delta_i > 0.
\end{align*}
where we have used the standard result that minimization of a linear function over a convex polytope must occur at one of the $K \choose d$ vertices to define $\Delta_i$ and that $\overline{\cM} \subset \Theta_{\mathrm{OPT}(\bm{\mu})}$ to conclude that $\Delta_i > 0$. Define $\Delta_{\min} = \min_{i \in K}\Delta_i$. Intuitively, $\Delta_{\min}$ represents the minimum sub-optimality gap of the bandit instance in the model space.
\end{proof}

\begin{remark}
    Note that the $\Delta_{\min}$ can be imagined as measuring a sense of misspecification in the model space.
\end{remark}

The crux of the proof of Theorem \ref{thm:bandit_LinUCB_app} relies on the following observation, whose proof is similar to the one found in \citet{oful}.
\begin{lemma}[Upper Bound of Sub-Optimal Plays]
\label{lemma:suboptimal_counts}
Given a feature matrix $\bm{\Phi}$ satisfying Assumptions \ref{assm:Bounded_Features} and \ref{assm:full_rank}, and any bandit parameter $\bm{\mu}$ which is an interior point of the \emph{robust observation region}, $\cC_{\mathrm{OPT}(\bm{\mu})}$ and satisfies Assumption \ref{assm:subGLin},  
the LinUCB algorithm, Algorithm \ref{alg:LinUCB}, plays sub-optimally at most $\Tilde{O}(d\sqrt{t})$. 

That is, with $\beta_t(\delta)$ set as $2 R^2 \log\Big( \frac{(1 + t/d)^{d/2}}{\delta}\Big)$ for any $\delta > 0$, we have with probability at least $1-\delta$
\begin{align*}
    \sum_{s=1}^t \mathbb{1}\{A_s \neq \mathrm{OPT}(\bm{\mu})\} \leq \frac{4\sqrt{t}R}{\Delta_{\min}} \sqrt{\log \Big(\frac{(1 + t/d)^{d/2}}{\delta}\Big)}\sqrt{\log{(1 + t/d)^d}},
\end{align*}
where $\Delta_{\min}$ is as defined in the previous Lemma \ref{lemma:per_instant_lower_bound}.
\end{lemma}

With the above lemma the regret bound is straightforward. Formally, 

\begin{lemma}
\label{lemma:LinUCB_regret}
    Given a feature matrix $\bm{\Phi}$ satisfying Assumptions \ref{assm:Bounded_Features} and \ref{assm:full_rank}, and any bandit parameter $\bm{\mu}$ which is an interior point of the \emph{robust observation region}, $\cC_{\mathrm{OPT}(\bm{\mu})}$ and satisfies Assumption \ref{assm:subGLin}, 
the LinUCB algorithm, Algorithm \ref{alg:LinUCB}, achieves regret of the order $\Tilde{O}(d\sqrt{t})$.

That is, with $\beta_t(\delta)$ set as $2 R^2 \log\Big( \frac{(1 + t/d)^{d/2}}{\delta}\Big)$ for any $\delta > 0$, we have with probability at least $1-\delta$
\begin{align*}
    \sum_{s=1}^t \mu_{\mathrm{OPT}(\bm{\mu})} - \mu_{A_s} \leq \frac{4\sqrt{t}R \Delta_{\max}}{\Delta_{\min}} \sqrt{\log \Big(\frac{(1 + t/d)^{d/2}}{\delta}\Big)}\sqrt{\log{(1 + t/d)^d}},
\end{align*}
where $\Delta_{\min}$ is as defined in Lemma \ref{lemma:per_instant_lower_bound} and $\Delta_{\max}$ is defined as the worst sub-optimal gap, that is, $\Delta_{\max} = \max_{i \in [K]} \mu_{\mathrm{OPT}(\bm{\mu})} - \mu_{i}$.
\end{lemma}
This proves our Theorem \ref{thm:bandit_LinUCB_app}. We present the proof of Lemma \ref{lemma:LinUCB_regret} for the sake of completeness.
\begin{proof}[Proof of Lemma \ref{lemma:LinUCB_regret}]
The cumulative regret is defined as,
\begin{align*}
    &\sum_{s=1}^t \mu_{\mathrm{OPT}(\bm{\mu})} - \mu_{A_s}\\
    &= \sum_{s=1}^t\Big(\mu_{\mathrm{OPT}(\bm{\mu})} - \mu_{A_s}\Big)\mathbb{1}\{A_s \neq \mathrm{OPT}(\bm{\mu})\}\\
    &\leq \Delta_{\max}\sum_{s=1}^t\mathbb{1}\{A_s \neq \mathrm{OPT}(\bm{\mu})\}.
\end{align*}    
The proof is completed using Lemma \ref{lemma:suboptimal_counts}.
\end{proof}

We present the proof of Lemma \ref{lemma:suboptimal_counts}
\begin{proof}[Proof of Lemma \ref{lemma:suboptimal_counts}]
Consider the following quantity,
\begin{align*}
    x^\top\mathbf{P}^{\Lambda(t)}(\bm{\hat{\mu}_t}) - x^\top \mathbf{P}^{\Lambda(t)}(\bm{\mu}),
\end{align*}
for any vector $x \in \Real^d$. Recall the definitions of $\mathbf{P}^{\Lambda(t)}(\bm{\hat{\mu}_t})$ and $\mathbf{P}^{\Lambda(t)}(\bm{\mu})$, as
\begin{align*}
    &\mathbf{P}^{\Lambda(t)}(\bm{\hat{\mu}_t}) = \Big(\sum_{s=1}^t\phi_{A_s}\phi_{A_s}^\top\Big)^{-1}\sum_{s=1}^t\phi_{A_s}(\mu_{A_s} + \eta_s) \triangleq \hat{\theta}_t\\
    &\mathbf{P}^{\Lambda(t)}(\bm{\mu}) = \Big(\sum_{s=1}^t\phi_{A_s}\phi_{A_s}^\top\Big)^{-1}\sum_{s=1}^t\phi_{A_s}\mu_{A_s}, 
\end{align*}
where $\phi_{A_s}$ is the feature of the arm played at time $s$ and $\mu_{A_s} + \eta_s$ is the observation at time $s$ having played arm $A_s$. This gives
\begin{align*}
    &x^\top\mathbf{P}^{\Lambda(t)}(\bm{\hat{\mu}_t}) - x^\top \mathbf{P}^{\Lambda(t)}(\bm{\mu}) \\
&=x^\top\Big(\sum_{s=1}^t\phi_{A_s}\phi_{A_s}^\top\Big)^{-1}\sum_{s=1}^t\phi_{A_s}\eta_s \\
&\leq\|x\|_{{V_t}^{-1}}\Big\|\sum_{s=1}^t\phi_{A_s}\eta_s\Big\|_{{V_t}^{-1}}.
\end{align*}
We can now use Lemma \ref{lemma:self_normalize} to write (after noting that $V = I$, the Identity Matrix in $d$ dimension),
\begin{align*}
    x^\top\mathbf{P}^{\Lambda(t)}(\bm{\hat{\mu}_t}) - x^\top \mathbf{P}^{\Lambda(t)}(\bm{\mu}) \leq |x\|_{{V_t}^{-1}} R \sqrt{2\log \Big(\frac{\mathrm{det}V_t^{1/2}}{\delta}\Big)}.
\end{align*}
Thus setting $x = V_t\Big(\mathbf{P}^{\Lambda(t)}(\bm{\hat{\mu}_t}) - \mathbf{P}^{\Lambda(t)}(\bm{\mu})\Big)$, we have
\begin{align}
\label{eq:high_conf}
    \Big\|\mathbf{P}^{\Lambda(t)}(\bm{\hat{\mu}_t}) - \mathbf{P}^{\Lambda(t)}(\bm{\mu})\Big\|_{V_t} \leq R \sqrt{2\log \Big(\frac{\mathrm{det}V_t^{1/2}}{\delta}\Big)},
\end{align}
with probability at least $1-\delta$ for all $t \geq 1$.

What this means is that under any sampling distribution, the Projection under that sampling distribution ( Definition \ref{def: projection_bandit}) lies within the high confidence ellipsoid centred around the estimated least squares solution.

Now let us consider the regret in the function space at any time $t$ as defined by the sampling distribution $\Lambda(t)$, given by 
\begin{align*}
    \phi_{\mathrm{OPT}(\bm{\mu})}^\top\mathbf{P}^{\Lambda(t)}(\bm{\mu}) - \phi_{A_t}^\top\mathbf{P}^{\Lambda(t)}(\bm{\mu}).
\end{align*}
The algorithm, at any time chooses $(\phi_{A_t}, \tilde{\theta}_t )$ as the optimistic estimate, where $\tilde{\theta}_t$ lies in the high confidence ellipsoid defined by Equation \ref{eq:high_conf}. To be consistent in our notation, we shall call the optimistic estimate $\tilde{\theta}_t$ as $\mathbf{P}^{\Lambda(t)}(\bm{\tilde{\mu}_t})$. This gives,
\begin{align*}
    &\phi_{\mathrm{OPT}(\bm{\mu})}^\top\mathbf{P}^{\Lambda(t)}(\bm{\mu}) - \phi_{A_t}^\top\mathbf{P}^{\Lambda(t)}(\bm{\mu})\\
    &\leq \phi_{A_t}^\top\mathbf{P}^{\Lambda(t)}(\bm{\tilde{\mu}_t}) - \phi_{A_t}^\top\mathbf{P}^{\Lambda(t)}(\bm{\mu}) \\
    &=\phi_{A_t}^\top\mathbf{P}^{\Lambda(t)}(\bm{\tilde{\mu}_t})-\phi_{A_t}^\top\mathbf{P}^{\Lambda(t)}(\bm{\hat{\mu}_t}) + \phi_{A_t}^\top\mathbf{P}^{\Lambda(t)}(\bm{\hat{\mu}_t}) -
    \phi_{A_t}^\top\mathbf{P}^{\Lambda(t)}(\bm{\mu}) \\
    &=\phi_{A_t}^\top\Big(\mathbf{P}^{\Lambda(t)}(\bm{\tilde{\mu}_t})-\mathbf{P}^{\Lambda(t)}(\bm{\hat{\mu}_t})\Big) + \phi_{A_t}^\top\Big(\mathbf{P}^{\Lambda(t)}(\bm{\hat{\mu}_t})-\mathbf{P}^{\Lambda(t)}(\bm{\mu})\Big)\\
&\leq\|\phi_{A_t}\|_{{V_t}^{-1}}\Big\|\mathbf{P}^{\Lambda(t)}(\bm{\tilde{\mu}_t})-\mathbf{P}^{\Lambda(t)}(\bm{\hat{\mu}_t})\Big\|_{V_t} + \|\phi_{A_t}\|_{{V_t}^{-1}}\Big\|\mathbf{P}^{\Lambda(t)}(\bm{\tilde{\mu}_t})-\mathbf{P}^{\Lambda(t)}(\bm{\hat{\mu}_t})\Big\|_{V_t}\\
&\leq 2\|\phi_{A_t}\|_{{V_t}^{-1}}\sqrt{\beta_t(\delta)},
\end{align*}
where $\beta_t(\delta) = 2R^2\log \Big(\frac{\mathrm{det}V_t^{1/2}}{\delta}\Big)$, with probability at least $1-\delta$.
Thus the cumulative regret in the function space is
\begin{align*}
    &\sum_{s=1}^t\phi_{\mathrm{OPT}(\bm{\mu})}^\top\mathbf{P}^{\Lambda(s)}(\bm{\mu}) - \phi_{A_s}^\top\mathbf{P}^{\Lambda(s)}(\bm{\mu})\\
    &\leq \sum_{s=1}^t 2\|\phi_{A_s}\|_{{V_s}^{-1}}\sqrt{\beta_s(\delta)}\\
    &=\sum_{s=1}^t 2\|\phi_{A_s}\|_{{V_s}^{-1}}R \sqrt{2\log \Big(\frac{\mathrm{det}V_s^{1/2}}{\delta}\Big)}\\
    &\leq 2\sqrt{2}R\sum_{s=1}^t \|\phi_{A_s}\|_{{V_s}^{-1}} \sqrt{\log \Big(\frac{(1 + s/d)^{d/2}}{\delta}\Big)},
\end{align*}
where in the last inequality we use Lemma \ref{lemma:det_trace} and our Assumption \ref{assm:Bounded_Features} on bounded features.
Thus continuing, we have,
\begin{align*}
    &2\sqrt{2}R\sum_{s=1}^t \|\phi_{A_s}\|_{{V_s}^{-1}} \sqrt{\log \Big(\frac{(1 + t/d)^{d/2}}{\delta}\Big)} \\
    &\leq 2\sqrt{2}R\sqrt{\log \Big(\frac{(1 + t/d)^{d/2}}{\delta}\Big)}\sum_{s=1}^t \|\phi_{A_s}\|_{{V_s}^{-1}} \\
    &\leq 2\sqrt{2}R\sqrt{\log \Big(\frac{(1 + t/d)^{d/2}}{\delta}\Big)} \sqrt{t}\sqrt{\sum_{s=1}^t \|\phi_{A_s}\|^2_{{V_s}^{-1}}}.
\end{align*}
From our assumption that $\|\phi_{A_s}\|_2 \leq 1$ and from the forced exploration start we have $\lambda_{\min}(V) = 1$, we use Lemma \ref{lemma:matrix_self_normalize} to get,
\begin{align*}
  & 2\sqrt{2}R\sqrt{\log \Big(\frac{(1 + t/d)^{d/2}}{\delta}\Big)} \sqrt{t}\sqrt{\sum_{s=1}^t \|\phi_{A_s}\|^2_{{V_s}^{-1}}}\\
  & 4\sqrt{t}R\sqrt{\log \Big(\frac{(1 + t/d)^{d/2}}{\delta}\Big)} \sqrt{\log\mathrm{det}V_t}\\
  &\leq 4\sqrt{t}R \sqrt{\log \Big(\frac{(1 + t/d)^{d/2}}{\delta}\Big)}\sqrt{\log{(1 + t/d)^d}}\\
  &=\tilde{O}(d\sqrt{t}).
\end{align*}
Thus what we have shown is that the regret in the function space is of order $\tilde{O}(d\sqrt{t})$. Note that from our definition of the minimum regret in the function space, $\Delta_{\min}$ given by Lemma \ref{lemma:per_instant_lower_bound}, we have
\begin{align*}
    &\sum_{s=1}^t\phi_{\mathrm{OPT}(\bm{\mu})}^\top\mathbf{P}^{\Lambda(s)}(\bm{\mu}) - \phi_{A_s}^\top\mathbf{P}^{\Lambda(s)}(\bm{\mu})\\
    &\geq \sum_{s=1}^t \mathbb{1}\{A_s \neq \mathrm{OPT}(\bm{\mu})\}\Delta_{\min}.
\end{align*}
Together, this gives the result.
\end{proof}

\begin{remark}
    Note that from the results of \citet{oful}, we can also derive a logarithmic regret bound. However, we leave that as an exercise for the reader.
\end{remark}

\paragraph{Comparison to the works of \citet{pmlr-v216-liu23c}} The authors study the problem of misspecification under a robustness criterion which characterizes the misspecification to be dominated by the sub-optimality gap. Under such a condition, they show that LinUCB enjoys {$O(\sqrt{T})$} regret when the misspecification has low order, specifically, it is of order {$O\Big(\frac{1}{d\sqrt{\log{T}}}\Big)$}. Note that this result is still non-trivial since the worst-case regret for LinUCB under uniform model error is {$\rho T$} if {$\rho$} is the misspecification error. However, we would like to address the following points while comparing our work

\begin{enumerate}
    
    \item Our notion of robustness is significantly different from theirs as we are able to show examples which achieve sub-linear regret even if the misspecification error dominates the sub-optimality gap.
    
    \item For our analysis of LinUCB we do not require the assumption of low misspecification error.

    \item We were able to analyse both $\epsilon$-greedy and LinUCB under our notion of robustness in both the bandit and the contextual setting. Our analysis of LinUCB indicates that Thompson Sampling could be analysed as well under the same robustness condition, whereas with their notion of robustness, they were only able to analyse LinUCB under a low misspecification assumption.

\end{enumerate}

\section{LinUCB in Contextual Bandits}
\label{sec:linucb_context}

We show that in the contextual setting, LinUCB achieves sub-linear regret for any contextual bandit instance lying in the robust observation region. In particular, we give the details and the proof of Theorem \ref{thm:linucb_context} presented in the main paper.

\begin{algorithm}[htbp]
\caption{OFUL Algorithm}
\label{alg:LinUCB_context}
\begin{algorithmic}[1] 
\STATE \textit{Forced Exploration Phase of $d$ linearly independent features}
\STATE Set $V = \mathbf{0}^{d \times d}$ and $S = \mathbf{0}^d$
\FOR{i = 1 to d}
\STATE Observe context $X_i$
\STATE Play feature $\phi_{X_i, A_i}$ and observe noisy reward $Y_i$
\STATE Compute $V = V + \phi_{X_i, A_i} \phi_{X_i, A_i}^\top$
\STATE Compute $S = S + \phi_{X_i, A_i}Y_i$
\ENDFOR
\STATE \textit{Standard OFUL Phase}
\STATE Set $V_t = V$ and $S_t = S$
\FOR{t = 1 to T}
\STATE Estimate $\hat{\theta}_t = \big[V_t\big]^{-1}S_t$
\STATE Observe context $X_t$
\STATE Play arm $A_t$, such that $\phi_{A_t}, \tilde{\theta}_t = \argmax_{a \in \cA, \theta \in \cC_t}\phi_{X_t, A}^\top\theta$, where $\cC_t = \Big\{\theta \;:\; \big\|\theta - \hat{\theta}_t\big\|_{V_t} \leq \sqrt{\beta_t(\delta)} \Big\}$
\STATE Observe the reward $Y_t$.
\STATE Update $V_{t+1} =  V_t + \phi_{X_t, A_t}\phi_{X_t, A_t}^\top$ 
\STATE Update $S_{t+1} =  S_t + \phi_{A_t}Y_t$
\ENDFOR
\end{algorithmic}
\end{algorithm}

As in the bandit setting, we shall have a forced exploration phase of $d$ linearly independent features to ensure the invertibility of the design matrix. This adds, at most, a constant order of regret. The remaining assumptions for the bandit setting also remain valid for the contextual setting. We shall see in Appendix \ref{sec:ridge_app} how to remove this condition by bringing a regularizer.

\begin{assumption}[Conditionally sub-Gaussian Noise]
\label{assm:subGLin_context}
    At any time $t$, the observation $Y_t$ corresponding to the arm played $A_t$ at context $X_t$, is given by
    \begin{align*}
        Y_t = \mu_{X_t, A_t} + \eta_t,
    \end{align*}
    where $\eta_t$ is conditionally $R$-sub Gaussian, that is,
    \begin{align*}
        \mathbf{E}[e^{\lambda \eta_t}\mid A_{1:t}, \eta_{1:t-1}] \leq \exp{\Big(\frac{\lambda^2 R^2}{2}\Big)}\;\;\;\; \forall\;\lambda\in\Real.
    \end{align*}
\end{assumption}

\begin{assumption}[Bounded Features]
\label{assm:Bounded_Features_context}
We assume that the features $\phi_{x,a}$ is bounded in the $l^2$ norm by $1$, that is,
\begin{align*}
    \|\phi_{x,a}\|_2 \leq 1\;\;\;\forall\; x \in \cX\,, a \in \cA.
\end{align*}
\end{assumption}

\begin{assumption}[$d$ rank feature matrix]
\label{assm:full_rank_context}
We assume that the design matrix computed in the forced exploration phase, $V$ has minimum eigenvalue $\lambda_{\min}(V) \geq 1$.
\end{assumption}
\begin{remark}
    The last assumption can be satisfied by designing the context-specific feature matrix $\bm{\Phi_x}$, a $\abs{\cA}\times d$ matrix to have the first $d$ features from an orthonormal basis from $\Real^d$ for every context $x$. This assumption can be dropped by introducing a regularization parameter $\lambda$ and requiring $\lambda \geq 1$. We shall see this in Appendix \ref{sec:ridge_app}.  
\end{remark}

From the definition of the model estimate $\mathbf{P}^{\Lambda(t)}(\bm{\mu})$ being the $t^{th}$ least squares estimate under noiseless observations of $\bm{\mu}$ under a sampling distribution $\{\alpha(x,a,t)\}_{(x,a) \in \cX \times \cA}$ in the $\abs{\cX\cA}$ dimensional simplex $\Delta_{\mathrm{XA}}$, we have the following lemma.
\begin{lemma}[Lower Bound of per instant Regret in the Function Space]
\label{lemma:per_instant_lower_bound_context}
If $\bm{\mu}$ is an interior point of the robust observation region, that is, if $\bm{\mu} \in \mathrm{Int}(\cC^{\cX})$, then  there exists a $\Delta_{\min} > 0$, such that under any sampling distribution $\{\alpha(x, a, t)\}_{a \in \cA}$ at any time $t \geq 1$, we have for any context $x$, 
\begin{align*}
    \phi_{x, \mathrm{OPT}(x)}^\top \mathbf{P}^{\Lambda(t)}(\bm{\mu}) -\phi_{x,a}^\top \mathbf{P}^{\Lambda(t)}(\bm{\mu}) \geq \Delta_{\min}\;,
\end{align*} 
for any sub-optimal arm $a$ at context $x$.    
\end{lemma}
The proof uses the same argument as in the bandit setting (see Proof of Lemma \ref{lemma:per_instant_lower_bound}).

We also have, in a similar argument as for the bandit setting, the following results (See Proof of Lemma \ref{lemma:suboptimal_counts})

\begin{lemma}[High Confidence Ellipsoids]
\label{lemma:ellipsoid_context}
Under Assumptions \ref{assm:subGLin_context}, \ref{assm:Bounded_Features_context} and \ref{assm:full_rank_context},we have for any that for $t \geq 1$ and for any $\delta > 0$, we have,
\begin{align*}
    \Big\|\mathbf{P}^{\Lambda(t)}(\bm{\hat{\mu}_t}) - \mathbf{P}^{\Lambda(t)}(\bm{\mu})\Big\|_{V_t} \leq R \sqrt{2\log \Big(\frac{(1 + t/d)^{d/2}}{\delta}\Big)},
\end{align*}
with probability at least $1-\delta$.
\end{lemma}

Recall from our notation of $\mathbf{P}^{\Lambda(t)}(\bm{\hat{\mu}_t})$ to denote the usual least squares estimate $\hat{\theta}_t$.

\begin{lemma}[Cumulative Regret in the Model Space]
\label{lemma:regret_model_space_context}
Under Assumptions \ref{assm:subGLin_context}, \ref{assm:Bounded_Features_context} and \ref{assm:full_rank_context},we have for any that for $t \geq 1$ and for any $\delta > 0$, we have,
\begin{align*}
    \sum_{s=1}^t\phi_{X_s, \mathrm{OPT}(X_s)}^\top\mathbf{P}^{\Lambda(s)}(\bm{\mu}) - \phi_{X_s, A_s}^\top\mathbf{P}^{\Lambda(s)}(\bm{\mu}) \leq 4\sqrt{t}R \sqrt{\log \Big(\frac{(1 + t/d)^{d/2}}{\delta}\Big)}\sqrt{\log{(1 + t/d)^d}}\;,
\end{align*}
with probability at least $1-\delta$
\end{lemma}

Lemma \ref{lemma:per_instant_lower_bound_context} and Lemma \ref{lemma:regret_model_space_context} gives us the number of times any sub-optimal arm is played at any given context.

\begin{lemma}[Upper Bound on Sup-Optimal Plays]
\label{lemma:suboptimal_count_context}
Under Assumptions \ref{assm:subGLin_context}, \ref{assm:Bounded_Features_context} and \ref{assm:full_rank_context}, for any contextual bandit instance $\bm{\mu}$ lying in the robust region $\cC^{\cX}$, we have for any $t \geq 1$, and for any $\delta> 0$,
\begin{align*}
    \sum_{s=1}^t \mathbb{1}\{A_S \neq \mathrm{OPT}(X_s)\} \leq \frac{4\sqrt{t}R}{\Delta_{\min}} \sqrt{\log \Big(\frac{(1 + t/d)^{d/2}}{\delta}\Big)}\sqrt{\log{(1 + t/d)^d}}\;,
\end{align*}
with probability at least $1-\delta$.
\end{lemma}
This gives us the regret for LinUCB in contextual bandits,

\begin{theorem}
\label{thm:LinUCB_context_regret_app}
Under Assumptions \ref{assm:subGLin_context}, \ref{assm:Bounded_Features_context} and \ref{assm:full_rank_context}, for any contextual bandit instance $\bm{\mu}$ lying in the robust region $\cC^{\cX}$ of a given feature matrix $\bm{\Phi}$, we have for any $t \geq 1$, and for any $\delta> 0$, the regret of LinUCB as,
\begin{align*}
    \sum_{s=1}^T \mu_{X_s, \mathrm{OPT}(X_s)} - \mu_{X_s, A_S} \leq \frac{4\sqrt{t}R\Delta_{\max}}{\Delta_{\min}} \sqrt{\log \Big(\frac{(1 + t/d)^{d/2}}{\delta}\Big)}\sqrt{\log{(1 + t/d)^d}}\;,
\end{align*}
with probability at least $1-\delta$. Here $\Delta_{\max} = \max_{(x,a) \in \cX \times \cA} \mu_{x \mathrm{OPT}(x)} - \mu_{x,a}$.
\end{theorem}

\begin{proof}
\begin{align*}
    &\sum_{s=1}^T \mu_{X_s, \mathrm{OPT}(X_s)} - \mu_{X_s, A_S} \\
    &= \sum_{s=1}^T \mathbb{1}\{A_S \neq \mathrm{OPT}(X_s)\}\Delta_{X_s, A_s} \\
    &\leq \sum_{s=1}^T \mathbb{1}\{A_S \neq \mathrm{OPT}(X_s)\} \Delta_{\max} \\
    & \leq \frac{4\sqrt{t}R\Delta_{\max}}{\Delta_{\min}} \sqrt{\log \Big(\frac{(1 + t/d)^{d/2}}{\delta}\Big)}\sqrt{\log{(1 + t/d)^d}}\;.
\end{align*}
\end{proof}

This completes the proof our Theorem \ref{thm:LinUCB_context_regret_app}.

\paragraph{Comparison to \citet{zhang2023interplay}} The authors have studied misspecified contextual bandits with the misspecification dominated by the minimum sub-optimality gap. In this regard, we believe this is more in line with the work of \citet{pmlr-v216-liu23c}. Under this robustness criterion, they developed a sophisticated algorithm, DS-OFUL, which they analyzed to be regret optimal. In terms of the characterization of robustness, we believe that it is not comparable to our work. For example, we are able to analyze problems whose misspecification error is larger than the sub-optimality gap.
\section{Weighted Ridge Regression}
\label{sec:ridge_app}
In our definition of the \emph{robust observation region} we have implicitly assumed invertibility of the design matrix. For the purposes of mathematical rigor, a consistent definition of the robust observation region would be (see Lemma \ref{lemma:forsgren}),

\begin{definition}[\textbf{Robust Observation Region}]
\label{def:robust_observation_bandit_app}
For a given feature matrix $\bm{\Phi}$, we define the $k^{th}$ \emph{robust observation region} $\cC_k$, as the set of all $K$ armed bandit instances $\bm{\mu}$ with optimal arm $k$, such that under any sampling distribution $\{\alpha(i)\}_{i=1}^K \in \Delta_{\mathrm{K}}$, the corresponding regularized model estimate, $\mathbf{P}^{\Lambda}_\lambda(\bm{\mu})$, lies in the $k^{th}$ robust parameter region, $\Theta_k$.
\begin{align*}
    \cC_k = \Big\{\bm{\mu} \in \cR_k : \mathbf{P}^{\Lambda}(\bm{\mu}) \in \Theta_k \;\forall\; \Lambda \in \bm{\Lambda} \Big\}\;\;\;\textit{for any arm } k,
\end{align*}
where 
\begin{align*}
\bm{\Lambda} = \Big\{\Lambda = \mathrm{diag}(\{\alpha(i)\}_{i=1}^K) :  \{\alpha(i)\}_{i=1}^K \in \Delta_{\mathrm{K}}\; \land\; \bm{\Phi}^\top\Lambda\bm{\Phi} \text{ is invertible} \Big\}.
\end{align*}
\end{definition}
This necessitates a forced exploration phase in our algorithms. In practice, however one uses a regularizer $\lambda > 0$ to by pass the issue of invertibility. We discuss, how our theory of robust regions extend naturally to the case when one uses a weighted regularized least squares estimate rather than the ordinary least squares estimate.   
Consider the weighted regularized least squares estimate, defined as per our notation, the \emph{regularized model estimate} with regularizer $\lambda > 0$.
\begin{definition}[\textbf{Regularized Model Estimate under Sampling Distribution}]
\label{def: regularized_projection_bandit}
    For any bandit instance $\bm{\mu}$ in $\Real^{\mathrm{K}}$, we shall denote the model estimate of the $K$ dimensional element $\bm{\mu}$ under a sampling distribution, defined by $\Lambda = \mathrm{diag}(\{\alpha(i)\}_{i=1}^K)$, as     
    \begin{align*}
        \mathbf{P}^{\Lambda}_\lambda(\bm{\mu}) \triangleq \Big( \bm{\Phi}^\top\Lambda\bm{\Phi} + \lambda I\Big)^{-1} \bm{\Phi}^\top\Lambda\bm{\mu} \;,
    \end{align*}
    for some $\lambda>0$.
\end{definition}
Define the corresponding \emph{regularized robust observation region} as
\begin{definition}[\textbf{Regularized Robust Observation Region}]
\label{def:regularized_robust_observation}
For a given feature matrix $\bm{\Phi}$, we define the $k^{th}$ \emph{robust observation region} $\cC_k$, as the set of all $K$ armed bandit instances $\bm{\mu}$ with optimal arm $k$, such that under any sampling distribution $\{\alpha(i)\}_{i=1}^K \in \Delta_{\mathrm{K}}$, the corresponding regularized model estimate, $\mathbf{P}^{\Lambda}_\lambda(\bm{\mu})$, lies in the $k^{th}$ robust parameter region, $\Theta_k$.
\begin{align*}
    \cC^\lambda_k = \Big\{\bm{\mu} \in \cR_k : \mathbf{P}^{\Lambda}_\lambda(\bm{\mu}) \in \Theta_k \;\forall\; \Lambda \in \bm{\Lambda} \Big\}\;\;\;\textit{for any arm } k,
\end{align*}
where 
\begin{align*}
\bm{\Lambda} = \Big\{\Lambda = \mathrm{diag}(\{\alpha(i)\}_{i=1}^K) :  \{\alpha(i)\}_{i=1}^K \in \Delta_{\mathrm{K}} \Big\}.
\end{align*}
\end{definition}
Here $\Theta_k$ is the \emph{robust parameter region} as defined in Definition \ref{def:robust_parameter_bandit}.
Note that with this definition we have solved all our problems regarding invertibility. We can again give an explicit description of the set $\cC^\lambda_k$. Given a feature matrix $\bm{\Phi}$, a sampling distribution $\Lambda \in \bm{\Lambda}$ and a $K$ dimensional element $\bm{\mu}$, define the following augmented elements
\begin{align*}
    \bm{\Phi}^* = \begin{bmatrix}
        \bm{\Phi} \\
        \sqrt{\lambda}I_{d \times d} 
    \end{bmatrix} \;,\;\;\;\;\;\; \Lambda^* = \begin{bmatrix}
            \Lambda&& 0 \\
            0 && I_{d \times d}
        \end{bmatrix}\;\;\;\;\; \text{and }\;\;
        \bm{\mu}^* = \begin{bmatrix}
            \bm{\mu}\\
            0
        \end{bmatrix},
\end{align*}
where $\bm{\Phi}^* \in \Real^{K+d \times d}$, $\Lambda^* \in \Real^{K+d \times K+d}$ and $\bm{\mu}^* \in \Real^{K+d}$. Define the set $\cJ(\bm{\Phi}^*)$ as the set of row indices associated with non-singular $d \times d$ sub-matrices of $\bm{\Phi}^*$, and let $\cI$ be the set of row indices corresponding to $\{K+1, \cdots, K+d\}$. Then we have the following characterization of the regularized robust observation region.
\begin{theorem}
    For any reward vector $\bm{\mu}$ with optimal arm $k$,
    \begin{align*}
        \bm{\mu} \in \cC^\lambda_k \iff  \Phi_J^{*^{-1}}\bm{\mu}^*_J \in \Theta_k
    \end{align*}
    for all $J \in \cJ(\bm{\Phi}^*)\setminus\cI$.
\end{theorem}
\begin{proof}
For a sampling distribution, defined by $\Lambda \in \bm{\Lambda}$ we have
\begin{align*}
    \mathbf{P}^{\Lambda}_\lambda(\bm{\mu})& = \Big(\bm{\Phi}^\top\Lambda\bm{\Phi} + \lambda I\Big)^{-1} \bm{\Phi}^\top\Lambda\bm{\mu} \\
    &= \Big(\bm{\Phi}^{*^\top}\Lambda^*\bm{\Phi}^* \Big)^{-1} \bm{\Phi}^{*^\top}\Lambda^*\bm{\mu}^*
\end{align*}
where 
\begin{align*}
    \bm{\Phi}^* = \begin{bmatrix}
        \bm{\Phi} \\
        \sqrt{\lambda}I_{d \times d} 
    \end{bmatrix} \;,\;\;\;\;\;\; \Lambda^* = \begin{bmatrix}
            \Lambda&& 0 \\
            0 && I_{d \times d}
        \end{bmatrix}\;\;\;\;\; \text{and }\;\;
        \bm{\mu}^* = \begin{bmatrix}
            \bm{\mu}\\
            0
        \end{bmatrix}.
\end{align*}
Here $\bm{\Phi}^* \in \Real^{K+d \times d}$, $\Lambda^* \in \Real^{K+d \times K+d}$ and $\bm{\mu}^* \in \Real^{K+d}$. Thus, from the result in Lemma \ref{lemma:forsgren}, we have that regularized model estimate $\mathbf{P}^{\Lambda}_\lambda(\bm{\mu})$ to lie within the convex hull of at most $K+d \choose d$ basic solutions $\Phi_d^{*^{-1}}\bm{\mu}_d^*$ for any $\Lambda \in \bm{\Lambda}$. That is, 
\begin{align*}
   \mathbf{P}^{\Lambda}_\lambda(\bm{\mu}) = \sum_{J \in \cJ(\bm{\Phi}^*)} \Big(\frac{\mathrm{det}\Lambda^*_J\,\mathrm{det}\Phi_J^{*^{2}}}{\sum_{K \in \cJ(\bm{\Phi}^*)}\mathrm{det}\Lambda^*_K \,\mathrm{det}\Phi_K^{*^{2}}} \Big) \Phi_J^{*^{-1}}\bm{\mu}^*_J\,,\;\;\; \forall\,\Lambda \in \bm{\Lambda}
\end{align*}
where $\cJ(\bm{\Phi}^*)$ is the is the set of row indices associated with non-singular $d \times d$ sub-matrices of $\bm{\Phi}^*$. Consider the set of row indices, $\cI = \{K+1,\cdots,K+d\}$. Corresponding to this set of indices, $\Phi^{*^{-1}}_{\cI}\bm{\mu}^*_{\cI}$ is $0$. Thus decomposing the set of row indices $\cJ(\bm{\Phi}^*)$ into $\widetilde{\cJ}(\bm{\Phi}^*)$ and $\cI$, we observe, 
\begin{align*}
    \mathbf{P}^{\Lambda}_\lambda(\bm{\mu}) &= \sum_{J \in \cJ(\bm{\Phi}^*)} \Big(\frac{\mathrm{det}\Lambda^*_J\,\mathrm{det}\Phi_J^{*^{2}}}{\sum_{K \in \cJ(\bm{\Phi}^*)}\mathrm{det}\Lambda^*_K \,\mathrm{det}\Phi_K^{*^{2}}} \Big) \Phi_J^{*^{-1}}\bm{\mu}^*_J\\
    &= \sum_{J \in \widetilde{\cJ}(\bm{\Phi}^*)} \Big(\frac{\mathrm{det}\Lambda^*_J\,\mathrm{det}\Phi_J^{*^{2}}}{\sum_{K \in \widetilde{\cJ}(\bm{\Phi}^*)}\mathrm{det}\Lambda^*_K \,\mathrm{det}\Phi_K^{*^{2}} + \lambda^d} \Big) \Phi_J^{*^{-1}}\bm{\mu}^*_J \\
    &= \sum_{J \in \widetilde{\cJ}(\bm{\Phi}^*)} \Big(\frac{\mathrm{det}\Lambda^*_J\,\mathrm{det}\Phi_J^{*^{2}}}{\sum_{K \in \widetilde{\cJ}(\bm{\Phi}^*)}\mathrm{det}\Lambda^*_K \,\mathrm{det}\Phi_K^{*^{2}}} \Big) c^\lambda \Phi_J^{*^{-1}}\bm{\mu}^*_J\,,
\end{align*}
where $c^\lambda = \frac{\sum_{K \in \widetilde{\cJ}(\bm{\Phi}^*)}\mathrm{det}\Lambda^*_K \,\mathrm{det}\Phi_K^{*^{2}}}{\sum_{K \in \widetilde{\cJ}(\bm{\Phi}^*)}\mathrm{det}\Lambda^*_K \,\mathrm{det}\Phi_K^{*^{2}} + \lambda^d}$ is necessarily positive for any $\Lambda \in \bm{\Lambda}$.

Thus, we have

\begin{align*}
    \bm{\mu} \in \cC^\lambda_k \iff  \Phi_J^{*^{-1}}\bm{\mu}^*_J \in \Theta_k
\end{align*}
for all row indices $J$ in the set $\cJ(\bm{\Phi}^*)\setminus\cI$ associated with non-singular $d \times d$ sub-matrices of $\bm{\Phi}^*$. 

\end{proof}

With this notion of robust region, we immediately have the following results for the $\epsilon$-greedy algorithm, but without requiring any forced exploration phase.

\begin{theorem}
    For a feature matrix $\bm{\Phi}$, and its associated regularized robust observation region $\cC^\lambda_k$, any $\frac{1}{2} \text{ sub-Gaussian}$ bandit parameter $\bm{\mu}$ which is an interior point of the regularized robust observation region, that is $\bm{\mu} \in \mathrm{Int}(\cC^\lambda_k)$, the $\epsilon$-greedy algorithm with $\epsilon_t$ set as $\frac{1}{\sqrt{t}}$, achieves a regret of $O(\Delta_{\max}\sqrt{T})$, where $\Delta_{\max}$ is the largest sub-optimal gap, that is $\Delta_{\max} = \max_{i \in [K]} \mu_k - \mu_i$.
\end{theorem}

We also have a more relaxed version of the LinUCB result because we don't need any forced exploration phase.
The only change in the assumptions required for the analysis to go through is presented.

\begin{assumption}[Bounded Features]
\label{assm:Bounded_Features_Ridge}
We assume that for any arm $i$ in the arm set $[K]$, the corresponding feature $\phi_i$ is bounded in the $l^2$ norm by $1$, that is,
\begin{align*}
    \|\phi_i\|_2 \leq L\;\;\;\forall\; i \in [K].
\end{align*}
\end{assumption}

\begin{assumption}[Regularization]
\label{assm:Ridge_Regularizer}
We assume that the regularizer $\lambda$, so chosen satisfies
\begin{align*}
    \lambda \geq \max\{1, L^2\}\;,
\end{align*}
where $L$ is as defined in Assumption \ref{assm:Bounded_Features_Ridge}
\end{assumption}

Note that these are more relaxed Assumptions than the one presented in Section \ref{sec:linucb}. With this new set of assumptions, we have the following result.

\begin{theorem}
Given a feature matrix $\bm{\Phi}$ satisfying Assumptions \ref{assm:Bounded_Features_Ridge}, any bandit parameter $\bm{\mu}$ which is an interior point of the regularized \emph{robust observation region}, $\cC^\lambda_{\mathrm{OPT}(\bm{\mu})}$ and satisfies Assumption \ref{assm:subGLin},  
the LinUCB algorithm, with regularizer chosen as according to Assumption \ref{assm:Ridge_Regularizer}, enjoys a regret of $\widetilde{O}(d\sqrt{T})$ on $\bm{\mu}$.
\end{theorem}
We leave the verification of these results to the reader. We note that these results could also be extended to the contextual bandit setting. With this, we can conclude that the algorithms in practice are well suited for misspecification, provided the misspecified instances fall in the notion of robustness as defined.
\section{Experiments}
\label{sec:exp_app}

In this section we run some simple experiments to corroborate our findings.

\subsection{Bandits}
\label{subsec:exp_bandits}
In this section we shall focus on the bandit setting and deal with the contextual setting in the next section. We shall use Theorem \ref{thm:bandit_convex_hull} to explicitly calculate the \emph{robust observation region} for a simple feature matrix. We shall then sample bandit instances from the calculated \emph{robust observation region} and run $\epsilon$-greedy algorithm on these bandit instances. We shall demonstrate two phenomenon
\begin{enumerate}
    \item Visualization of the \emph{robust observation regions}.
    \item Sub-linear growth of cumulative regret on misspecified bandits.
\end{enumerate}

\subsubsection{Robust Regions}
Calculation of the \emph{robust parameter region} essentially involves solving a system of over-specified linear inequalities. Given a feature matrix of size $K\times d$, we need to solve a system of $K-1$ linear inequalities in $d$ parameters estimate the robust parameter region corresponding to arm i, $\Theta_i$. We can do this for each arm $i \in [K]$, to get all possible robust parameter regions, $\Theta_i\, \forall\, i \in [K]$. We can now use Theorem \ref{thm:bandit_convex_hull} to calculate each \emph{robust observation region} for any arm. This involves solving at most $K \choose d$ linear inequalities in $d$ dimensions corresponding to the basic solutions as defined in Theorem \ref{thm:bandit_convex_hull}.

For the purposes of visualization and ease of calculation, let us work in the space of all $3$ armed bandit problems. To avoid trivialization, we choose the ambient parameter dimension, $d$, to be two. We choose an arbitrary feature matrix, $\bm{\Phi}$ in $3 \times 2$ dimension as
$\begin{bmatrix}
    2 & 3\\
    4 & 5 \\
    2 & 1
\end{bmatrix}$. 
We note that there is nothing special about this feature matrix, except that it has full rank. We shall start by computing the \emph{robust parameter regions}. Recall from the definition of the robust parameter region $\Theta_i$ as the domain of the feature matrix $\bm{\Phi}$ such that the range belongs to the $i^{th}$ greedy region $\cR_i$ (see Definition \ref{def:robust_parameter_bandit}). Thus for any arm in $[3]$, we have $\Theta_i = \Big\{\theta \in \Real^2\;: \begin{bmatrix}
    2 & 3\\
    4 & 5 \\
    2 & 1
\end{bmatrix} \theta \in \cR_i \Big\}$. Thus for each $i \in [3]$, we need to solve for $2$ linear inequalities in $2$ unknowns. Namely for $\Theta_1$ we have the following set of inequalities
\begin{align*}
    &\phi_1^\top \theta - \phi_2^\top \theta > 0 \\
    &\phi_1^\top \theta - \phi_3^\top \theta > 0\;,
\end{align*}
which upon solving, we get the following condition $\theta_1 < -\theta_2 \land \theta_2 > 0$ for any  $\theta = \begin{bmatrix}
    \theta_1 \\
    \theta_2
\end{bmatrix}$ to belong to $\Theta_1$

Similarly, for the regions $\Theta_2$ and $\Theta_3$ we have the following set of equations, 
\begin{align*}
    &\phi_2^\top \theta - \phi_1^\top \theta > 0 \\
    &\phi_2^\top \theta - \phi_3^\top \theta > 0
\end{align*}
and 
\begin{align*}
    &\phi_3^\top \theta - \phi_2^\top \theta > 0 \\
    &\phi_3^\top \theta - \phi_1^\top \theta > 0.
\end{align*}
Upon solving, we arrive at the following descriptions of the \emph{robust parameter regions}
\begin{align}
\label{eg:bandit_parameter}
    &\Theta_1 = \{ \theta \in \Real^2\; : \theta_1 < -\theta_2 \land \theta_2 > 0\}  \\
 &\Theta_2 = \{ \theta \in \Real^2\; : (\theta_1 > 0 \land \theta_2 > - \theta_1/2) \lor (\theta_1 < 0 \land \theta_2 > - \theta_1) \} \\
 &\Theta_3 = \{ \theta \in \Real^2\; : \theta_2 < 0 \land \theta_2 < - \theta_1/2\}.
\end{align}
In Figure \ref{fig:parameter_space}, as a matter of interest, we plot \emph{robust parameter regions} in $\Real^2$ . We note that the parameter space $\Real^2$ is partitioned by the \emph{robust parameter regions}. This is not surprising since the greedy regions $\cR_i$ partition the $\Real^\mathrm{K}$ space.

\begin{figure}[htbp]
    \centering
    \includegraphics[width = 0.6\linewidth]{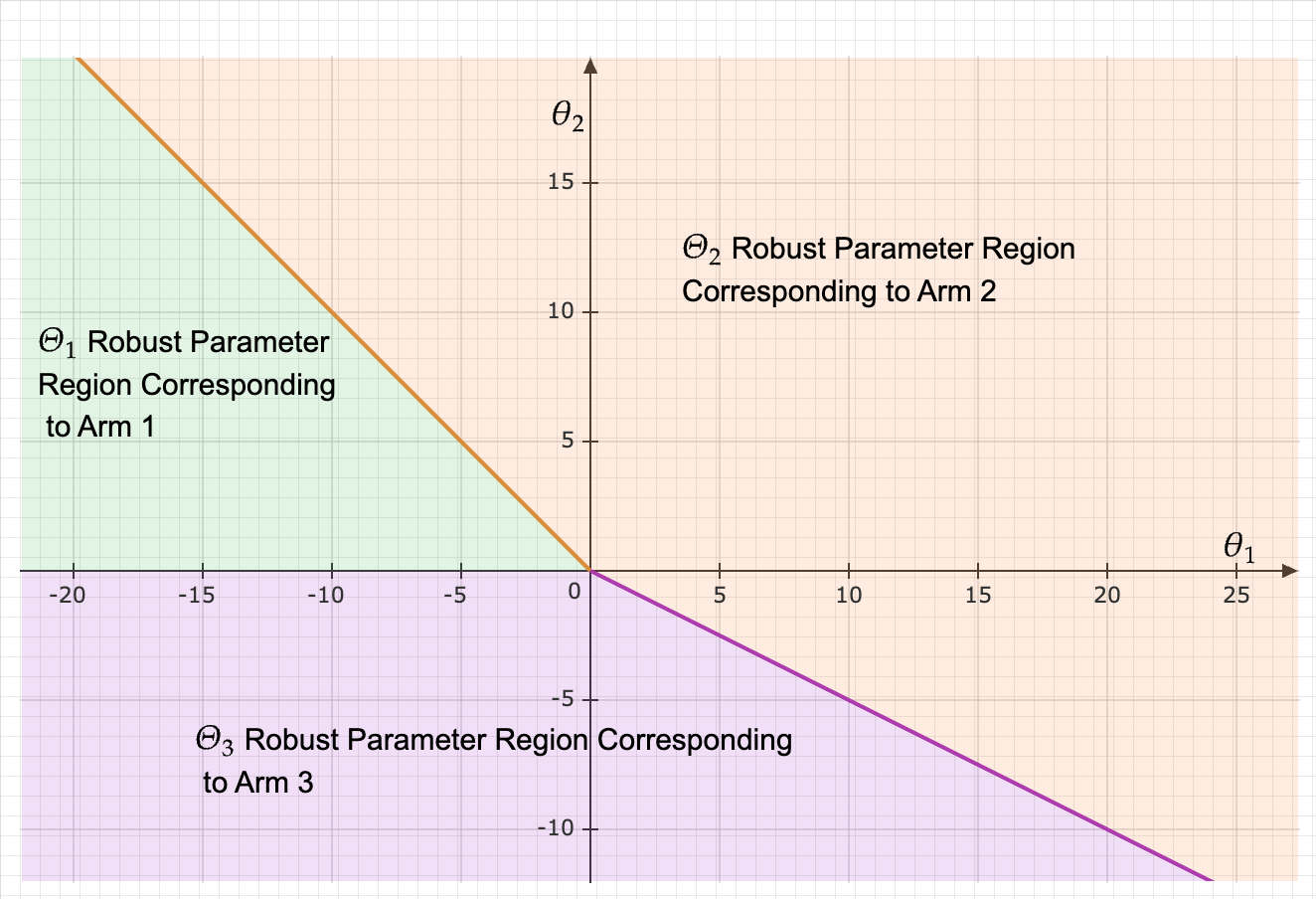}
    \caption{The parameter space $\Real^2$ is partitioned into disjoint sets of the \emph{robust parameter regions} corresponding to the different arms for feature matrix $\Phi = \begin{bmatrix}
        2 & 3\\
    4 & 5 \\
    2 & 1
    \end{bmatrix}$.}
    \label{fig:parameter_space}
\end{figure}
With the exact descriptions of the sets $\Theta_i$, we can calculate the \emph{robust observation regions}. Recall that the robust observation region for any arm arm $i$ is defined as $\cC_i = \{\bm{\mu} \in \Real^3\;: \mathbf{P}^\Lambda(\bm{\mu}) \in \Theta_i\}$ for all sampling distribution $\Lambda$ (see Definition \ref{def:robust_observation_bandit}). Using Theorem \ref{thm:bandit_convex_hull}), this turns out to be the set of all $\bm{\mu}$ such that $\Phi_2^{-1}\bm{\mu}_2$ belongs to $\Theta_i$ for every $2 \times 2$ full rank sub-matrix $\Phi_2$ of $\bm{\Phi}$. What this means is that for the \emph{robust observation region} corresponding to arm one, $\cC_1$, we have the following equations, 

\begin{align*}
&    \begin{bmatrix}
        2  \;&\; 3 \\
        4 \;&\; 5
    \end{bmatrix}^{-1} 
\begin{bmatrix}
    \mu_1\\
    \mu_2
\end{bmatrix} \in \Theta_1\;, \\
&    \begin{bmatrix}
        2  \;&\; 3 \\
        2 \;&\; 1
    \end{bmatrix}^{-1} 
\begin{bmatrix}
    \mu_1\\
    \mu_3
\end{bmatrix} \in \Theta_1\;, \\
&    \begin{bmatrix}
        4  \;&\; 5 \\
        2 \;&\; 1
    \end{bmatrix}^{-1} 
\begin{bmatrix}
    \mu_2\\
    \mu_3
\end{bmatrix} \in \Theta_1. 
\end{align*}
Thus solving for $\mu_1,\mu_2$ and $\mu_3$ gives the description for $\cC_1$ as 
\begin{align*}
    \cC_1 = \Big\{\bm{\mu} \in \Real^3\; : (\mu_1 > \mu_2/2) \land (\mu_1 > \mu_2) \land (\mu_2 > 2\mu_3) \land (\mu_2 < -\mu_3) \land (\mu_1 > \mu_3) \land (\mu_1 < - \mu_3) \Big\}\;.
\end{align*} 
Similarly from the descriptions of $\Theta_2$ and $\Theta_3$, we get the following descriptions for $\cC_2$ and $\cC_3$ respectively, 
\begin{align*}
\label{eg:bandit_observation}
    \cC_2 = \Big\{\bm{\mu} \in \Real^3\; : (\mu_1 < \mu_2 < 3\mu_1) \land (-\mu_1 < \mu_3 < 3\mu_1) \land \big((\mu_3 < \mu_2 < 5\mu_3) \lor (\mu_2 > 5\mu_3 \land \mu_2 > - \mu_3)\big) \Big\}
\end{align*}
and 
\begin{align*}
    \cC_3 = \Big\{\bm{\mu} \in \Real^3\; : (\mu_1 < \mu_2/2) \land (\mu_1 < \mu_2/3) \land (\mu_3 > 3\mu_1) \land (\mu_3 > \mu_1) \land (\mu_3 > \mu_2) \land (\mu_3 >  \mu_2/2) \Big\}.
\end{align*}

In Figure \ref{fig:observation_space} we try to visualize these regions in $\Real^3$. In particular Figures \ref{fig_a:bandit_obs_1}, \ref{fig_b:bandit_obs_2} and \ref{fig_c:bandit_obs_3} represent the regions $\cC_1, \cC_2$ and $\cC_3$ respectively. Note that these images represent regions in $\Real^3$. In particular, to highlight the three dimensional nature of the regions, we have used a shading effect. The range space of the feature matrix, $\bm{\Phi}\theta$, has also been highlighted as a two dimensional plane, passing through the robust regions in Figures \ref{fig_a:bandit_obs_1} and \ref{fig_b:bandit_obs_2}, while bordering the region $\cC_3$ in Figure \ref{fig_c:bandit_obs_3}. Note that these images were plotted by restricting $\bm{\mu}$ to lie within a bounded region and so might give the appearance of being bounded. As could be deduced from the set theoretic descriptions of the \emph{robust observation regions} these are convex cones and are unbounded sets.

\subsubsection{Experiments with $\epsilon$-greedy Algorithm.}
 We sample $10$ instances from the robust region $\cC_2$. We observe from Figure \ref{fig:bandit_regret} the mean and dispersion of the cumulative regret generated by $\epsilon$-greedy algorithm with $\epsilon_t = 1/\sqrt{t}$ on these sampled instances. We also note the misspecification error (denoted by $\rho$), the maximum sub-optimality gap (denoted by $\Delta_{\max}$), and the minimum sub-optimality gap (denoted by $\Delta_{\min}$) for each of these instances. To demonstrate our results with high probability, we form confidence intervals of the cumulative regret with three standard deviations. It is observed that the cumulative regret grows at a sub-linear rate with high probability.
 
 \paragraph{Observations}We note that instances with higher $\Delta_{\max}$ values tend to have higher regret than the instances with lower $\Delta_{\max}$ values. In this regard, we note that the misspecification error $\rho$ does not influence the regret as much as does the $\Delta_{\max}$ which corroborates our theory. For example, note that the regret curve corresponding to misspecification error $\rho = 9.02$ is lower than the curve corresponding to $\rho = 0.23$. This can be explained by the fact that the $\Delta_{\max}$ of the former curve is $12.07$ whereas, for the latter curve, it is $50.69$. The fact that misspecification error does not play a big role can be observed in the near regret curves for the instances whose $\Delta_{\max}$ are the same marked as $12.07$ however, one has misspecification error $\rho$ as $1.03$ whereas the other has $9.02$. We also note the presence of one sampled instance whose $\Delta_{\min}$ is dominated by the misspecification error $\rho$ (The instance in focus has $\Delta_{\min} = 4.83$ while misspecification error $\rho$ is $9.02$). This is in sharp contrast to the type of robust instances considered in the works of \citet{pmlr-v216-liu23c}.

\begin{figure}[!htbp]
\centering
    \begin{subfigure}[\footnotesize{The robust region for arm one $\cC_1$, shown in the blue shade, is a subset of $\Real^3$. The range space of the feature matrix is depicted as the plane.}\label{fig_a:bandit_obs_1}]{\includegraphics[width = 0.5\linewidth]{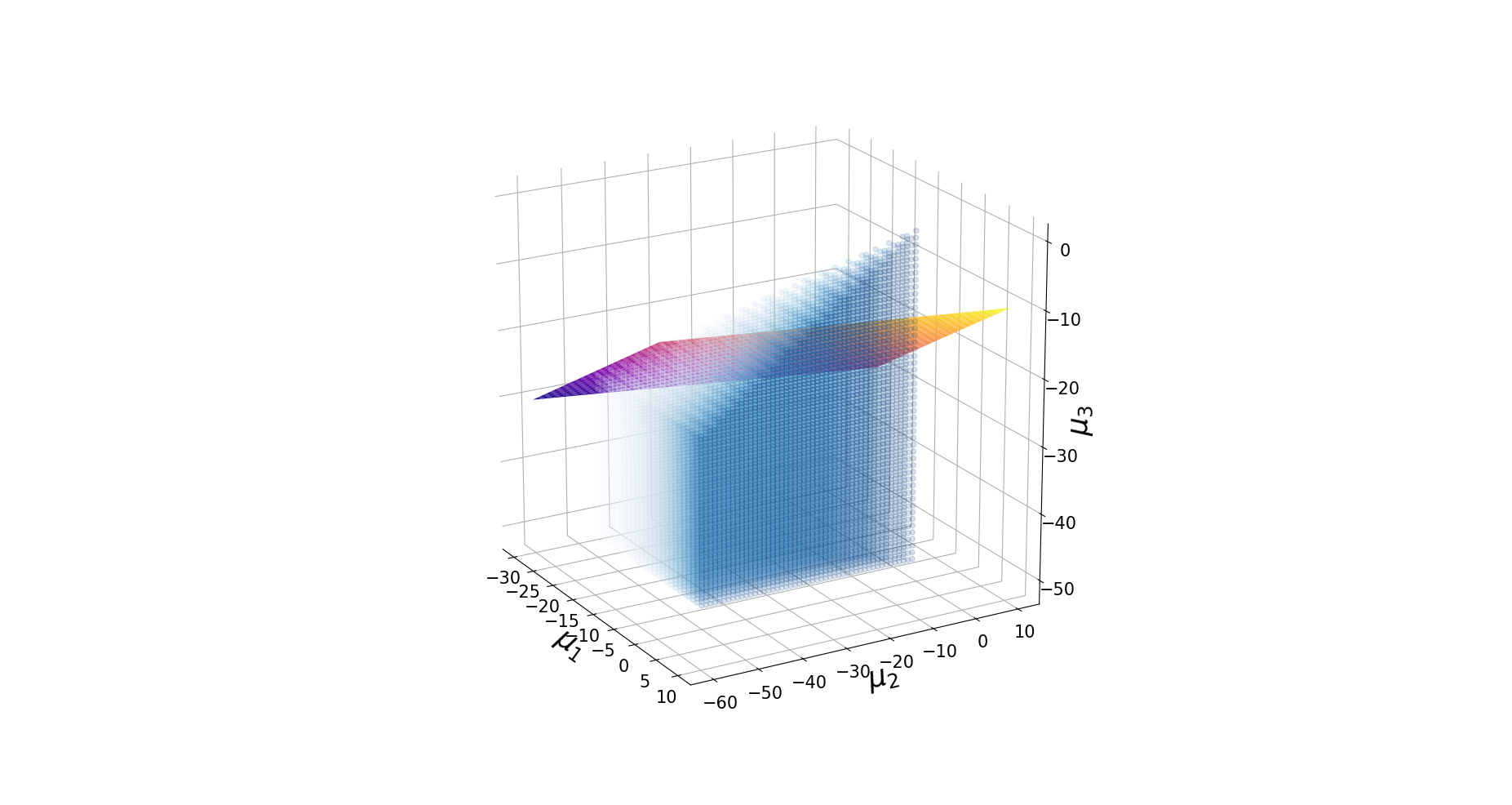}}
    \end{subfigure}
    \begin{subfigure}[\footnotesize{The robust region for arm two $\cC_2$ shown in the gray shade, is a subset of $\Real^3$. The range space of the feature matrix is depicted as the plane.}\label{fig_b:bandit_obs_2}]{\includegraphics[width = 0.5\linewidth]{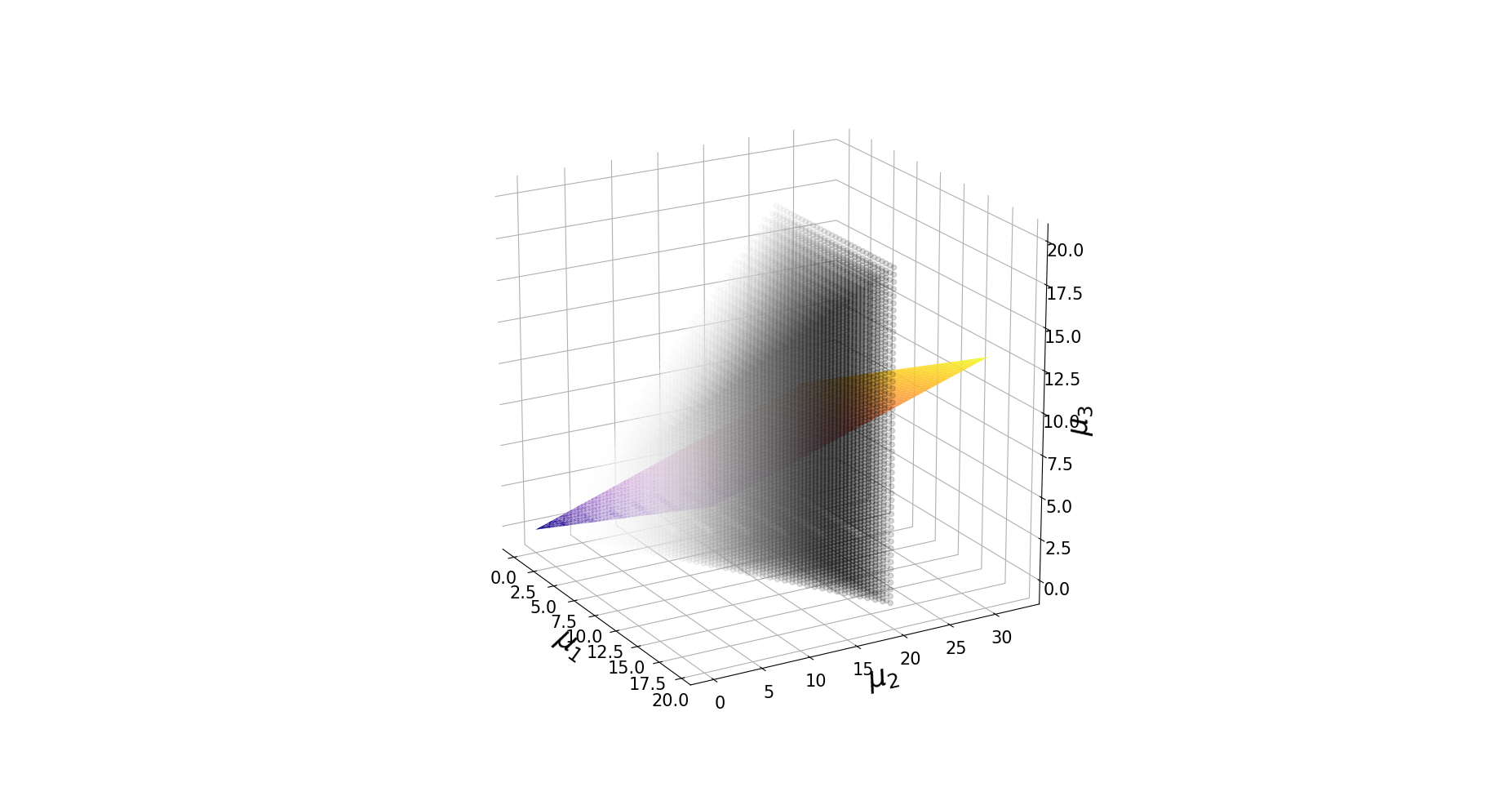}}
    \end{subfigure}
    \begin{subfigure}[\footnotesize{The robust region for arm three $\cC_3$, shown in the green shade, is a subset of $\Real^3$. The range space of the feature matrix is depicted as the plane.}\label{fig_c:bandit_obs_3}]{\includegraphics[width = 0.6\linewidth]{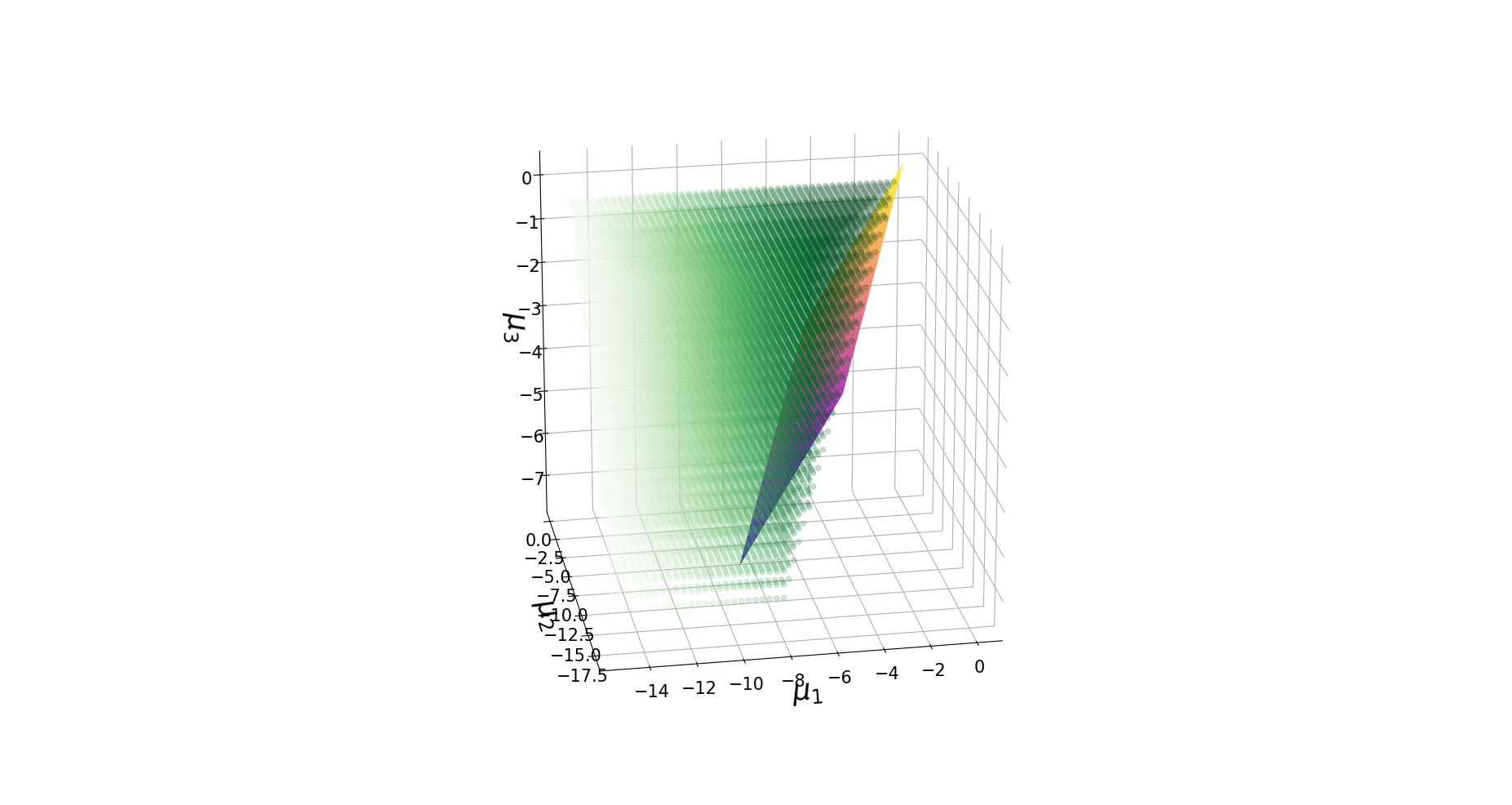}}
    \end{subfigure}
\caption{\footnotesize{Visualization of the \emph{robust observation regions} $-\cC_i$ for a three armed bandit problem, calculated for the feature matrix $\small{\bm{\Phi} =\begin{bmatrix}
    2 & 3\\
    4 & 5 \\
    2 & 1
\end{bmatrix}}$, along with the range space of the feature matrix $\bm{\Phi}\theta$. Note that these are $3$-dimensional plots with the robust regions $\cC_i$ shown in shaded regions of colors blue, gray and green. These regions are subsets of $\Real^3$ whereas the range space of the feature matrix, shown in a "plasma" color, spans $\Real^2$.}}
\label{fig:observation_space}
\end{figure}

\begin{figure}[!htbp]
        \includegraphics[width = \linewidth]{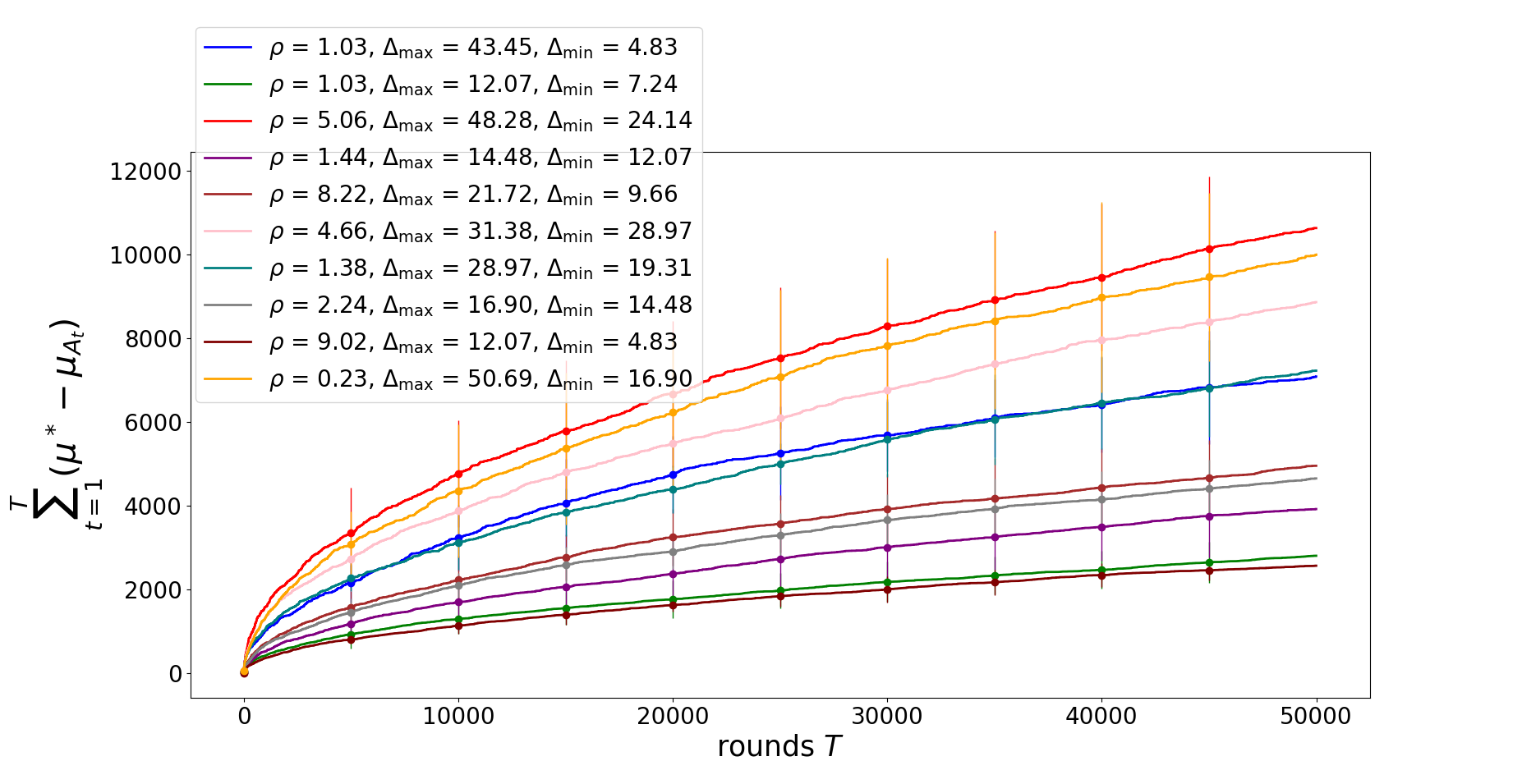}
        \caption{The growth of the cumulative regret for $10$ misspecified bandit instances sampled from the robust region of $\cC_2$ under the $\epsilon$-greedy algorithm with $\epsilon_t = 1/\sqrt{t}$. The plot represent the average of $10$ trials. The $Y$-axis denotes the cumulative regret $\sum_{t=1}^T\mu^* - \mu_{A_t}$. The $X$-axis denotes the rounds $T$. We observe the sub-linear growth trend of the cumulative regret. For each instance the values of the $l_\infty$ misspecification error ($\rho$), the maximum sub-optimality gap ($\Delta_{\max}$) and the minimum sub-optimality gap ($\Delta_{\min}$) are also noted. It is observed that instances with higher $\Delta_{\max}$ suffer more regret at any time than instances with lower $\Delta_{\max}$ as expected from our theorem.}
        \label{fig:bandit_regret}
\end{figure}

\subsection{Contextual Bandits}
\label{subsec:exp_context}
We repeat the same exercise for the contextual bandit setting. However, visualization in the contextual setting is more difficult since the reward vector for any non-trivial problem lies in a space of more than three dimensions.
\subsubsection{Robust Observation}
We shall look at a $2$ context bandit setting, with contexts $x_1$ and $x_2$, each having $3$ arms. We shall keep the feature matrix of the first context as in the example in the previous section, namely $\bm{\Phi_{x_1}} = \begin{bmatrix}
    2 & 3\\
    4 & 5 \\
    2 & 1
    \end{bmatrix}$ and we shall set the feature matrix for the second context as $\bm{\Phi_{x_2}} = \begin{bmatrix}
        2 & 3 \\
        4 & 5 \\
        6 & 7
    \end{bmatrix}$. Note that for this setting, contextual rewards are in $\Real^6$, and thus the \emph{robust observation regions} would be of dimension at most $6$. From the definition of the \emph{robust parameter region} for context $x_1$ and the calculation done in the previous section we have, 

\begin{align*}
&\Theta^{x_1}_1 = \{ \theta \in \Real^2\; : \theta_1 < -\theta_2 \land \theta_2 > 0\}  \\
 &\Theta^{x_1}_2 = \{ \theta \in \Real^2\; : (\theta_1 > 0 \land \theta_2 > - \theta_1/2) \lor (\theta_1 < 0 \land \theta_2 > - \theta_1) \} \\
 &\Theta^{x_1}_3 = \{ \theta \in \Real^2\; : \theta_2 < 0 \land \theta_2 < - \theta_1/2\}.
\end{align*}
For context $x_2$, we can repeat the same procedure to find the \emph{robust parameter space for context $x_2$} as 
\begin{align*}
&\Theta^{x_2}_1 = \{ \theta \in \Real^2\; : \theta_1 < -\theta_2 \}  \\
&\Theta^{x_2}_3 = \{ \theta \in \Real^2\; : \theta_1 > -\theta_2\}.
\end{align*}
Note that for this feature matrix choice we have found that the robust parameter region for arm two at context $x_2$, $\Theta^{x_2}_2$ is $\emptyset$.


We shall now use Theorem \ref{thm:context_convex_hull} to evaluate the \emph{robust observation regions}, namely a bandit instance $\bm{\mu} \in \cC^{x_1}_i \cap \cC^{x_2}_j$ if and only if the basic solutions belong to $\Theta^{x_1}_i \cap \Theta^{x_2}_j$.\footnote{See Theorem \ref{thm:context_convex_hull} for what we mean by basic solutions} 
The non empty intersections of $\Theta^{x_1}_i$ and $\Theta^{x_2}_j$ are as follows
\begin{align*}
    &\Theta^{x_1}_1 \cap \Theta^{x_2}_1 = \Theta^{x_1}_1 \\
    &\Theta^{x_1}_3 \cap \Theta^{x_2}_1 = \{\theta \in \Real^2 \;:\; \theta_1 < -\theta_2 \land \theta_2 < 0\} \\
    &\Theta^{x_1}_2\cap \Theta^{x_2}_3 = \Theta^{x_1}_2 \\
    &\Theta^{x_1}_3 \cap \Theta^{x_2}_3 = \{\theta \in \Real^2 \;:\; \theta_2 <0, -\theta_1 < \theta_2 <-\theta_1/2\}.
\end{align*}
Thus $\cC^{x_1}_1 \cap \cC^{x_2}_1$ is the set of all $\bm{\mu} \in \Real^6$, such that the basic solutions are in $\Theta^{x_1}_1 \cap \Theta^{x_2}_1 = \Theta^{x_1}_1$.

Solving, gives us the description of the set $\cC^{x_1}_1 \cap \cC^{x_2}_1$ as the set of all $\bm{\mu} \in \Real^6$ such that the following conditions are satisfied

\begin{align*}
    \cC^{x_1}_1 \cap \cC^{x_2}_1 = \Big\{ \bm{\mu} \in \Real^6\;:\;& \mu_1 > \mu_2 \land \mu_1 > \mu_2/2, \\
                & \mu_1 > \mu_3 \land \mu_1 < -\mu_3, \\
                & \mu_1 > \mu_5 \land \mu_1 >  \mu_5/2, \\
                & \mu_1 > \mu_6 \land \mu_1 > \mu_6/3, \\
                & \mu_2 > 2\mu_3 \land \mu_2 < -\mu_3, \\
                & \mu_4 > \mu_2/2 \land \mu_2 < \mu_4, \\
                & \mu_2 > \mu_6 \land \mu_2 > 2/3\mu_6, \\
                & \mu_4 > \mu_3 \land \mu_3 < -\mu_4, \\
                & \mu_5 > 2\mu_3 \land \mu_3 < -\mu_5, \\
                & \mu_6 > 3\mu_3 \land \mu_3 < -\mu_6, \\
                & \mu_4 > \mu_5 \land \mu_4 > \mu_5/2, \\
                & \mu_4 > \mu_6 \land \mu_4 > \mu_6/3, \\
                & \mu_5 > \mu_6 \land \mu_5 > 2/3\mu_6 \Big\}.
\end{align*}

In Figure \ref{fig:context_observation_eg} we try to visualize the region $\cC^{x_1}_1 \cap \cC^{x_2}_1$ an element in $\Real^6$ in terms of its projections in the $\Real^3$ space. To do so, we have taken projections of the six dimensional region in every possible three dimensional basis representation of $(\mu_i, \mu_j, \mu_k)$ and plotted it. For example the first image in the top most row of Figure \ref{fig:context_observation_eg} is the plot of $(\mu_1, \mu_2,, \mu_3)$ for any $\bm{\mu} \in \cC^{x_1}_1 \cap \cC^{x_2}_1$. Similarly the plot to the immediate right of it is  the plot of $(\mu_1, \mu_2, \mu_4)$ and so on. Note that each of these plots represents a three dimensional projection of the high dimensional region. To highlight the $3$-D nature of these plots, we have used a gradient color shading  to give the perception of depth\footnote{The specific type of shading we have used is 'viridis' in Matplotlib}. Note that these images were plotted by restricting $\bm{\mu}$ to lie within a bounded region and so might give the appearance of being bounded. As could be deduced from the set theoretic descriptions of the \emph{robust observation regions} these are convex cones and are unbounded sets.

As an exercise we also compute $\cC^{x_1}_3 \cap \cC^{x_2}_1$, the robust region of all contextual bandit problems where context $x_1$ has arm $3$ as optimal and context $x_2$ has arm $1$ as optimal. In the same procedure, we calculate the region $\cC^{x_1}_3 \cap \cC^{x_2}_1$ as the set of all $\bm{\mu} \in \Real^6$ such that the following conditions are satisfied 

\begin{align*}
    \cC^{x_1}_3 \cap \cC^{x_2}_1 = \Big\{ \bm{\mu} \in \Real^6\;:\;& \mu_1 > \mu_2 \land \mu_1 < \mu_2/2, \\
                & \mu_1 < \mu_3 \land \mu_1 < -\mu_3, \\
                & \mu_1 > \mu_5 \land \mu_1 <  \mu_5/2, \\
                & \mu_1 > \mu_6 \land \mu_1 < \mu_6/3, \\
                & \mu_2 < 2\mu_3 \land \mu_2 < -\mu_3, \\
                & \mu_4 < \mu_2/2 \land \mu_2 < \mu_4, \\
                & \mu_2 > \mu_6 \land \mu_2 < 2/3\mu_6, \\
                & \mu_4 < \mu_3 \land \mu_3 < -\mu_4, \\
                & \mu_5 < 2\mu_3 \land \mu_3 < -\mu_5, \\
                & \mu_6 < 3\mu_3 \land \mu_3 < -\mu_6, \\
                & \mu_4 > \mu_5 \land \mu_4 < \mu_5/2, \\
                & \mu_4 > \mu_6 \land \mu_4 < \mu_6/3, \\
                & \mu_5 > \mu_6 \land \mu_5 < 2/3\mu_6 \Big\}.
\end{align*}

In Figure \ref{fig:context_observation_eg_2} we plot the region $\cC^{x_1}_3 \cap \cC^{x_2}_1$, in the same principle as in the previous visualization. In this figure, we have also plotted the range space of the feature matrix $\bm{\Phi}\theta$ along every possible three-dimensional basis representation of $(\mu_i, \mu_j, \mu_k)$ in a shade of 'plasma'. For example, the first image of the topmost row of Figure \ref{fig:context_observation_eg_2} shows the plot of the first three coordinates of $\bm{\Phi}\theta$ as $\theta$ ranges in $\Theta^{x_1}_3 \cap \Theta^{x_2}_1$ alongside the projection of the robust observation region. Similarly, the figure to the immediate right is the plot of the first, second, and fourth coordinates. Note that the range space of $\bm{\Phi}\theta$ appears as a plane in every image. This is expected since the $\bm{\Phi}$ rank is two. 

\paragraph{Experiments with $\epsilon$-greedy Algorithm.}  We observe from Figure \ref{fig:context_regret} the mean and dispersion of the cumulative regret on $10$ contextual bandit instances, sampled from the robust region $\cC^{x_1}_1 \cap \cC^{x_2}_1$, for the $\epsilon$-greedy algorithm with $\epsilon_t = 1/\sqrt{t}$ for different context distributions. We also note the misspecification error (denoted by $\rho$), the maximum sub-optimality gap (denoted by $\Delta_{\max}$), and the minimum sub-optimality gap (denoted by $\Delta_{\min}$) for each such sampled instance. To demonstrate our results with high probability, we form confidence intervals of the cumulative regret with three standard deviations. In Figure \ref{fig:context_regret_02} we repeat the same experiment but for samples drawn from $\cC^{x_1}_3 \cap \cC^{x_2}_1$. It is observed for both these figures that the growth of the cumulative regret is sub-linear with high probability.  

\paragraph{Observations}We note that instances with higher $\Delta_{\max}$ values tend to have higher regret than the instances with lower $\Delta_{\max}$ values. In this regard, we note that the misspecification error $\rho$ does not influence the regret as much as does the $\Delta_{\max}$ which corroborates our theory. For example, note that in Figure \ref{fig_c:context_regret_eg2_09}, the regret curve corresponding to misspecification error $\rho = 2.11$ is lower than the curve corresponding to $\rho = 1.58$. This can be explained by the fact that the $\Delta_{\max}$ of the former curve is $10.53$ whereas, for the latter curve, it is $25.26$. The fact that misspecification error does not play a big role can be observed in the different regret curves for the instances whose misspecification error $\rho$ are the same marked as $2.11$ but has different $\Delta_{\max}$ as $21.05,\;\; 16.84$ and $10.53$. We also note the presence of one sampled instance whose $\Delta_{\min}$ is dominated by the misspecification error $\rho$ (The instance in focus has $\Delta_{\min} = 2.11$ while misspecification error $\rho$ is $2.63$). This is in sharp contrast to the type of robust instances considered in the works of \citet{zhang2023interplay}.

\begin{figure}[htbp]
\centering
    \hspace{-1.6cm}
    \includegraphics[width=1.5in]{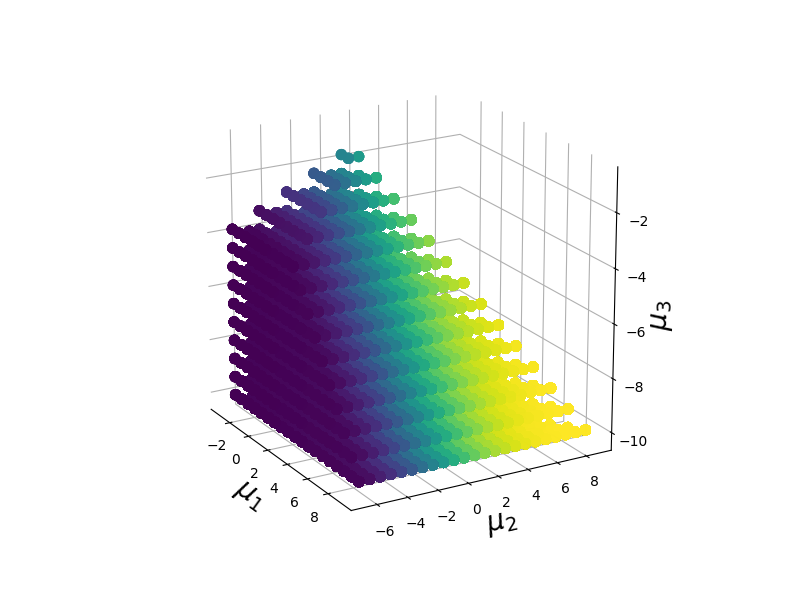}\hspace{-0.95cm}
    \includegraphics[width=1.5in]{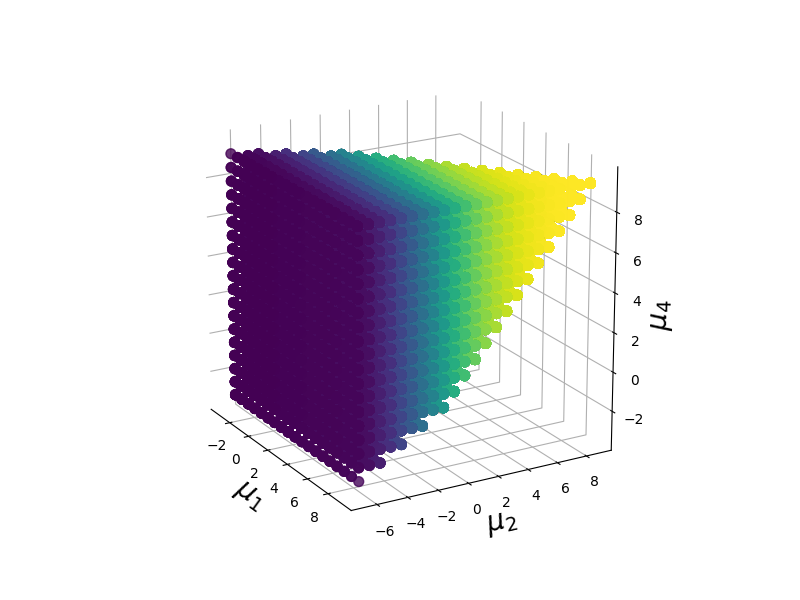}\hspace{-0.95cm}
    \includegraphics[width=1.5in]{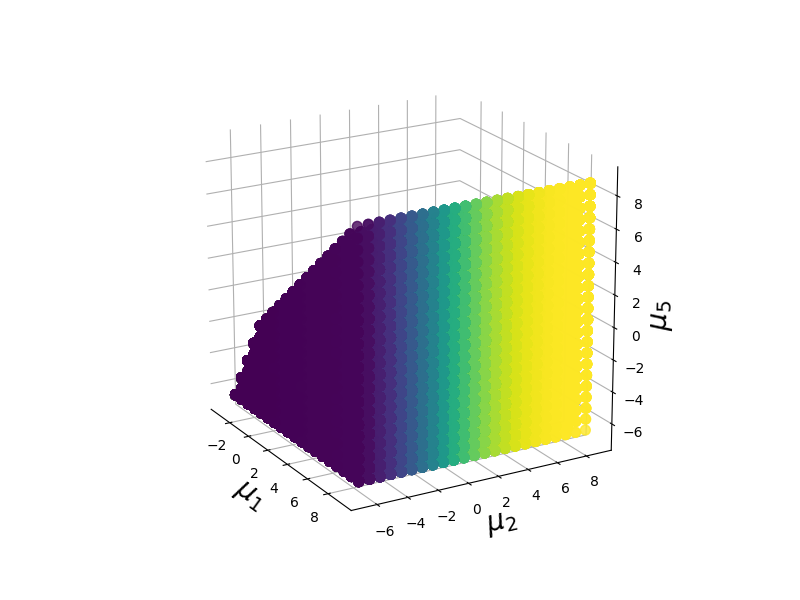}\hspace{-0.95cm}
    \includegraphics[width=1.5in]{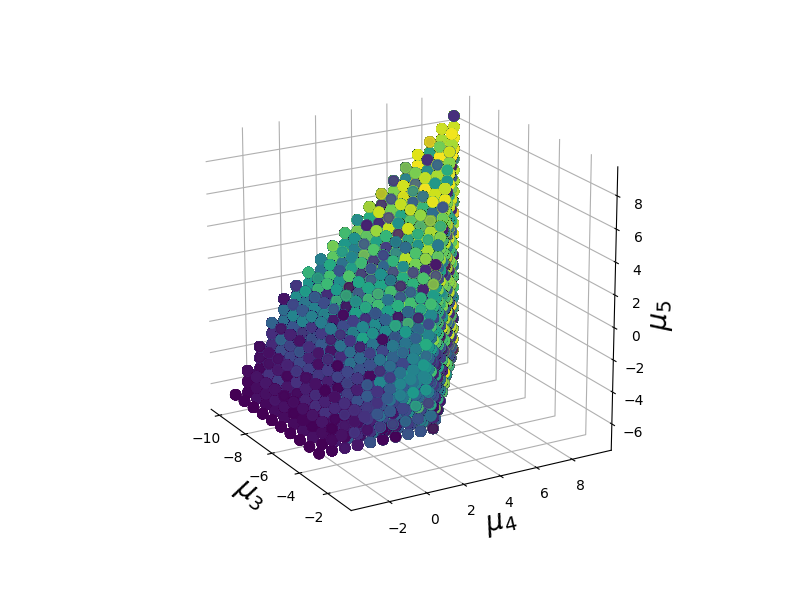}
    \\[\smallskipamount]
    \hspace{-1.6cm}
    \includegraphics[width=1.5in]{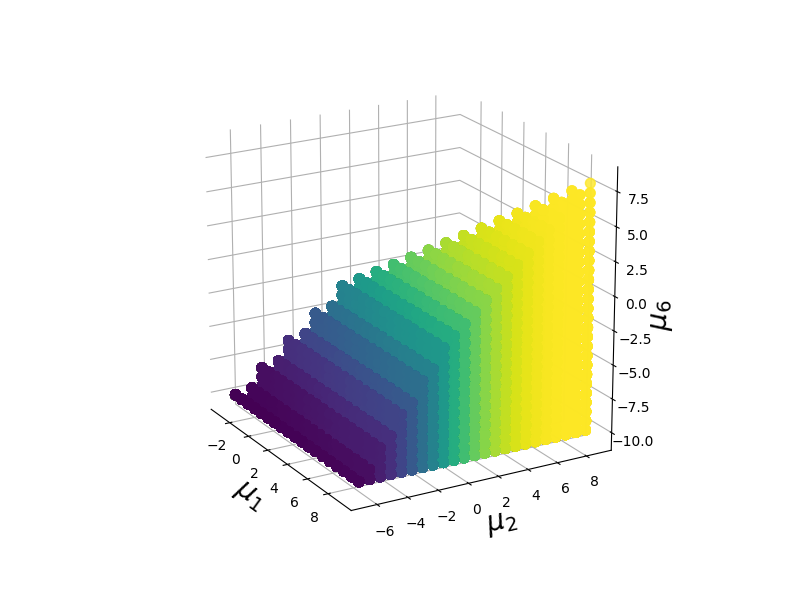}\hspace{-0.95cm}
    \includegraphics[width=1.5in]{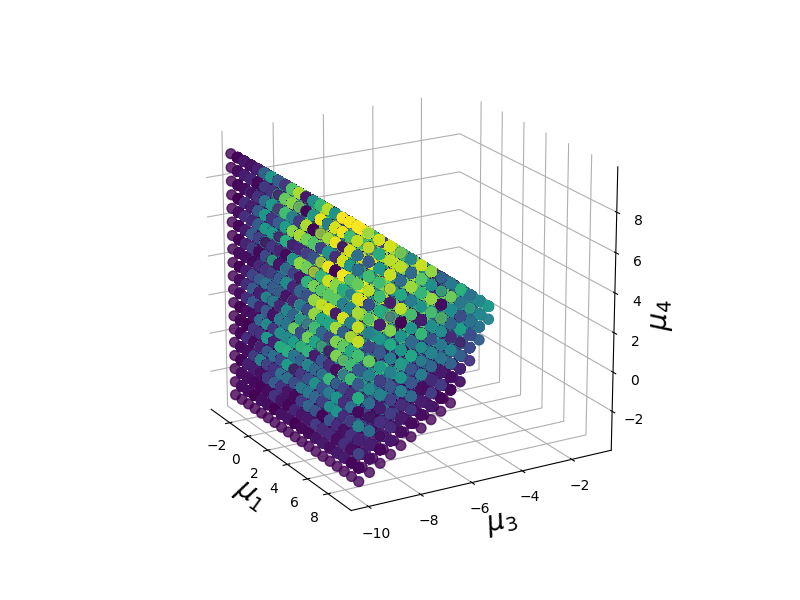}\hspace{-0.95cm}
    \includegraphics[width=1.5in]{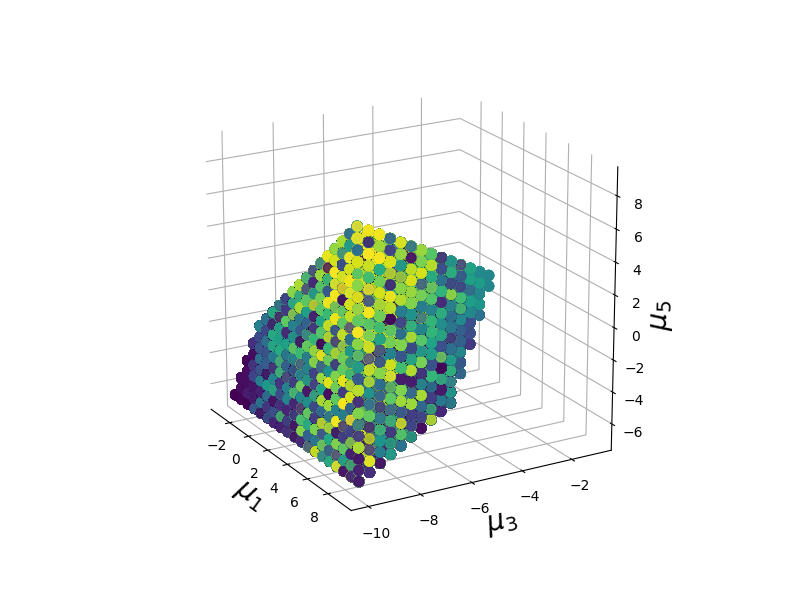}\hspace{-0.95cm}
    \includegraphics[width=1.5in]{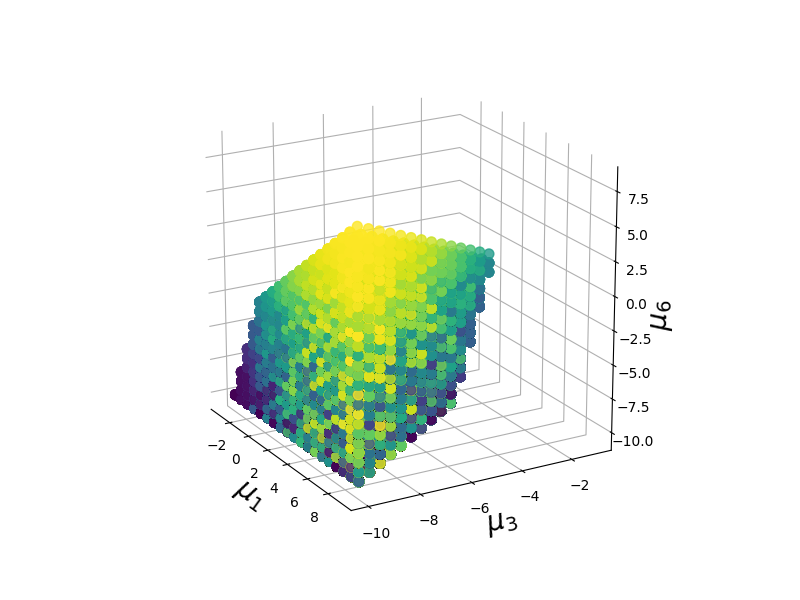}
    \\[\smallskipamount]
    \hspace{-1.6cm}
    \includegraphics[width=1.5in]{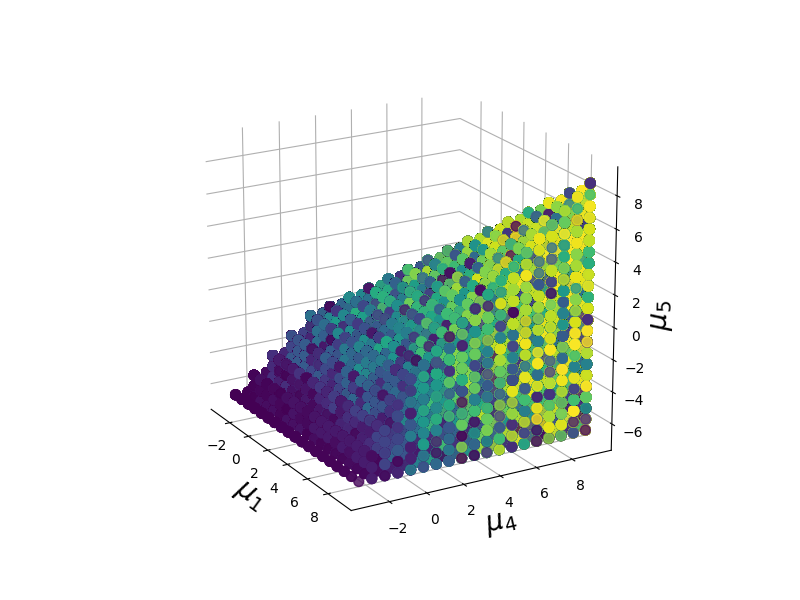}\hspace{-0.95cm}
    \includegraphics[width=1.5in]{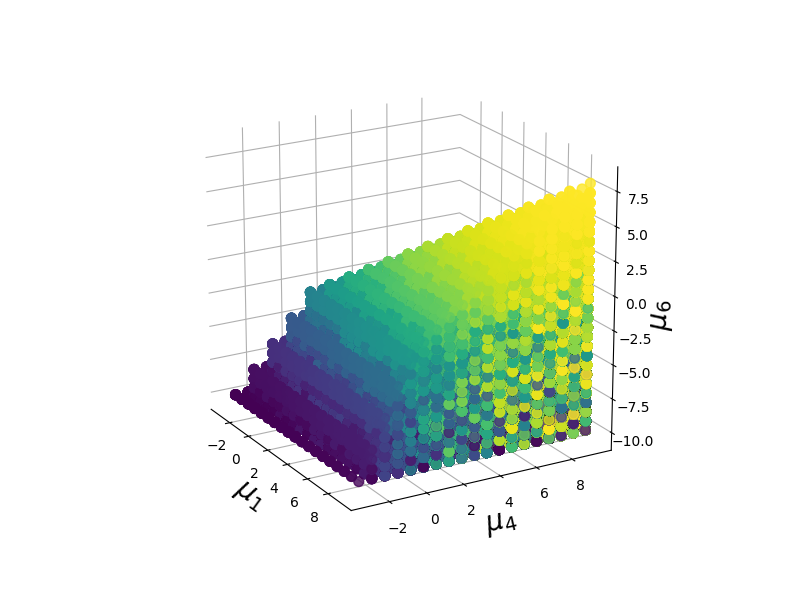}\hspace{-0.95cm}
    \includegraphics[width=1.5in]{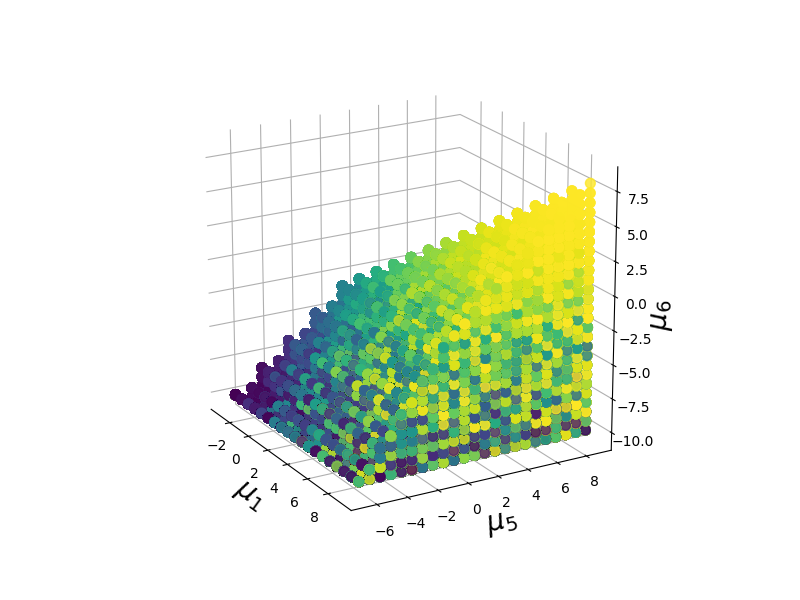}\hspace{-0.95cm}
    \includegraphics[width=1.5in]{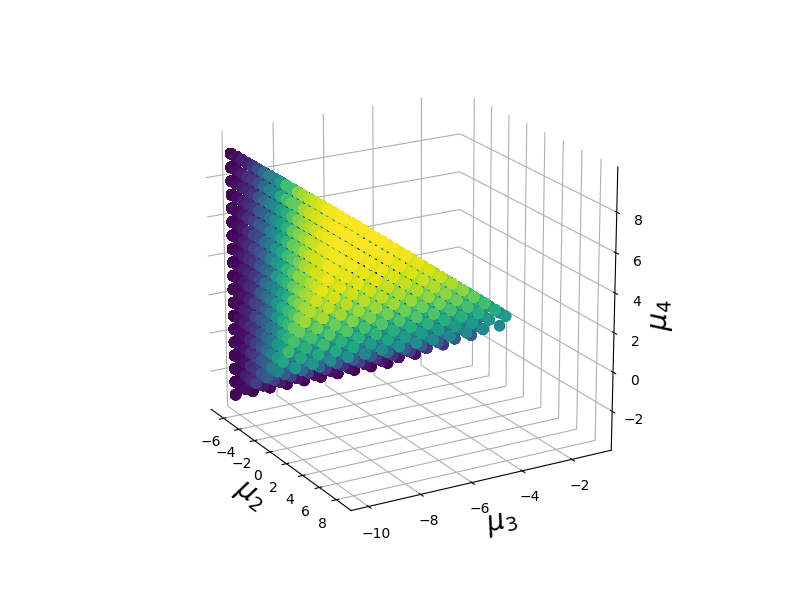}
    \\[\smallskipamount]
    \hspace{-1.6cm}
    \includegraphics[width=1.5in]{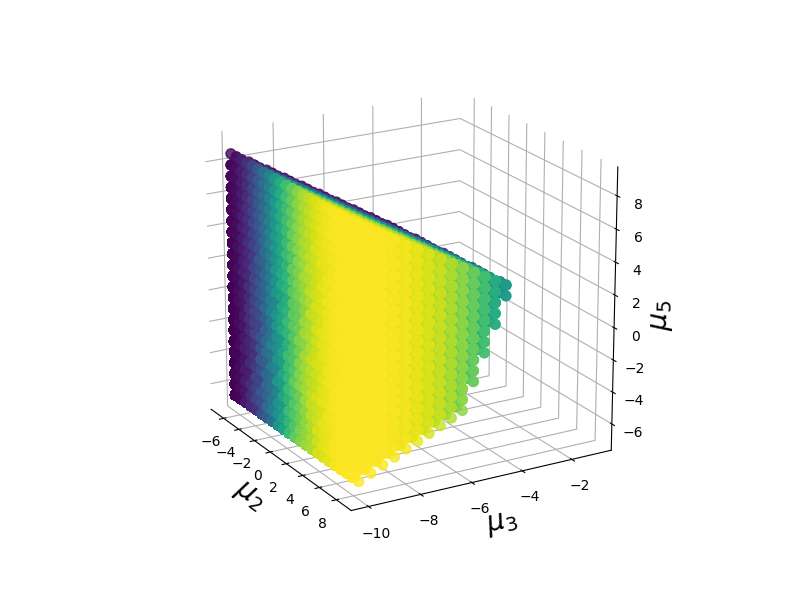}\hspace{-0.95cm}
    \includegraphics[width=1.5in]{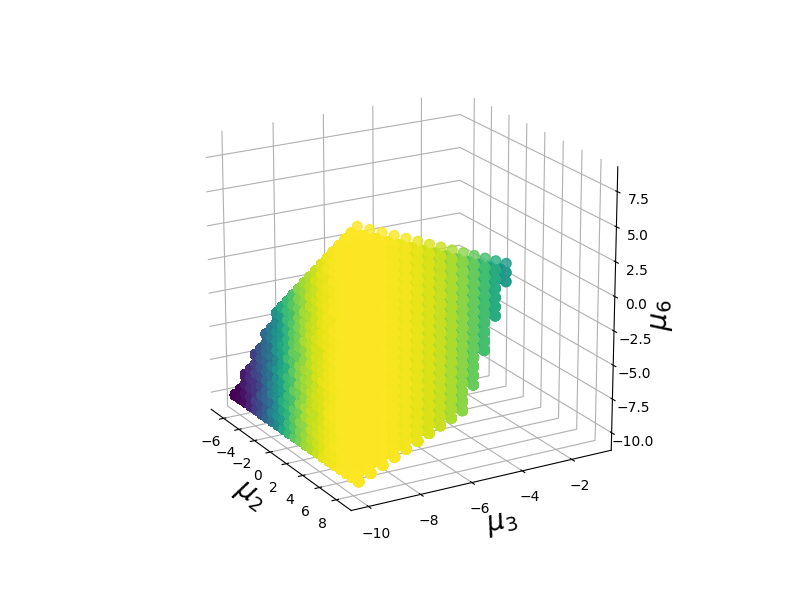}\hspace{-0.95cm}
    \includegraphics[width=1.5in]{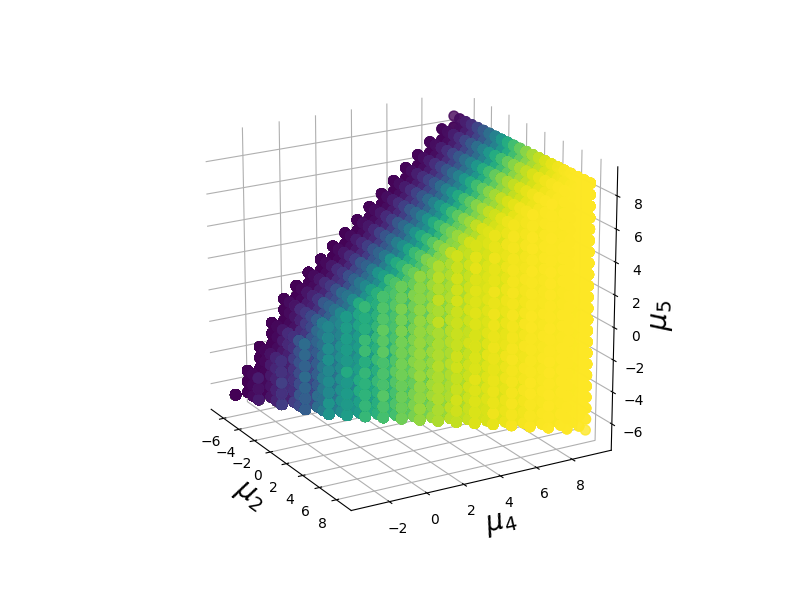}\hspace{-0.95cm}
    \includegraphics[width=1.5in]{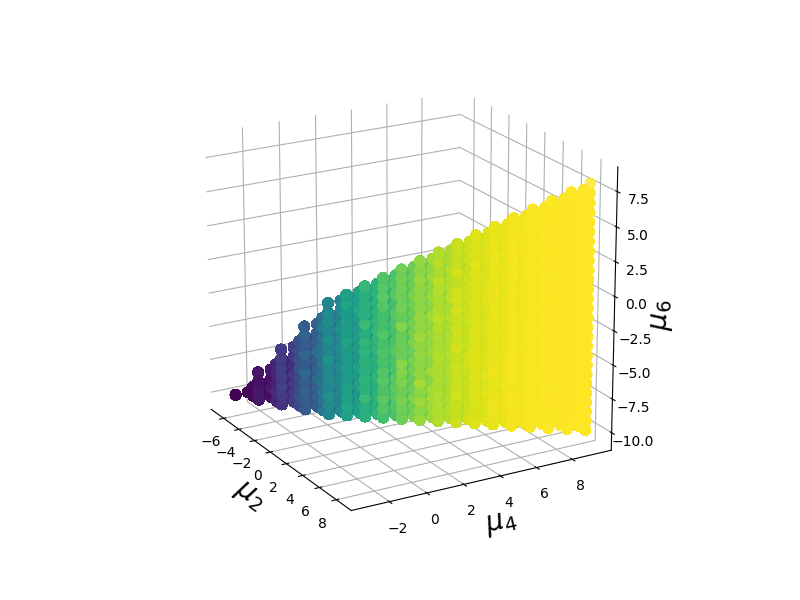}
    \\[\smallskipamount]
    \hspace{-1.6cm}
    \includegraphics[width=1.5in]{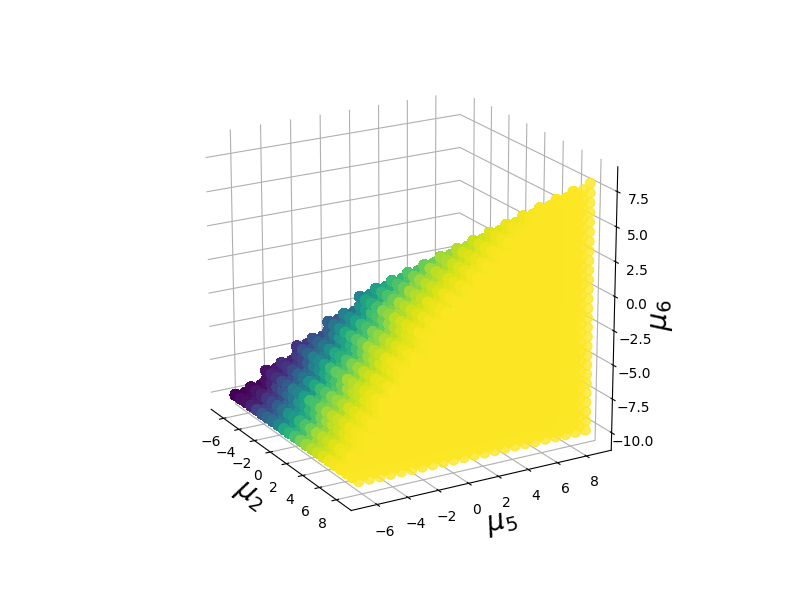}\hspace{-0.95cm}
    \includegraphics[width=1.5in]{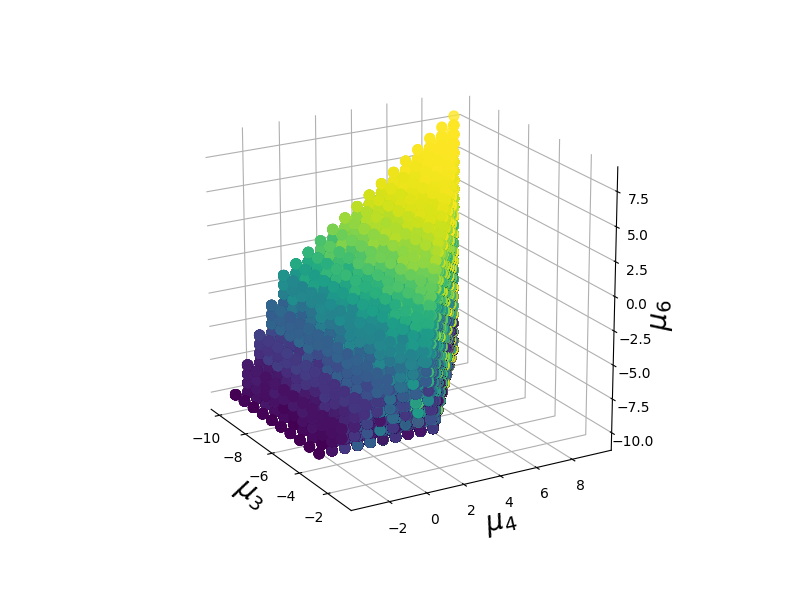}\hspace{-0.95cm}
    \includegraphics[width=1.5in]{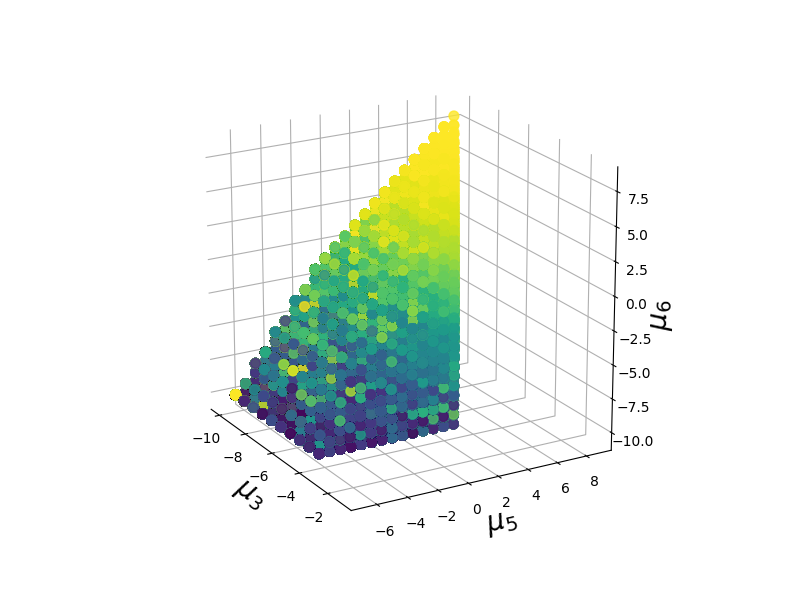}\hspace{-0.95cm}
    \includegraphics[width=1.5in]{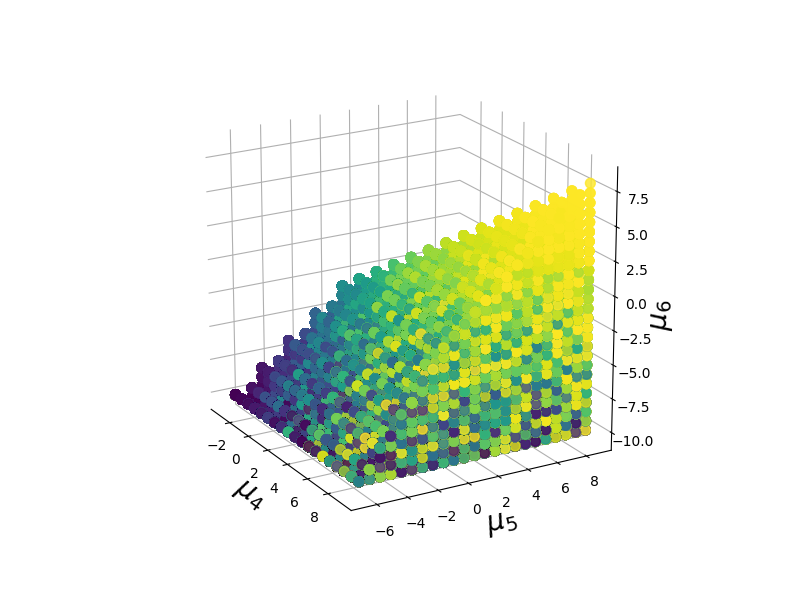}
    \caption{Visualization of the \emph{robust observation region} $\cC^{x_1}_3 \cap \cC^{x_2}_1$ for the contextual bandit problem, described in Section \ref{subsec:exp_context}. The robust region is a subset in $\Real^6$. We are plotting the projections of the high-dimensional region along all possible three-dimensional bases. Each figure represents a $3$-dimensional projection of the robust region along some orientation of the basis vectors. The projected regions are shown as a shaded color gradient to give the perception of depth in $\Real^3$.}
    \label{fig:context_observation_eg}
\end{figure}

\begin{figure}[!htbp]
    \begin{subfigure}[$\mathbf{p}_{x_1} = 0.1$ and $\mathbf{p}_{x_2} = 0.9$ \label{fig_a:context_regret_eg1_01}]{\includegraphics[width = 0.5\linewidth]{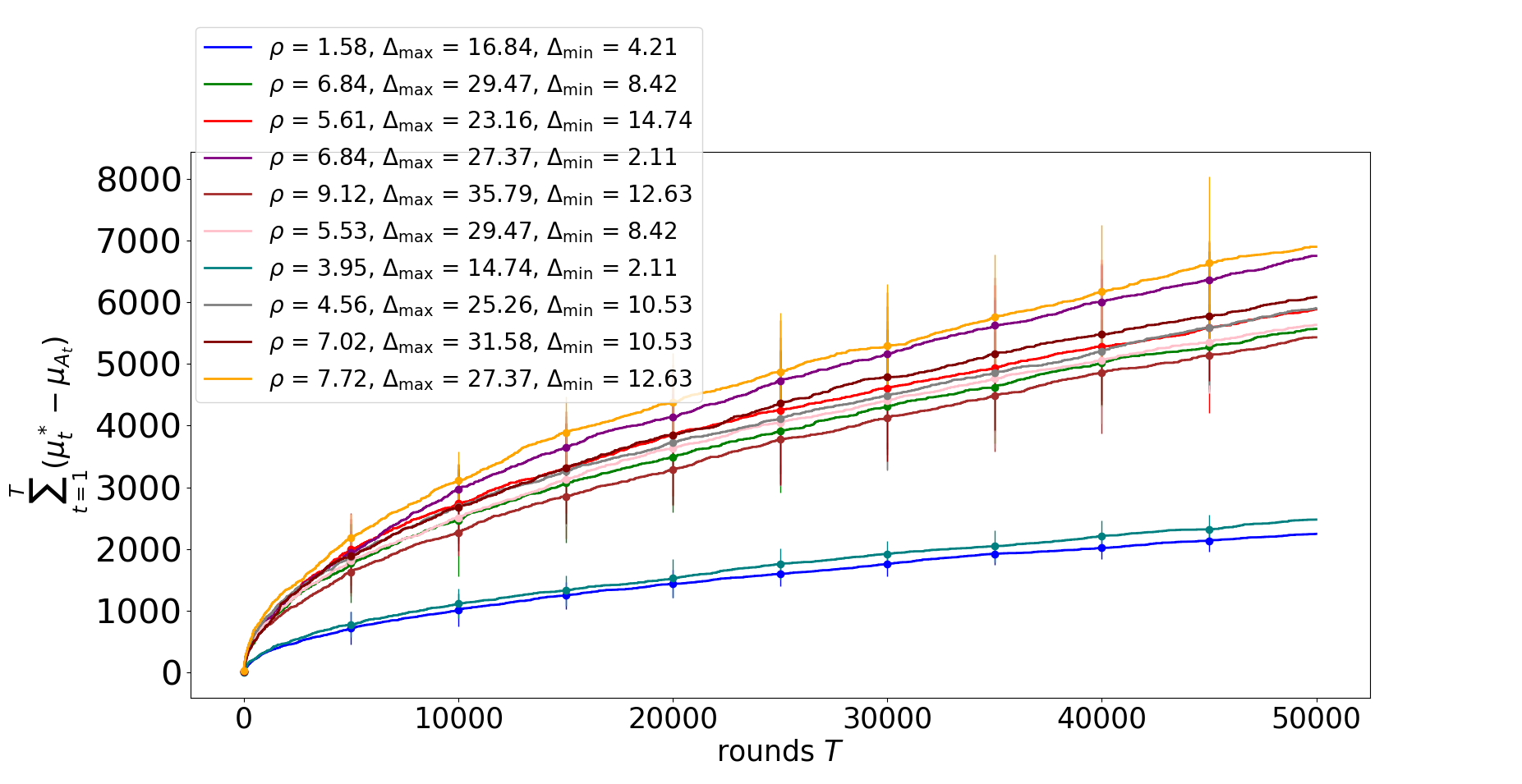}}
    \end{subfigure}
    \begin{subfigure}[$\mathbf{p}_{x_1} = 0.5$ and $\mathbf{p}_{x_2} = 0.5$ \label{fig_b:context_regret_eg1_05}]{\includegraphics[width = 0.5\linewidth]{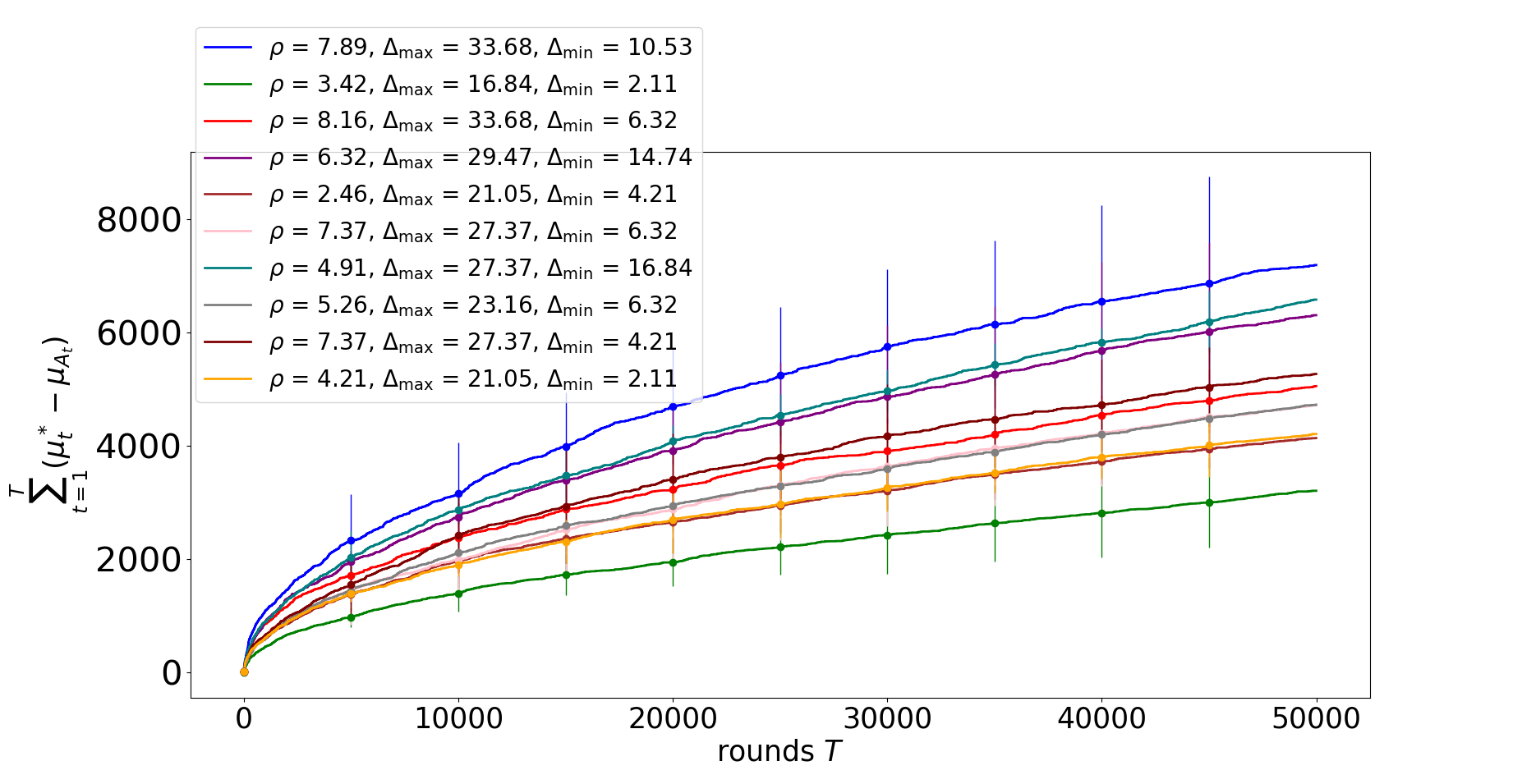}}
    \end{subfigure}
    
    \centering
    \begin{subfigure}[$\mathbf{p}_{x_1} = 0.9$ and $\mathbf{p}_{x_2} = 0.1$ \label{fig_c:context_regret_eg1_09}]{\includegraphics[width = 0.8\linewidth]{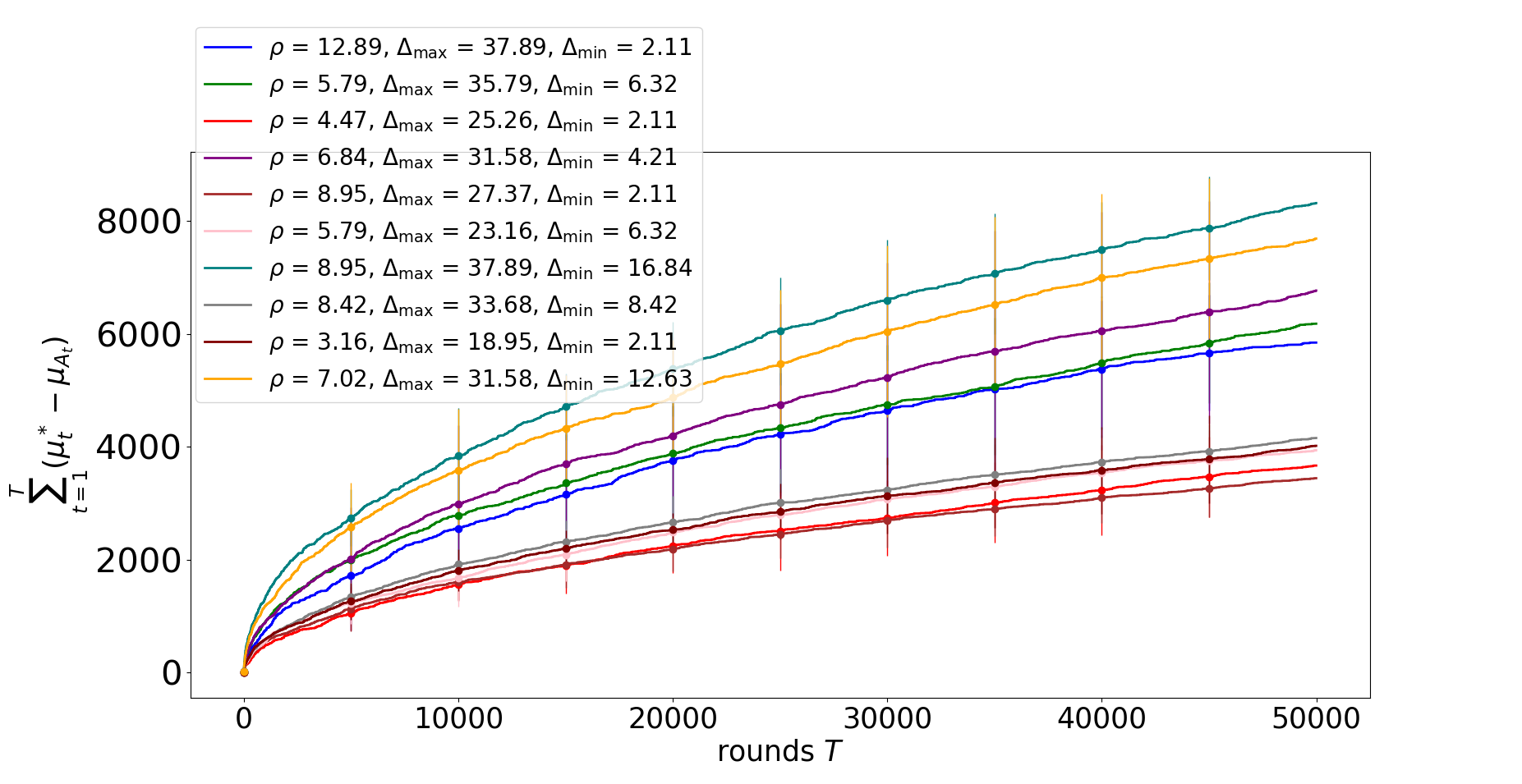}}
    \end{subfigure}    
    
    \caption{The growth of the cumulative regret for $10$ misspecified bandit instances sampled from the robust region of $\cC^{x_1}_1 \cap \cC^{x_2}_1$ under the $\epsilon$-greedy algorithm with $\epsilon_t = 1/\sqrt{t}$ for different context distributions, $\mathbf{p}_{x_1}$ and $\mathbf{p}_{x_2}$. The plots represent the average of $10$ trials. The $Y$-axis denotes the cumulative regret $\sum_{t=1}^T\mu_t^* - \mu_{A_t}$. The $X$-axis denotes the rounds $T$. We observe the sub-linear growth trend of the cumulative regret. For each instance the values of the $l_\infty$ misspecification error ($\rho$), the maximum sub-optimality gap ($\Delta_{\max}$) and the minimum sub-optimality gap ($\Delta_{\min}$) are also noted. It is observed that instances with higher $\Delta_{\max}$ suffer more regret at any time than instances with lower $\Delta_{\max}$ as expected from our theorem.}
    \label{fig:context_regret}
\end{figure}

\begin{figure}[htbp]
\centering
    \hspace{-1.6cm}
    \includegraphics[width=1.5in]{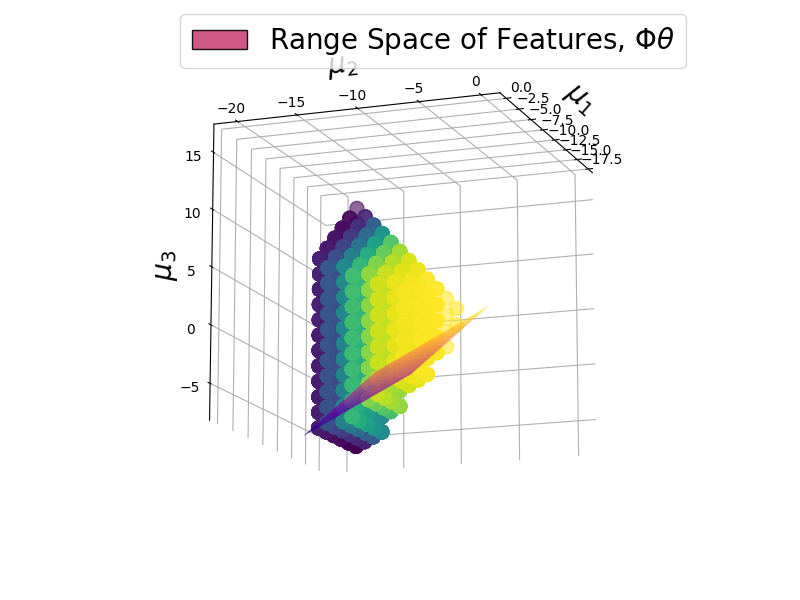}\hspace{-0.90cm}
    \includegraphics[width=1.5in]{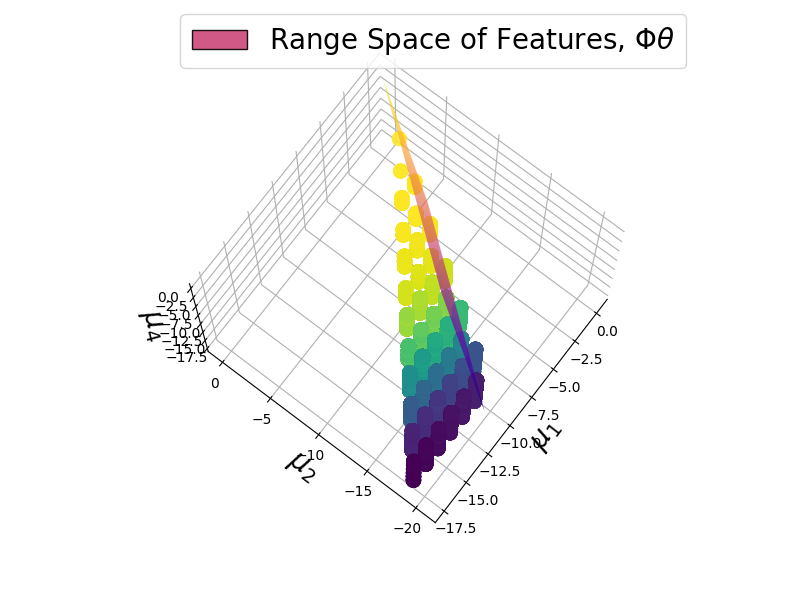}\hspace{-0.90cm}
    \includegraphics[width=1.5in]{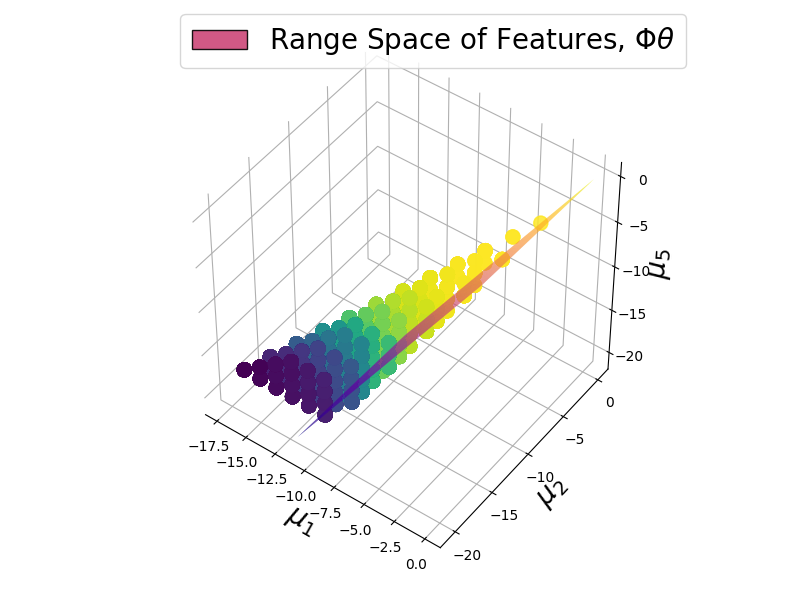}\hspace{-0.90cm}
    \includegraphics[width=1.5in]{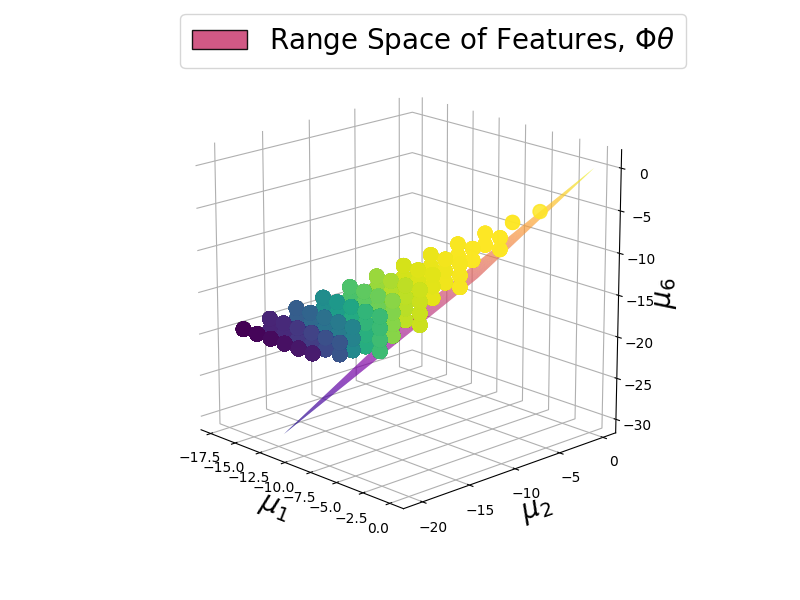}
    \\[\smallskipamount]
    \hspace{-1.6cm}
    \includegraphics[width=1.5in]{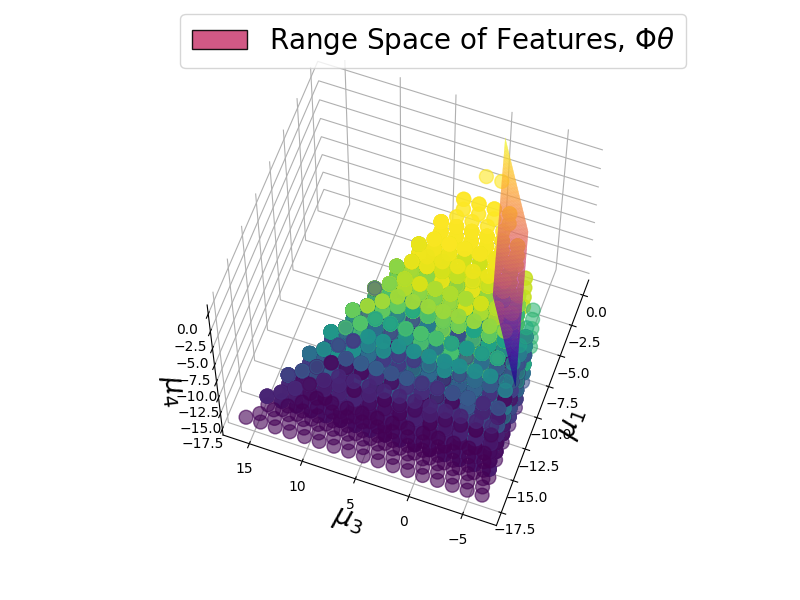}\hspace{-0.90cm}
    \includegraphics[width=1.5in]{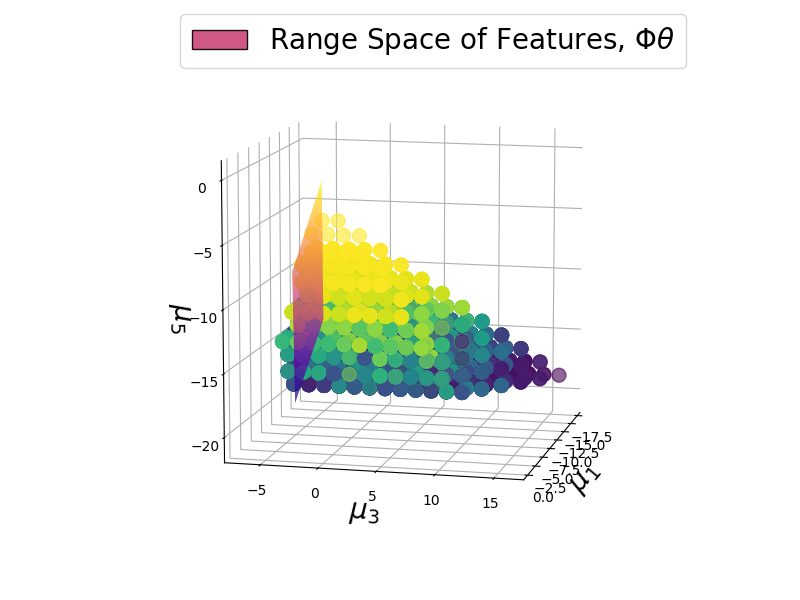}\hspace{-0.90cm}
    \includegraphics[width=1.5in]{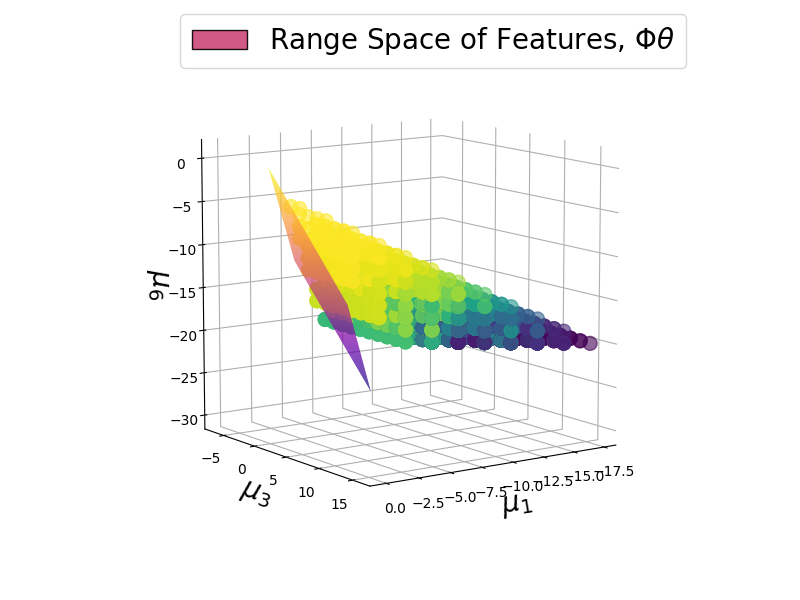}\hspace{-0.90cm}
    \includegraphics[width=1.5in]{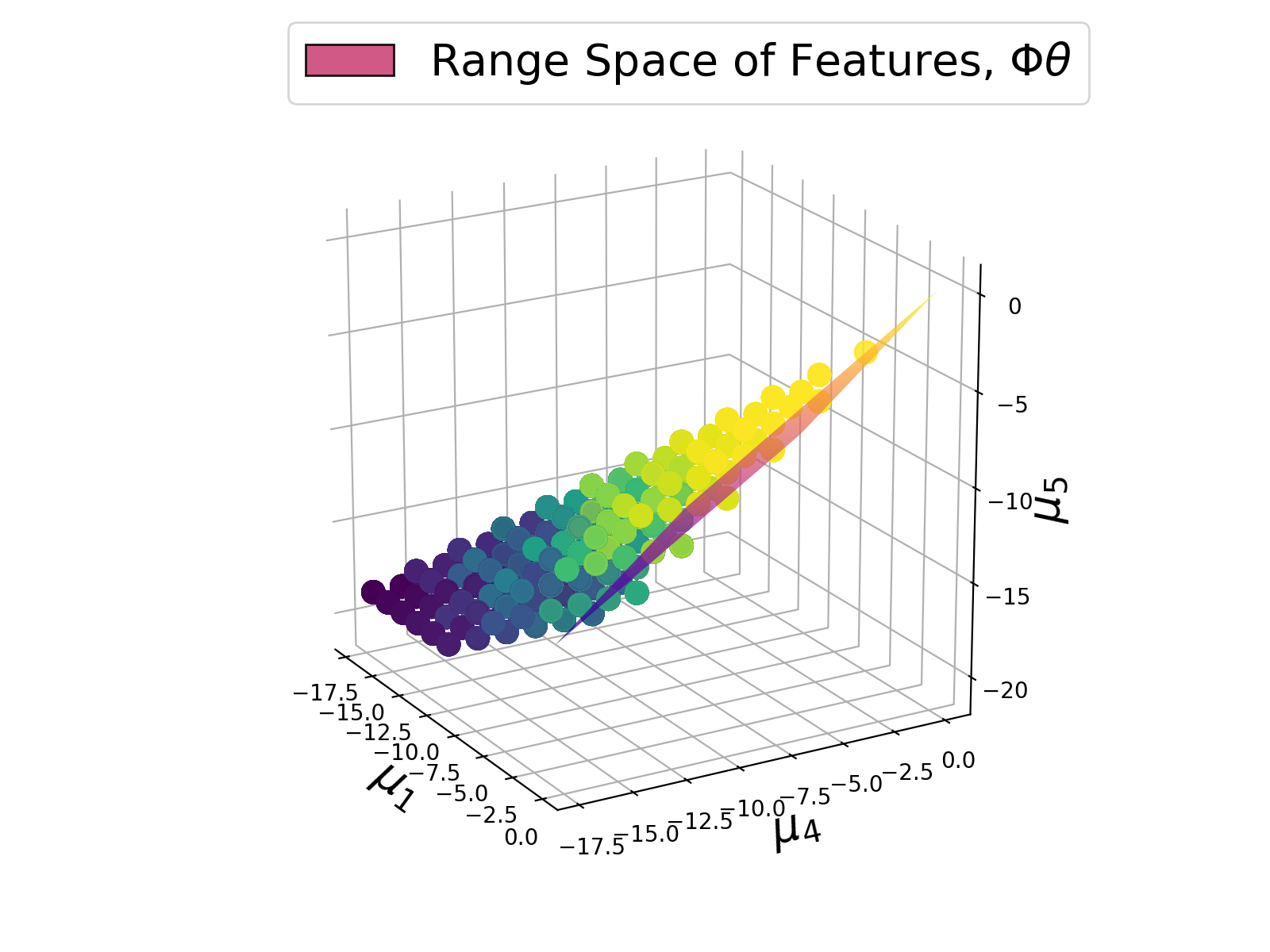}
    \\[\smallskipamount]
    \hspace{-1.6cm}
    \includegraphics[width=1.5in]{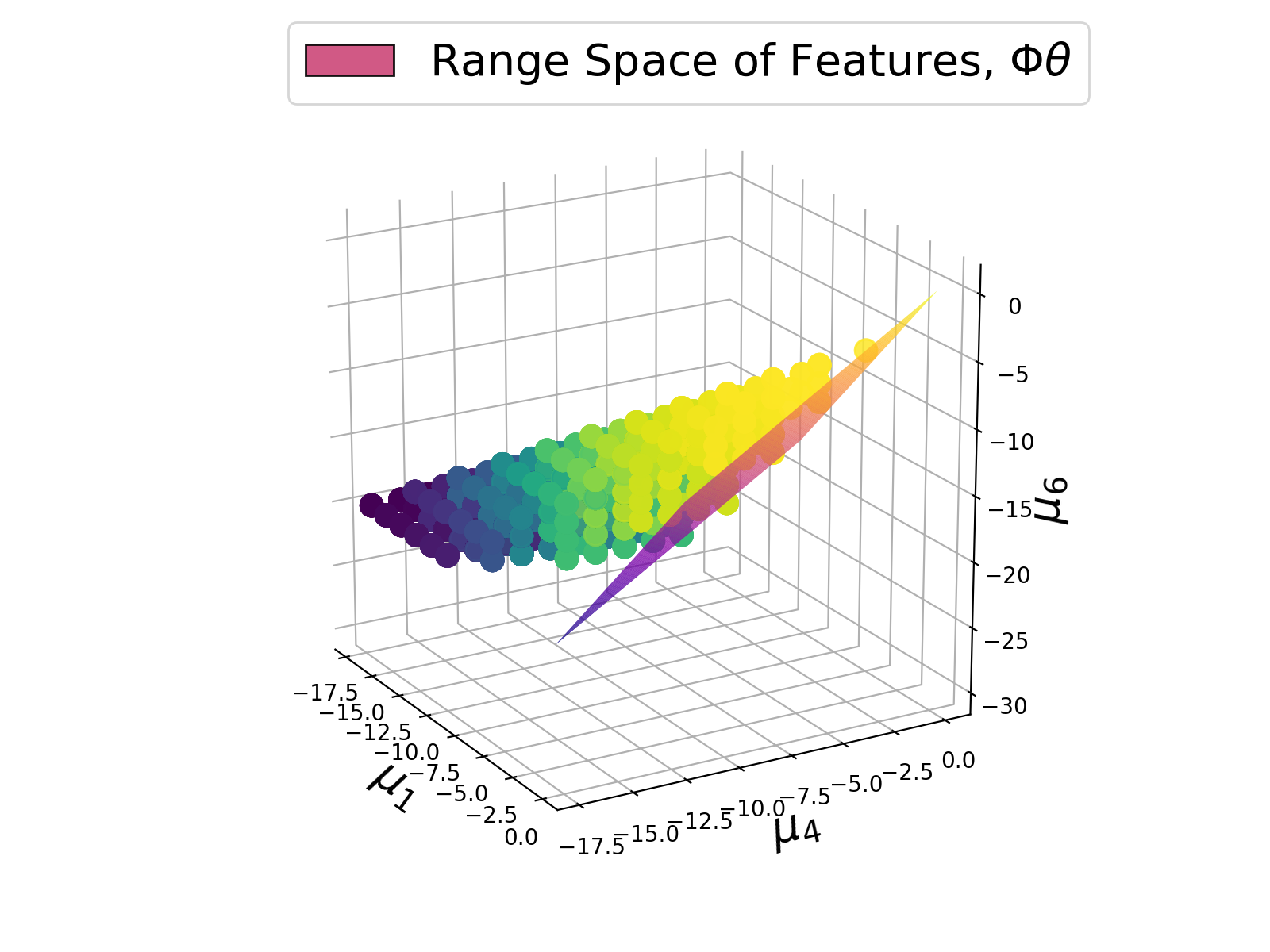}\hspace{-0.90cm}
    \includegraphics[width=1.5in]{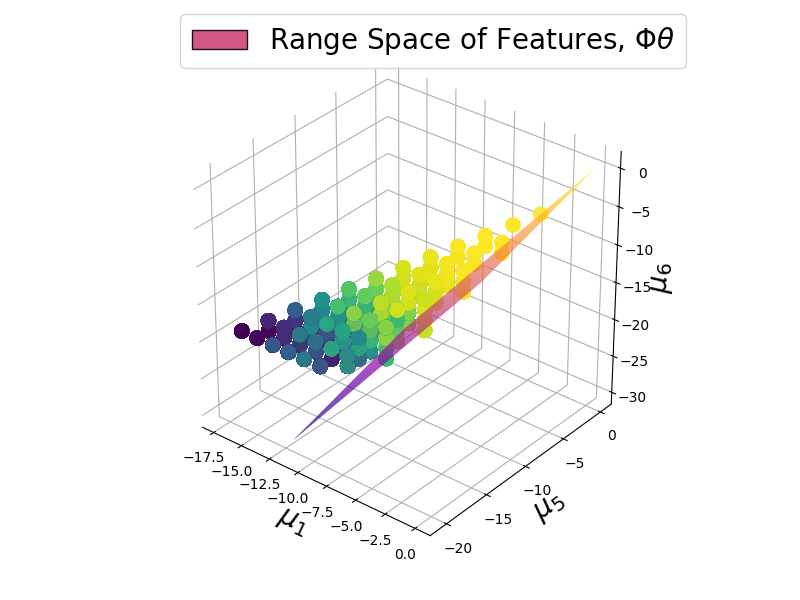}\hspace{-0.90cm}
    \includegraphics[width=1.5in]{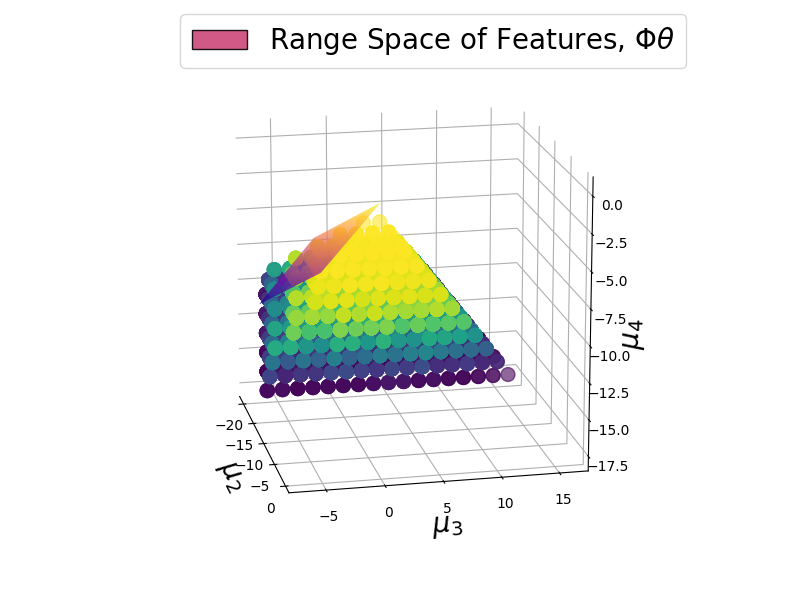}\hspace{-0.90cm}
    \includegraphics[width=1.5in]{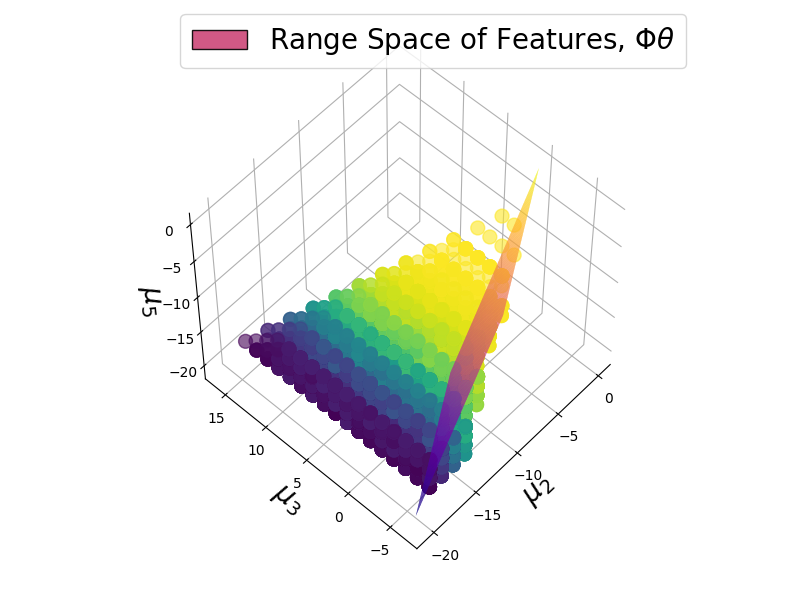}
    \\[\smallskipamount]
    \hspace{-1.6cm}
    \includegraphics[width=1.5in]{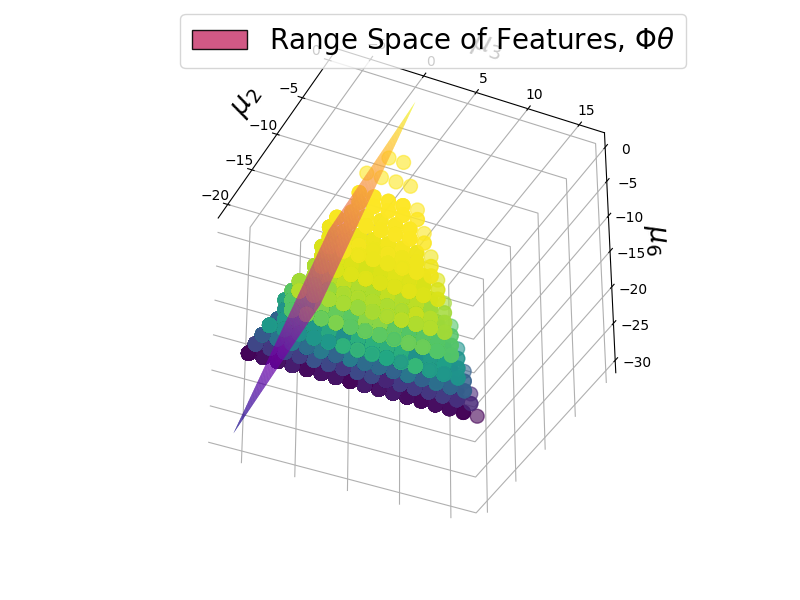}\hspace{-0.90cm}
    \includegraphics[width=1.5in]{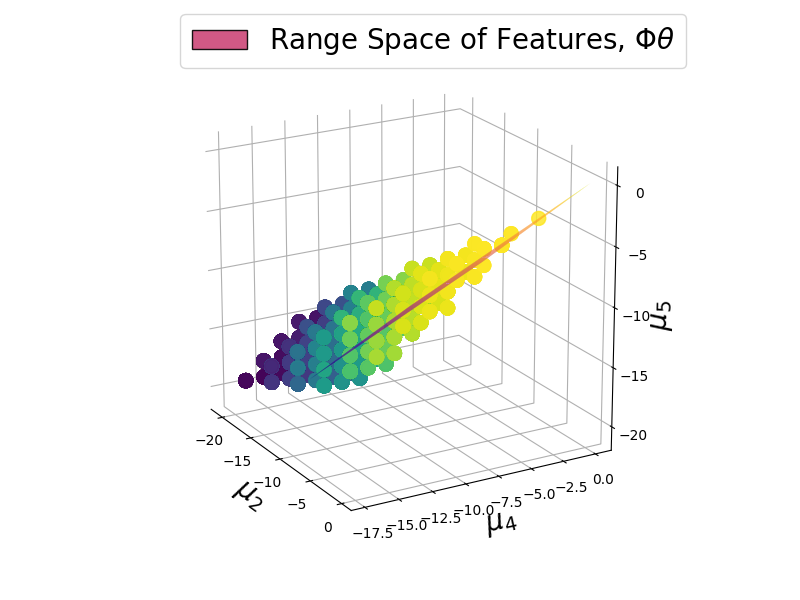}\hspace{-0.90cm}
    \includegraphics[width=1.5in]{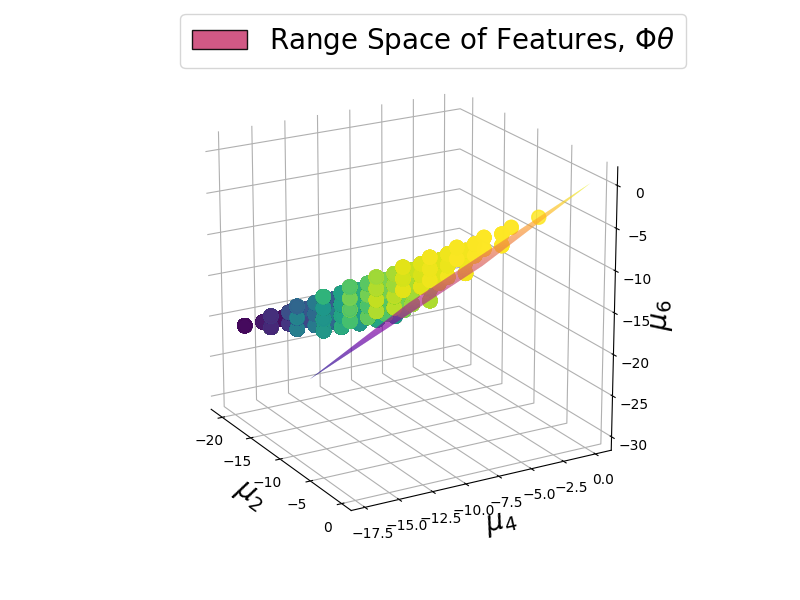}\hspace{-0.90cm}
    \includegraphics[width=1.5in]{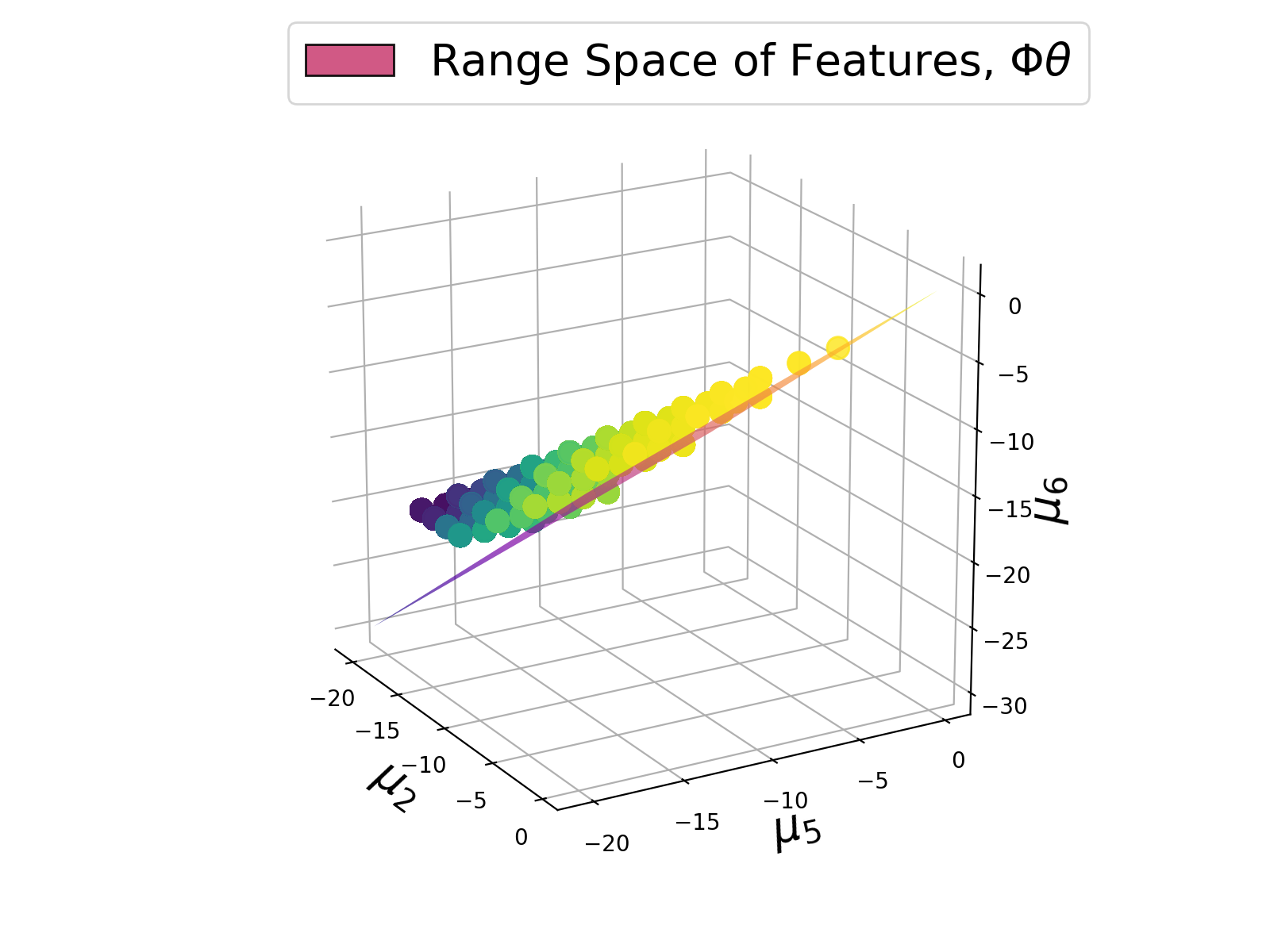}
    \\[\smallskipamount]
    \hspace{-1.6cm}
    \includegraphics[width=1.5in]{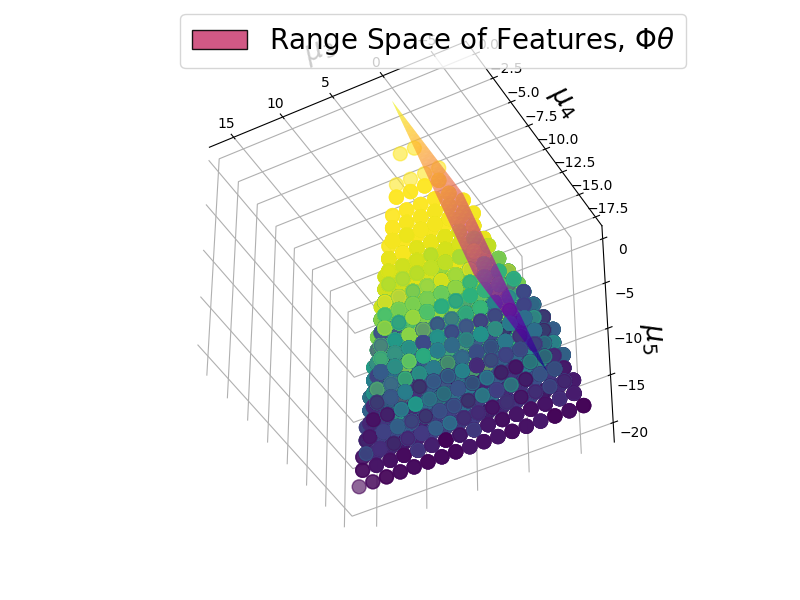}\hspace{-0.90cm}
    \includegraphics[width=1.5in]{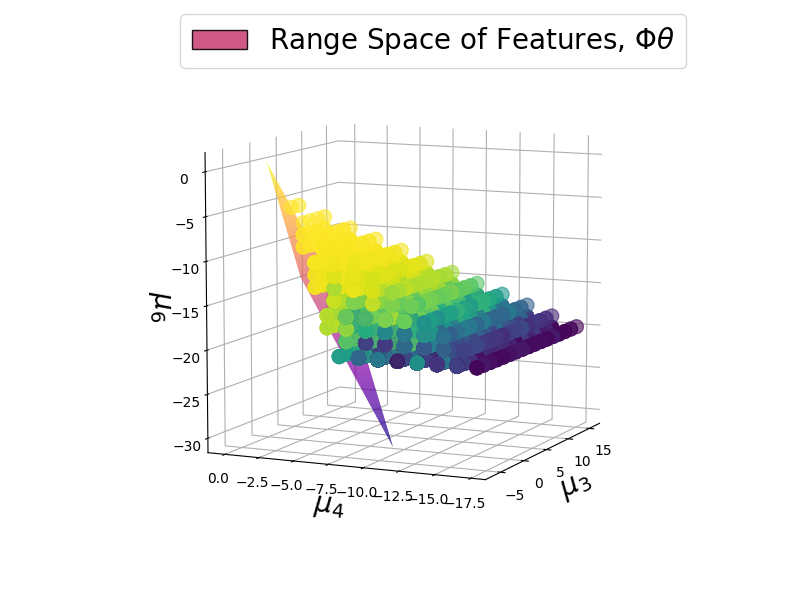}\hspace{-0.90cm}
    \includegraphics[width=1.5in]{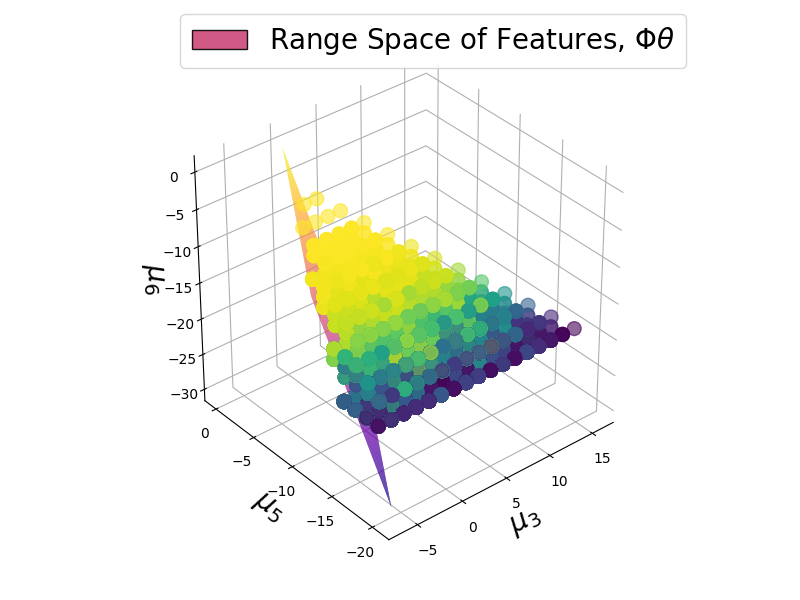}\hspace{-0.90cm}
    \includegraphics[width=1.5in]{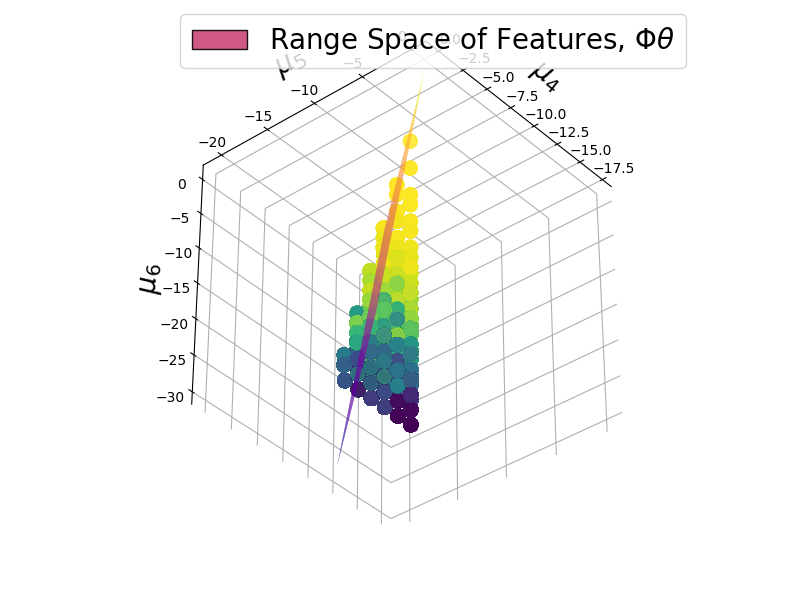}
    \caption{Visualization of the \emph{robust observation region} $\cC^{x_1}_3 \cap \cC^{x_2}_1$ calculated for the contextual bandit problem, described in Section \ref{subsec:exp_context}. The robust region is a subset in $\Real^6$. We are plotting the projections of the high-dimensional region along all possible three-dimensional bases. Each figure represents a $3$-dimensional projection of the robust region along some orientation of the basis vectors. The projected regions are shown as a shaded color gradient to give the perception of depth in $\Real^3$. We also plot the $3$-dimensional projections of the range space of the feature matrix $\bm{\Phi}\theta$. The range space is shaded in a 'plasma' color as indicated in the inset legend and is shown to span $\Real^2$.}
    \label{fig:context_observation_eg_2}
\end{figure}

\begin{figure}

    \begin{subfigure}[$\mathbf{p}_{x_1} = 0.1$ and $\mathbf{p}_{x_2} = 0.9$ \label{fig_a:context_regret_eg2_01}]{\includegraphics[width = 0.50\linewidth]{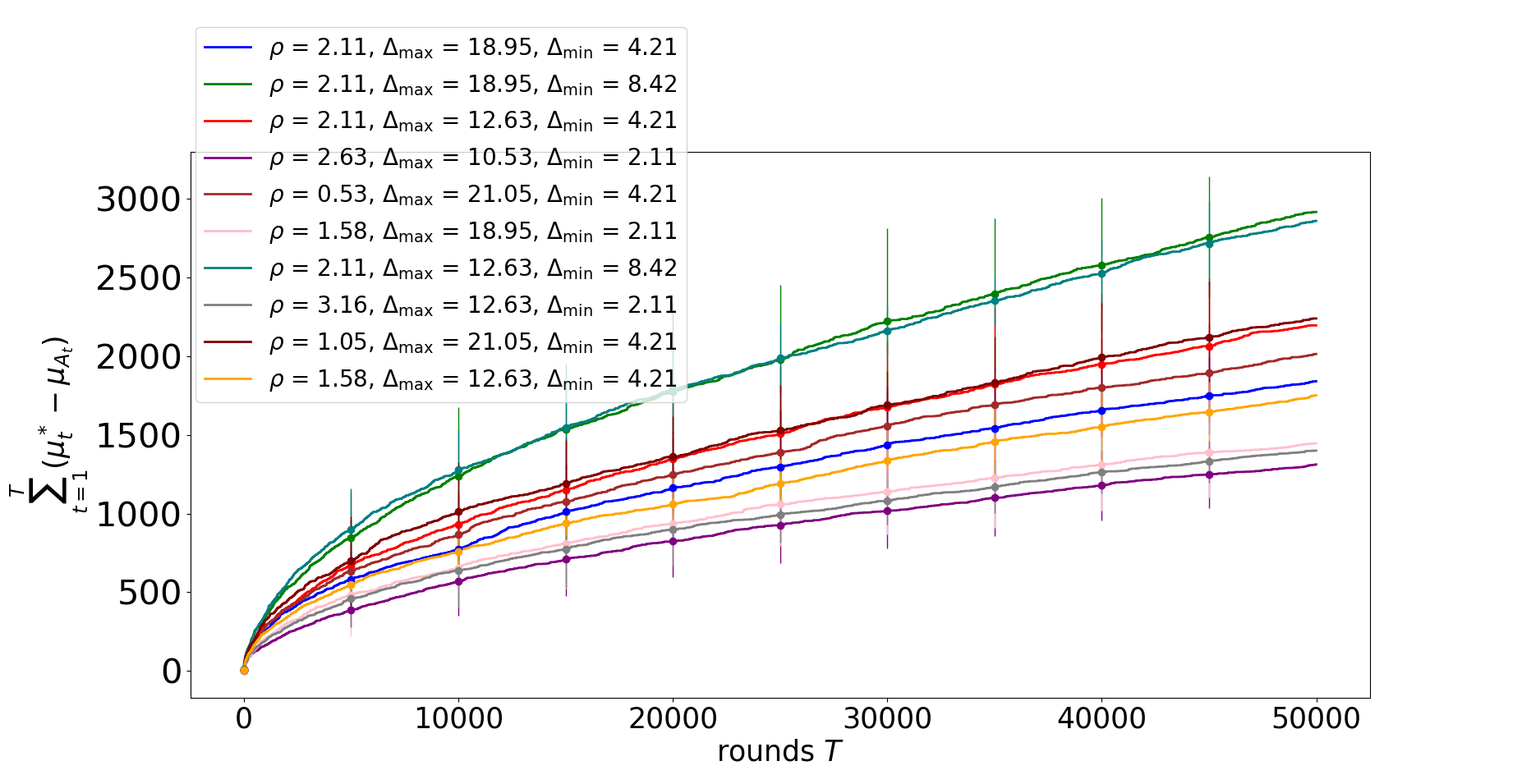}}
    \end{subfigure}
    \begin{subfigure}[$\mathbf{p}_{x_1} = 0.5$ and $\mathbf{p}_{x_2} = 0.5$ \label{fig_b:context_regret_eg2_05}]{\includegraphics[width = 0.50\linewidth]{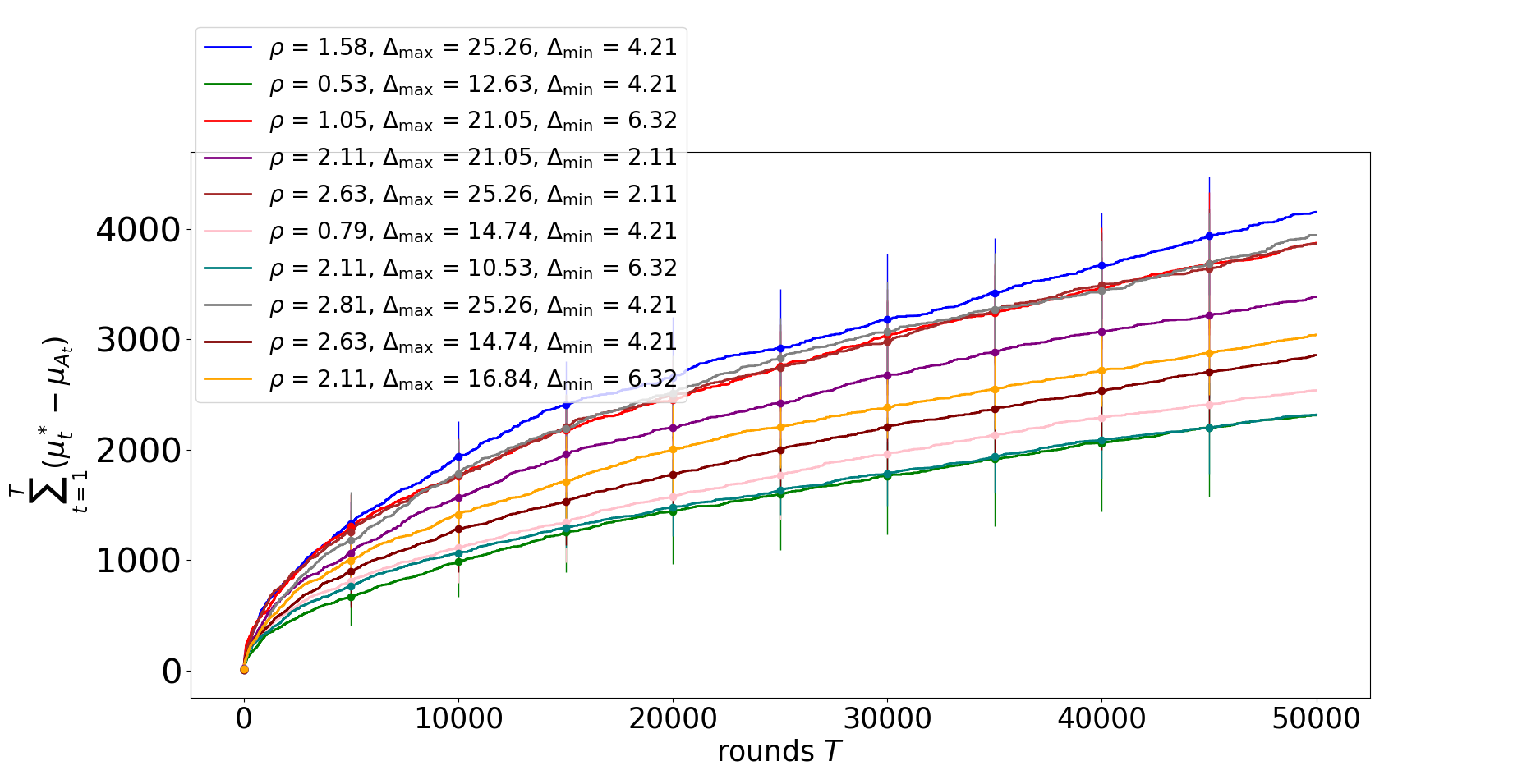}}
    \end{subfigure} 
    
    \centering
    \begin{subfigure}[$\mathbf{p}_{x_1} = 0.9$ and $\mathbf{p}_{x_2} = 0.1$ \label{fig_c:context_regret_eg2_09}]{\includegraphics[width = 0.8\linewidth]{figures/context_regret_05_eg2.png}}
    \end{subfigure}
    
\caption{The growth of the cumulative regret for $10$ misspecified bandit instances sampled from the robust region of $\cC^{x_1}_3 \cap \cC^{x_2}_1$ under the $\epsilon$-greedy algorithm with $\epsilon_t = 1/\sqrt{t}$ for different context distributions $\mathbf{p}_{x_1}$ and $\mathbf{p}_{x_2}$. The plots represent the average of $10$ trials. The $Y$-axis denotes the cumulative regret $\sum_{t=1}^T\mu_t^* - \mu_{A_t}$. The $X$-axis denotes the rounds $T$. We observe the sub-linear growth trend of the cumulative regret. For each instance the values of the $l_\infty$ misspecification error ($\rho$), the maximum sub-optimality gap ($\Delta_{\max}$) and the minimum sub-optimality gap ($\Delta_{\min}$) are also noted. It is observed that instances with higher $\Delta_{\max}$ suffer more regret at any time than instances with lower $\Delta_{\max}$ as expected from our theorem.}
    \label{fig:context_regret_02}    
\end{figure}

\section{Non-Linear Function Approximation}
\label{sec:non_lin_app}

\begin{algorithm}[htbp]
\caption{Generic $\epsilon$-greedy algorithm}
\label{alg:bandit_eps_greedy_nonlinear}
\textbf{Input}: A parameterized class of functions $\cF$ 
\begin{algorithmic}[1] 
\FOR{t = 1 to T}
\STATE With an estimate $\hat{\theta}_t$, play arm $a$ such that 
\begin{align*}
    A_t &= \argmax_{a \in \cA} f_{\hat{\theta}_t}(a) \;\mathrm{ w.p. }\;1-\epsilon_t \\
    &= \textit{play uniformly over } \cA \textit{ arms}  \;\mathrm{ w.p. }\;\epsilon_t
\end{align*} 
\STATE Observe the reward $Y_t$.
\STATE Update the estimate as 
\begin{align*}
    \hat{\theta}_{t+1} = \argmin_{\theta} \sum_{s=1}^t[f_{\theta}(A_s)-Y_s]^2
\end{align*}
\ENDFOR
\end{algorithmic}
\end{algorithm}

In this section, we show that the definition discussed in Section \ref{subsec:Bandits} for bandits also readily extends to general function classes, as does the regret results. Characterizing the robust regions is not simple and would depend on the function class. With the definition of the greedy regions as before, 

\begin{definition}[\textbf{Greedy Region $\mathcal{R}$}]
Define by $\cR_a$, for any $a \in \cA$, as the region in $\Real^{\abs{\cA}}$ for which the $a^{th}$ arm is the optimal, that is 
\begin{align}
    \cR_a \triangleq \Big\{ \bm{\mu} \in \Real^{\abs{\cA}} : \mu_a > \mu_i \forall i \neq a \Big\}.
\end{align}
\end{definition}
\begin{definition}[\textbf{Parameterized Function Class}]
we consider the following real-valued function class parameterized by $\theta$, which takes any arm $a \in \cA$ and for a parameter $\theta \in \Real^d$ gives a real value,
\begin{align*}
    \cF = \Big\{f_\theta : \cA \to \Real \;\forall\; \theta \in \Real^d\Big\}.
\end{align*}    
\end{definition}
\begin{remark}
    The only requirement for the elements of the function class $\{f_\theta\}$ is that we should be able to perform regression analysis with this class.
\end{remark}

For such a function class, the least squares estimate is calculated as the following
\begin{align*}
    \hat{\theta}_{t+1} = \argmin_\theta \sum_{s=1}^t\Big(f_\theta(A_s) - y_{A_s}\Big)^2,
\end{align*}
which in terms of the sampling distribution $\{\alpha(a)\}_{a\in\cA}$ can be rewritten as
\begin{align}
\label{eqn:bandit_LSE_non}
    \hat{\theta}_{t+1} = \argmin_\theta\sum_{a \in \cA}\alpha(a)\Big(f_\theta(a) - \hat{\mu}_a\Big)^2,
\end{align}
where $\hat{\mu}_a$ is the empirical average reward from arm $a$. Thus we can define, in an analogous way, (see Definition \ref{def: projection_bandit}), the model estimate based on a sampling distribution $\Lambda = \{\alpha_a\}_{a \in \cA}$,
\begin{definition}[\textbf{Model Estimate under sampling distribution}]
\label{def: projection_bandit_non}
    For any bandit instance $\bm{\mu}$ in $\Real^{\abs{\cA}}$ and a parameterized function class $\cF$, we shall denote the model estimate of $\bm{\mu}$ for a sampling distribution denoted by $\Lambda = \{\alpha_a\}_{a \in \cA}$, where $\{\alpha_a\}_{a \in \cA}$ lies in the $\abs{\cA}$ dimensional simplex $\Delta_{\mathrm{A}}$ as, 
    \begin{align*}
        \mathbf{P}^{\Lambda}_{\cF} (\bm{\mu}) \triangleq \argmin_\theta\sum_{a \in \cA}\alpha(a)\Big(f_\theta(a) - \mu_a\Big)^2\;.
    \end{align*} 
\end{definition}
We can also define, in an analogous way as in Definition \ref{def:robust_parameter_bandit}, the \emph{robust parameter region} as,
\begin{definition}[\textbf{Robust Parameter Region}]
\label{def:robust_parameter_bandit_non}
For a given parameterized function class $\cF$ parameterized by $\theta$, we define the $a^{th}$ Robust Parameter Region $\Theta_a$ for any arm $a \in \cA$ as the set of all parameters $\theta \in \Real^d$ such that the range space of $\cF$ restricted to $\Theta_a$ lies in the greedy region $\mathcal{R}_a$. That is,
\begin{align*}
    \Theta^{\cF}_a = \Big\{\theta \in \Real^d : f_\theta(a) > f_\theta(i)\;\forall\;i\neq a\Big\}.
\end{align*}
\end{definition}
And, analogous to Definition \ref{def:robust_observation_bandit}, define the \emph{robust observation region} as,
\begin{definition}[\textbf{Robust Observation Region}]
\label{def:robust_observation_bandit_non}
For a given parameterized function class $\cF$, we define the $a^{th}$ Robust Observation Region $\cC_a$ for any arm $a \in \cA$ as the set of all observations $\bm{\mu}$ with optimal arm $a$, such that under any sampling distribution $ \Lambda = \{\alpha(a)\}_{a \in \cA}$, the corresponding model estimate, $\mathbf{P}^{\Lambda}_{\cF} (\bm{\mu})$ lies in the $a^{th}$ robust parameter region $\Theta^{\cF}_a$. That is,
\begin{align*}
    \cC^{\cF}_a = \Big\{\bm{\mu} \in \cR_a : \mathbf{P}^{\Lambda}_{\cF} (\bm{\mu}) \in \Theta^{\cF}_a \;\forall\; \Lambda = \{\alpha(a)\}_{a \in \cA} \in \Delta_{\abs{\cA}} \Big\}.
\end{align*}
\end{definition}

With these definitions, we can again get the following results for $\epsilon$-greedy algorithm,

\begin{theorem}[\textbf{Sufficient Condition for Zero Regret in Misspecified Bandits}]
\label{thm:bandit_eps_greedy_nonlin}
For a parameterized function class $\cF$, any sub-Gaussian bandit parameter $\bm{\mu}$ which is an interior point of the robust observation region, $\mathrm{Int}(\cC^{\cF}_{\mathrm{OPT(\bm{\mu})}})$, the $\epsilon$-greedy algorithm, as described in Algorithm \ref{alg:bandit_eps_greedy_nonlinear}, with $\epsilon_t$ set as $\frac{1}{\sqrt{t}}$, achieves a regret of $O(\Delta_{\max}\sqrt{T})$, where $\Delta_{\max}$ is the largest sub-optimal gap, that is $\Delta_{\max} = \max_{a \in \cA} \mu_{\mathrm{OPT}(\bm{\mu})} - \mu_a$.
\end{theorem}

The theorem follows from the analysis of the $\epsilon$-greedy algorithm and the definitions of robustness.
\section{Robust Features in Bandits}
\label{sec:robust_bandit}
\begin{figure}[htbp]
    \centering
    \begin{subfigure}[\footnotesize{A $2$ dimensional function approximation class whose robust region is quite large. The function approximation class is denoted by the blue region denoting the union of the disjoint planes, given by, $\big\{(x,y,z) \text{ s.t }  x = y + \epsilon,\;\; x = z + \epsilon \big\} \cup \big\{(x,y,z) \text{ s.t } y = z + \epsilon,\;\; y = x + \epsilon \big\} \cup \big\{(x,y,z) \text{ s.t } z = x + \epsilon,\;\; z = y + \epsilon \big\}$. The red regions denote the non-robust regions. The green regions denote the robust region. This shows how we can design a function approximation class that can tolerate even high levels of misspecification.}\label{fig_a:rboust_features_1}]{\includegraphics[width = 0.8\linewidth]{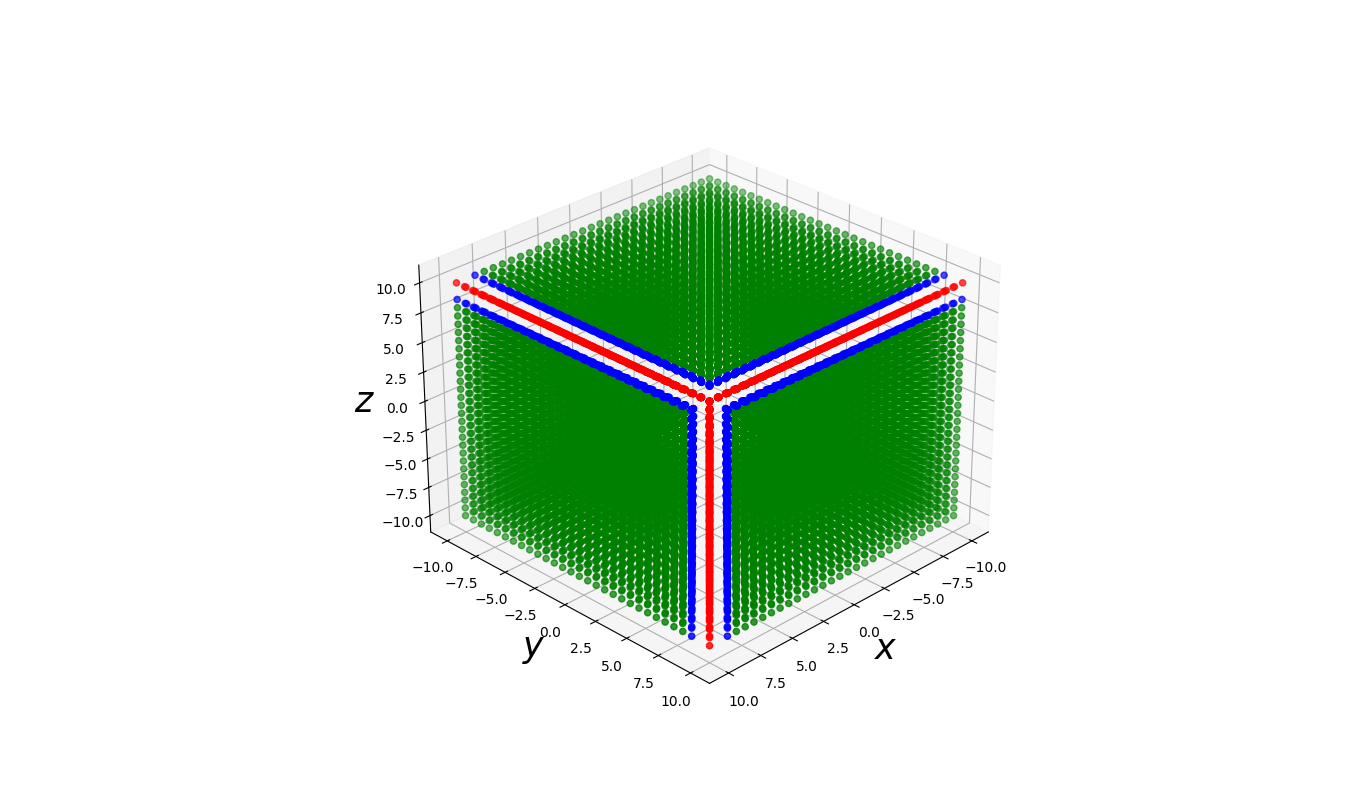}}
    \end{subfigure}
    \begin{subfigure}[\footnotesize{The same figure as Figure \ref{fig_a:rboust_features_1}, but from a different viewing angle, highlighting the planes of the non-robust regions}]{\includegraphics[width = 0.8\linewidth]{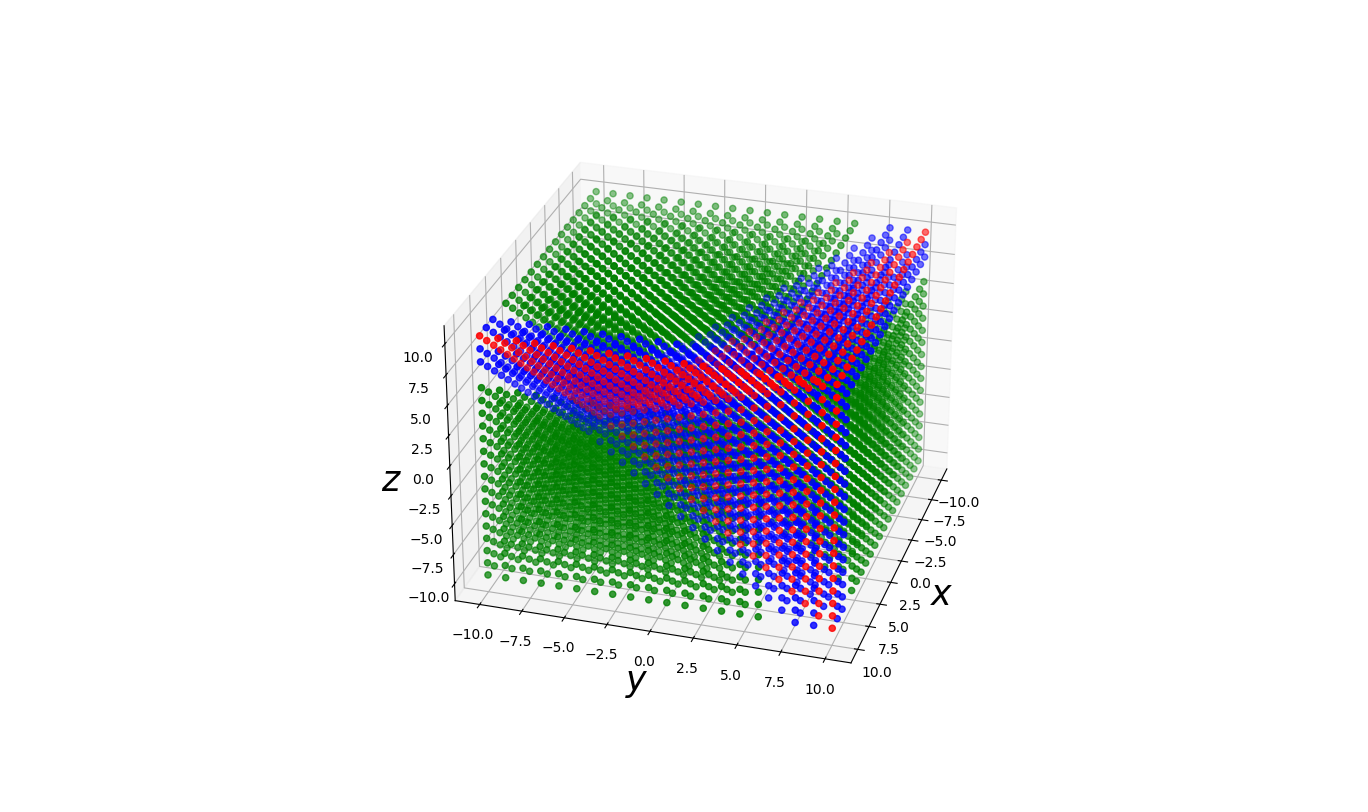}}
    \end{subfigure}
    \caption{The $2$-dimensional robust function class is represented by the blue region. The green regions denote the robust regions. The red regions denote the non-robust regions}
    \label{fig:robust_features}
\end{figure}

In this section, we introduce a feature class that has a provably large robust region in the Bandit setting, in the sense that the set of non-robust instances forms a measure zero set. The motivation of this feature class arises from Figure \ref{fig_b : features}, and in essence, we give a higher dimensional variant of the function approximation introduced in the Introduction.

Let us consider a $K$-armed bandit instance with a unique optimal arm $\bm{\mu} = \begin{bmatrix}
    &\mu_1 &\\
    &\mu_2 &\\
    &\vdots& \\
    &\mu_k
\end{bmatrix}$ as an element in $\Real^\mathrm{K}$ space. Note that the greedy regions $\{\cR_i\}_{i=1}^K$ partition the $\Real^\mathrm{K}$ space into $K$-disjoint partitions. Consider the $K$ connected $K$-dimensional manifolds 
\begin{align*}
 \cM_i = \big\{\bm{\mu} \in \Real^{\mathrm{K}} \text{ s.t } \mu_i > \mu_j + \epsilon \;\;\; \forall \;\; j\neq i \big\}\;,   
\end{align*}
for a fixed $\epsilon>0$ and for all $i \in [\mathrm{K}]$. Note that each $\cM_i \subset \cR_i$. We define the function approximation class as
\begin{align*}
    \cF = \cup_{i=1}^{\mathrm{K}}\partial \cM_i\;\;,
\end{align*}
where $\partial \cM_i$ is the boundary of the $K$-dimensional manifold $\cM_i$ and hence is a $K-1$ dimensional manifold. Thus $\cF$ is a $K-1$ dimensional manifold.

\paragraph{Illustration in $3$-dimesions} For $\Real^3$ we illustrate the function approximation class, $\cF$ in Figure \ref{fig:robust_features}. The blue points denote the boundary $\big\{(x,y,z) \text{ s.t }  x = y + \epsilon,\;\; x = z + \epsilon \big\} \cup \big\{(x,y,z) \text{ s.t } y = z + \epsilon,\;\; y = x + \epsilon \big\} \cup \big\{(x,y,z) \text{ s.t } z = x + \epsilon,\;\; z = y + \epsilon \big\}$ and represents a disjoint two dimensional manifold. The green regions mark the robust regions which are defined as $\big\{(x,y,z) \text{ s.t }  x > y + \epsilon,\;\; x > z + \epsilon \big\} \cup \big\{(x,y,z) \text{ s.t } y > z + \epsilon,\;\; y > x + \epsilon \big\} \cup \big\{(x,y,z) \text{ s.t } z > x + \epsilon,\;\; z > y + \epsilon \big\}$. The red regions mark the non-robust regions which are the set complement of the robust regions. Thus it can be seen that as one chooses lower and lower values of $\epsilon$, the non-robust region goes down, in the sense that more and more instances are robust. 

\begin{theorem}
    For the function approximation class defined above as $\cF$, the region $\cup_{i=1}^{\mathrm{K}}\cM_i$ is robust.
\end{theorem}

\begin{proof}
    Let us take an arbitrary $\bm{\mu} \in \cup_{i=1}^{\mathrm{K}}\cM_i$ and without loss of generality assume $\cM_k$ for some $k \in [\mathrm{K}]$. This is valid since the $\cM_i$-s form a disjoint set. Note that for any sampling distribution, the projection of the chosen $\bm{\mu}$ on $\cF$ belongs to $\partial \cM_k$ since $\partial \cM_k$ is the boundary of the set $\cM_k$. Since $\partial \cM_k$ is a subset of the greedy region $\cR_k$, we see that the robust condition is satisfied. This completes the proof.
\end{proof}

Note that as $\epsilon$ decreases the robust region increases and hence we can have an arbitrarily large class of robust bandit instances.

\section{Robust Features for Contextual Bandits}
\label{sec:robust_context}
The previous construction can be extended to include the contextual bandit case as well. To keep the notation from being overwhelming we study the case of all two-context two-armed instances and the ideas can be extended to general arbitrary finite-context finite-armed instances.  

Note that any two-context two-armed contextual bandit instance lies in $\Real^4$ space and thus can be represented as an element $(w,\;x,\;y,\;z)$ in $\Real^4$. Without loss of generality let any contextual bandit instance be represented by $[\mu_{(s_1, a_1)},\;\; \mu_{(s_1, a_2)},\;\; \mu_{(s_2, a_1)},\;\; \mu_{(s_2, a_2)}]^\top$, that is the coordinates of $\Real^4$ represent the reward of first context first action, first context second action etc.
Consider the greedy regions which partition $\Real^4$, as

\begin{align*}
&\cR^1_1 = \big\{(w,x,y,z) \text{ s.t } w > x,\;\; y > z\Big\} \\
&\cR^1_2 = \big\{(w,x,y,z) \text{ s.t } w > x, \;\; z > y \Big\} \\
&\cR^2_1 = \big\{(w,x,y,z) \text{ s.t } x > w, \;\; y > z \Big\} \\
&\cR^2_2 = \big\{(w,x,y,z) \text{ s.t } x > w,\;\; z > y \Big\}\;\;. 
\end{align*}

Let 
\begin{align*}
&\cM^1_1 = \big\{(w,x,y,z) \text{ s.t } w > x+ \epsilon,\;\; y > z + \epsilon \Big\} \\
&\cM^1_2 = \big\{(w,x,y,z) \text{ s.t } w > x+ \epsilon,\;\; z > y + \epsilon \Big\} \\
&\cM^2_1 = \big\{(w,x,y,z) \text{ s.t } x > w+ \epsilon,\;\; y > z + \epsilon \Big\} \\
&\cM^2_2 = \big\{(w,x,y,z) \text{ s.t } x > w+ \epsilon,\;\; z > y + \epsilon \Big\} 
\end{align*}
be the the disjointed $4$ dimensional manifolds such that each $\cM_{i}^j \subset \cR_i^j$. Note that any contextual bandit $\bm{\mu} \in \Real^4$ must lie in one of the greedy regions $\cR_i^j$. Like before we define the function class as
\begin{align*}
    \cF &= \cup_{i,j} \partial \cM_i^j \\
    &= \{ w = x+ \epsilon,\;\; y = z+\epsilon\} \cup
\{ w = x+ \epsilon,\;\; z = y + \epsilon\} \cup \{ x = w + \epsilon,\;\; y = z+\epsilon\} \cup \{ x = w+ \epsilon,\;\; z = y + \epsilon\}\;.
\end{align*}

As before, we have the following observation

\begin{theorem}
Any contextual bandit instance lying in the region $\cup_{i,j} \cM_i^j$ is robust in function approximation class defined above as $\cF$.
\end{theorem}

\section{Robust Features in MDPs}
\label{sec:robust_mdp}
In this section we return to the example presented in Section \ref{sec:example_mdp} of a deterministic MDP under uniform sampling. In this section we shall analyze a stochastic variant of the MDP introduced in the Section \ref{sec:example_mdp} under an arbitrary behavioral policy. We shall illustrate how with the help of the robust features introduced in Section \ref{sec:robust_context} we can arrive at a condition similar to the one we arrived at in Section \ref{sec:example_mdp}.

We shall be looking at the same two stage MDP as in Figure \ref{fig : examples} with three states. At stage $h = 1$, we are at state $s_1$. We can take two actions depending on a behavioral policy at state $s_1$, and move to either state $s_2$ or state $s_3$ depending on the transition probability and get an associated reward of $r_{11}$ or $r_{12}$. At stage $h=2$ we are thus either at state $s_2$ or state $s_3$. We can again choose either action $a_1$ or action $a_2$ depending on a behavioral policy and get an associated reward of $r_{21}$ or $r_{22}$ or $r_{31}$ or $r_{32}$ respectively. We keep the same constraints as introduced in Section \ref{sec:example_mdp}, namely $\pi^*(s_1) = a_2$, $\pi^*(s_2) = a_1$ and $\pi^*(s_3) = a_1$.

We shall denote by $\alpha^{\pi_b}_2(s_2, a_1)$, $\alpha^{\pi_b}_2(s_2, a_2)$, $\alpha^{\pi_b}_2(s_3, a_1)$ and $\alpha^{\pi_b}_2(s_3, a_2)$ the true states $s_i$ and action $a_j$ distribution observed under the transition dynamics of the MDP and the behavioral policy $\pi_b$ at stage $h=2$. We will choose the function class $\cF_2$ as the one introduced in Section $\ref{sec:robust_context}$. Thus we need to solve the following weighted least squares problem

\begin{align*}
    &\argmin_{(\mu_{21},\;\;\mu_{22},\;\;\mu_{31},\;\;\mu_{32} )} \sum_{ i \in \{2,3\},\;\; j \in \{1,2\}} \alpha^{\pi_b}_2(s_i,a_j) (\mu_{ij} - r_{ij})^2 \\
    &\text{such that } (\mu_{21},\;\;\mu_{22},\;\;\mu_{31},\;\;\mu_{32} ) \in \cF_2
\end{align*}

Note that we have from the MDP constraint $r_{21}>r_{22}$ and $r_{31}> r_{32}$. Thus, for a small enough $\epsilon$ this falls in the the region $\cM^1_1$ as per our notation in Section \ref{sec:robust_context}, and hence we have
\begin{align*}
    & \mu_{21} = \frac{\alpha_{21} r_{21} + \alpha_{22} r_{22} + \alpha_{22} \epsilon}{\alpha_{21} + \alpha_{22}} \\
    & \mu_{22} = \frac{\alpha_{21} r_{21} + \alpha_{22} r_{22} - \alpha_{21} \epsilon}{\alpha_{21} + \alpha_{22}} \\
    & \mu_{31} = \frac{\alpha_{31} r_{31} + \alpha_{32} r_{32} + \alpha_{32} \epsilon}{\alpha_{31} + \alpha_{32}} \\
    & \mu_{32} = \frac{\alpha_{31} r_{31} + \alpha_{32} r_{32} - \alpha_{31} \epsilon}{\alpha_{31} + \alpha_{32}} 
\end{align*}
for some small but fixed $\epsilon>0$ and we use the short hand notation of $\alpha_{ij}$ for $\alpha^{\pi_b}_2(s_i, a_j)$.

It can be observed that $\mu_{21} > \mu_{22}$ and $\mu_{31}>\mu_{32}$ and thus we have the optimal policy of $s_2$ and $s_3$. 

At stage $h=1$, we denote by $\alpha_1^{\pi_b}(s_1, a_1, s_2)$, $\alpha_1^{\pi_b}(s_1, a_1, s_3)$, $\alpha_1^{\pi_b}(s_1, a_2, s_2)$ and $\alpha_1^{\pi_b}(s_1, a_2, s_3)$ the true transition distribution of the form $(s_i, a_i, s_j)$ under the policy $\pi_b$ and true dynamics of the MDP. With the values of $\mu_{21}$ and $\mu_{31}$ that we calculated for stage $h=2$ we solve the following weighted least squares problem, 

\begin{align*}
    &\argmin_{(\mu_{11},\;\;\mu_{12})} (\alpha_{12} + \alpha_{13})\Big( \mu_{11} - r_{11} - \frac{\alpha_{12}\mu_{21} + \alpha_{13}\mu_{31}}{\alpha_{12} + \alpha_{13}}\Big)^2 + (\alpha_{22} + \alpha_{23})\Big( \mu_{12} - r_{12} - \frac{\alpha_{22}\mu_{21} + \alpha_{23}\mu_{31}}{\alpha_{22} + \alpha_{23}}\Big)^2\\
    & \text{ such that } (\mu_{11},\;\;\mu_{12}) \in \cF_1\;,
\end{align*}
where $\cF_1$ is the function class introduced in Section \ref{sec:example_mdp} in Figure \ref{fig : features}. Note that for a stochastic MDP we now have to take the expected bootstrapped $Q$ values.

Thus in order for the optimal policy at state $s_1$ to be action $a_2$, we need the following condition to hold,
\begin{align*}
    r_{11} + \frac{\alpha_{12}\mu_{21} + \alpha_{13}\mu_{31}}{\alpha_{12} + \alpha_{13}} < r_{12} + \frac{\alpha_{22}\mu_{21} + \alpha_{23}\mu_{31}}{\alpha_{22} + \alpha_{23}}\;.
\end{align*}

Thus, we now have a condition which depends on the true probabilities of observing the transitions, and many MDPs can be found satisfying it. 
\section{Technical Lemmas}

\begin{lemma}[Self Normalized Bound for Vector Valued Martingales \cite{oful}]
 \label{lemma:self_normalize}
 Let $\{\cF_t\}_{t=0}^\infty$ be a filtration. Let $\{\eta_t\}_{t=1}^\infty$ be a real valued stochastic process such that $\eta_t$ is $\mathcal{F}_t$- measurable and $\eta_t$ is conditionally $R$-sub-Gaussian for some $R > 0$, i.e.
 \begin{align*}
     \forall \lambda \in \Real\;\;\;\; \mathbf{E}[e^{\lambda\eta_t}\mid\mathcal{F}_{t-1}] \leq \exp\Big({\frac{\lambda^2R^2}{2}}\Big).
 \end{align*}
 Let $\{\phi_t\}_{t=1}^\infty$ be a $\Real^d$-valued stochastic process such that $\phi_t$ is $\mathcal{F}_{t-1}$-measurable. Assume $V$ is a $d \times d$ positive definite matrix and for any $t\geq 0$ define
 \begin{align*}
     V_t = V + \sum_{s=1}^t\phi_s\phi_s^\top \;\;\;\;\;\; S_t = \sum_{s=1}^t\eta_s\phi_s.
 \end{align*}
 Then for any $\delta>0$, with probability at least $1-\delta$, for all $t\geq0$,
 \begin{align*}
     \|S_t\|^2_{{V_t}^{-1}} \leq 2R^2 \log \Big(\frac{\mathrm{det}V_t^{1/2} \mathrm{det}V^{-1/2}}{\delta}\Big).
 \end{align*}
\end{lemma}

\begin{lemma}[Determinant Trace Inequality \cite{oful}]
    \label{lemma:det_trace}
    Suppose $\{\phi_s\}_{s=1}^t \subset \Real^d$ be such that $\|\phi_s\|_2 \leq L\; \forall\; s \in [t]$. Let $V_t = \lambda I + \sum_{s=1}^t\phi_s\phi_s^\top$ for some $\lambda>0$. Then
    \begin{align*}
        \mathrm{det}V_t \leq (\lambda + tL^2/d)^d
    \end{align*}
\end{lemma}

\begin{lemma}[\cite{oful}]
\label{lemma:matrix_self_normalize}
Let $\{\phi_t\}_{t=1}^\infty$ be a sequence in $\Real^d$, $V$ a $d \times d$ positive definite matrix and $V_t = V+\sum_{s=1}^t \phi_s \phi_s^\top$. Then
\begin{align*}
\log\Big(\frac{\mathrm{det}V_t}{\mathrm{det}V}\Big) \leq \sum_{s=1}^t\|\phi_s\|^2_{{V_s}^{-1}}.
\end{align*}
Moreover, if $\|\phi_s\|_2 \leq L$ for all $s$ and if $\lambda_{\min}(V) \geq \max\{1, L^2\}$, then
\begin{align*}
    \sum_{s=1}^t\|\phi_s\|^2_{{V_s}^{-1}} \leq 2 \log\Big(\frac{\mathrm{det}V_t}{\mathrm{det}V}\Big).
\end{align*}
\end{lemma}

\begin{lemma}[\cite{doi:10.1137/S0895479895284014}]
    \label{lemma:forsgren}
    Let $\Phi$ be a $K \times d$ full column matrix, let $\bm{\mu}$ be a $K$ dimensional matrix. Let $\bm{\Lambda}$ be the set of all positive semi-definite diagonal matrices such that $\Phi^\top \Lambda \Phi$ is invertible for any $\Lambda \in \bm{\Lambda}$, that is 
\begin{align*}
    \bm{\Lambda} = \{\Lambda \in \Real^{K \times K}\; : \; \Lambda \textit{ is diagonal and positive semi-definite } \land \Phi^\top \Lambda \Phi \textit{ is invertible}\}.
\end{align*}
Then the solution to the weighted least-squares problem lies in the convex hull of the basic solutions, that is,
\begin{align*}
    \Big(\Phi^\top\Lambda\Phi\Big)^{-1}\Phi^\top\Lambda\mu = \sum_{J \in \cJ(\Phi)} \Big(\frac{\mathrm{det}\Lambda_J\,\mathrm{det}\Phi_J^2}{\sum_{K \in \cJ(\Phi)}\mathrm{det}\Lambda_K \,\mathrm{det}\Phi_K^2} \Big) \Phi_J^{-1}\bm{\mu}_J,
\end{align*}
where $\cJ(\Phi)$ is the is the set of column indices associated with non-singular $d \times d$ sub-matrices of $\Phi$.
\end{lemma}

\begin{lemma}[sub-Gaussian Concentration]
\label{lemma:subG_conc}
    Assume $\{X_i - \mu\}_{i=1}^n$ are $n$ independent $\sigma$-sub-Gaussian random variables, then
    \begin{align*}
        \prob{\abs{\hat{\mu}_n - \mu} \geq \epsilon} \leq 2 \exp(- \frac{n\epsilon^2}{2 \sigma^2})\;,
    \end{align*}
    where $\hat{\mu}_n  = \frac{\sum_{i=1}^N X_i}{n}$.
\end{lemma}

\begin{lemma}[Bernstein inequality]
\label{lemma:Bernstein}
    Let $\{T_i\}_{i=1}^n$ be random variables in $[0,1]$, such that
    \begin{align*}
        \sum_{i=1}^n\mathbb{V}[T_i\mid T_{i-1}, T_{i-2}, \cdots, 1] = \sigma^2\;,
    \end{align*}
    then
    \begin{align*}
        \prob{\sum_{i=1}^n T_i \geq \mathbb{E}[\sum_{i=1}^n T_i] + \epsilon} \leq \exp\Big\{- \frac{\epsilon^2/2}{\sigma^2 + \epsilon/2}\Big\}.
    \end{align*}
\end{lemma}
\section{Related Work}
\label{sec:previous_app}
\vspace{-1mm}

\paragraph{Linear Bandits} The multi-armed bandit problem has been studied extensively, and a large body of literature surrounds it. The text by \citet{lattimore2020bandit} provides a comprehensive text on the study of bandit algorithms. The classical works on stochastic linear bandits with finitely many arms have been studied by \citet{auer2002using}. Algorithms based on the principle of \emph{optimism in the face of uncertainty} \cite{dani2008stochastic,oful} and Thompson Sampling \cite{pmlr-v23-agrawal12, pmlr-v54-abeille17a} are popular choices of algorithms which enjoy sub-linear regret even in the worst case of $\widetilde{O}(\sqrt{T})$. \emph{Phased elimination} with optimal design based algorithms attain a regret of $\widetilde{O}(\sqrt{dT})$ as shown in \citet{lattimore2020bandit}.

\paragraph{Contextual Linear Bandits} The bandit problem where the arm set changes at every time, described in the literature as the contextual bandit setup, has also been studied extensively. We mention the classic work of \citet{auer2002using, chu2011contextual}. The classic algorithms of SupLinRel \cite{auer2002using} and SupLinUCB \cite{chu2011contextual} are \emph{elimination} based algorithms and enjoy a worst-case regret bound of $\widetilde{O}(\sqrt{dT})$. In practice, optimism-based algorithms like LinUCB \cite{oful} as well as Thompson Sampling \cite{pmlr-v28-agrawal13} enjoy sub-linear regret of $\widetilde{O}(\sqrt{T})$.

\paragraph{Misspecified Bandits With Linear Regret} The model misspecification setting in linear bandit setting was first studied in \citet{ghosh2017misspecified}. They pointed that in the presence of model error, the worst-case regret of LinUCB \cite{oful} is of linear order. They further showed that in a favorable case, when one can test the linearity of the reward function, the RLB algorithm \cite{ghosh2017misspecified} can switch between the linear bandit algorithm and finite-armed bandit algorithm to address the misspecification issue and achieve a $O(\min\{\sqrt{K}, d\}\sqrt{n})$ regret. Under the definition of uniform $l_\infty$ model error $\epsilon$, \citet{pmlr-v119-lattimore20a} presented an algorithm based on the principle of \emph{phased elimination} which achieves a worst case regret of $O(d\sqrt{n} + \epsilon\sqrt{d}n)$. They also showed that contextual linear bandits \emph{LinUCB} achieve the same regret after modifying the confidence width using knowledge of the misspecification. In the same contextual bandit settings, \citet{pmlr-v119-foster20a} showed a similar linear regret. They showed that under the assumption of an oracle regressor, the algorithm SquareCB \cite{pmlr-v119-foster20a} suffers a regret of $O(\sqrt{dKn} + \epsilon\sqrt{d}n)$. \citet{zanette2020learning} has also shown a similar regret in the contextual setting. There has also been work based on altered definitions of misspecifications. \citet{pmlr-v130-kumar-krishnamurthy21a} defines the misspecification as an expected square loss between the true reward class and the approximation function class and uses a model-based algorithm $\epsilon$- FALCON, which again suffers linear regret. \citet{NEURIPS2020_84c230a5} defines an empirical misspecification as observed by the data. However, these still suffer linear regret.

\paragraph{Missecified Bandits With Sub-Linear Regrets} Recently, there has been some work which describes sub-linear regret in the misspecified setting. The result of \citet{pmlr-v216-liu23c} analyzes the LinUCB algorithm when the sub-optimality gaps of the arms bound the misspecification. They show that when the misspecification is of low order, the algorithm enjoys sub-linear regret. Under a similar condition, \citet{zhang2023interplay} were able to extend the study to the contextual setting. They propose a phased arm elimination  algorithm, which performs similarly to SupLinUCB \cite{chu2011contextual} but requires knowledge of the sub-optimality gap.

\paragraph{Markov Decision Processes} Function approximation in Reinforcement Learning has had a rich history. For reference we refer the reader to the manuscript of \cite{bertsekas1996neuro} as a comprehensive source. There has been recent interest in the finite-time analysis of function approximation in reinforcement learning. For example \cite{pmlr-v75-bhandari18a} shows the convergence of policy evaluation with linear function approximators. Control problems, as in problems that require both policy evaluation and improvement are notoriously hard to evaluate under function approximators, and well-known algorithms like $Q$-learning and SARSA are known to not converge with function approximations \citep{bertsekas1996neuro}. Theoretically, there have been efforts to mitigate this problem, for example in \cite{zou2019finite} analyses the sample complexity of SARSA with linear function approximators under a Lipschitz continuous policy improvement operator. On the other hand, there have been works on developing online algorithms founded on bandit literature that focus on the exploration-exploitation dynamics in an MDP. For example, \cite{van2014generalization} introduced RLSVI, a Thompson Sampling-based algorithm that has gone on to receive some attention in the recent past \citep{osband2016generalization, pmlr-v108-zanette20a, agrawal2021improved}. However, these algorithms are based on a realizability and closedness assumption termed Linear MDPs. The framework of Linear MDPs has been popular in the recent literature on the online learning framework of Reinforcement Learning because of its amenability to theory. It has been shown that under the Linear MDP model algorithms enjoy sublinear regret, \citep{jin2020provably} and has been extended to general function classes under the realizability and closedness assumption of the Bellman Operator \citep{dann2022guarantees}. However, without the assumption of realizability and closedness, the theory fails, in the sense that one is not able to show the algorithm learns an optimal policy. For example, \cite{jin2020provably} shows that Least-Squares-Value-Iteration with UCB exploration bias suffers linear regret if the Linear MDP assumption is removed. Similarly \cite{zanette2020learning} shows a linear regret under a misspecification notion termed as Inherent Bellman Error. In this regard we believe our work is a first of its kind to show that standard algorithms, like fitted $Q$-learning under a behavioral policy can learn the optimal policy, even if the model is grossly misspecified.

\end{document}